\documentclass[11pt, oneside]{article}

\usepackage{etoolbox}
\newcommand{\arxiv}[1]{\iftoggle{rlc}{}{#1}}
\newcommand{\rlc}[1]{\iftoggle{rlc}{#1}{}}
\newtoggle{rlc}
\global\togglefalse{rlc}

\usepackage{geometry}
\usepackage{graphicx}	
\usepackage{amssymb}
\usepackage{amsfonts,latexsym,amsthm,amssymb,amsmath,amscd,euscript}
\usepackage{bbm}
\usepackage{framed}
\usepackage{fullpage}
\usepackage{mathrsfs}
\usepackage{hyperref}

\hypersetup{
  colorlinks=true,
  linkcolor=blue,
  filecolor=blue,
  citecolor = black,      
  urlcolor=cyan,
}
\usepackage{accents}
\usepackage{setspace}
\hypersetup{colorlinks=true,citecolor=blue,urlcolor=black,linkbordercolor={1 0 0}}
\usepackage{tikz-cd}
\newcount\Comments  %
\ifnum\Comments=1
\usepackage[inline]{showlabels}

\fi
\usepackage{turnstile}
\usepackage{mathpartir}
\usepackage{tikz}
\usepackage{xspace}

\usepackage{algorithm}
\usepackage{verbatim}
\usepackage[noend]{algpseudocode}

\newcommand{\nc}{\newcommand}

\usepackage[nameinlink,capitalize]{cleveref}
\crefformat{equation}{#2(#1)#3}
\Crefformat{equation}{#2(#1)#3}

\Crefformat{figure}{#2Figure #1#3}
\Crefname{assumption}{Assumption}{Assumptions}
\Crefformat{assumption}{#2Assumption #1#3}
\Crefname{subsubsection}{Section}{Sections}
\crefformat{subsubsection}{#2Section #1#3}
\Crefformat{subsubsection}{#2Section #1#3}

\nc{\sups}[1]{^{\scriptscriptstyle{#1}}}
\nc{\subs}[1]{_{\scriptscriptstyle{#1}}}

\input{widebar}
\newcommand{\wb}{\widebar}

\newcommand{\epfinal}{\varepsilon_{\mathsf{final}}}
\newcommand{\epapx}{\varepsilon_{\mathsf{apx}}}

\newcommand{\Pilin}{\Pi^{\mathsf{lin}}}
\newcommand{\Pilinp}[1][\sigma]{\Pi^{\mathsf{Plin}, {#1}}}
\newcommand{\Pilinpp}{\Pi^{\mathsf{Plin}}}
\newcommand{\tilPilinp}[1][\sigma]{\til{\Pi}^{\mathsf{Plin}, {#1}}}

\newcommand{\Sm}[1][\sigma]{\mathsf{S}_{#1}}
\newcommand{\nrm}[1]{\left\| {#1} \right\|_2}
\arxiv{\newcommand{\citep}{\cite}}
\arxiv{\newcommand{\citet}{\cite}}

\nc{\cphi}{{\alpha}}
\nc{\Actor}{\texttt{Actor}\xspace}
\nc{\EstFeature}{\texttt{EstFeature}\xspace}
\nc{\ExpFTPL}{\texttt{ExpFTPL}\xspace}
\nc{\dist}{\mathrm{dist}}
\nc{\Bquad}{B^{\mathsf{quad}}}
\nc{\thetar}{\theta^{\mathsf{r}}}
\nc{\sign}{\mathrm{sign}}

\makeatletter
\newtheorem*{rep@theorem}{\rep@title}
\newcommand{\newreptheorem}[2]{%
\newenvironment{rep#1}[1]{%
 \def\rep@title{#2 \ref{##1}}%
 \begin{rep@theorem}}%
 {\end{rep@theorem}}}
\makeatother

\makeatletter
\newcommand\xlabel[2][]{\phantomsection\def\@currentlabelname{#1}\label{#2}}
\makeatother

\theoremstyle{plain}
\newtheorem{theorem}{Theorem}
\newreptheorem{theorem}{Theorem}
\newtheorem{lemma}[theorem]{Lemma}
\newtheorem{corollary}[theorem]{Corollary}

\newtheorem{assumption}[theorem]{Assumption}
\theoremstyle{definition}
\newtheorem{definition}{Definition}
\newtheorem{defn}[definition]{Definition}

\newtheorem{remark}[definition]{Remark}

\numberwithin{theorem}{section}
\numberwithin{definition}{section}

\nc{\DMO}{\DeclareMathOperator}

\DeclareMathOperator*{\argmax}{arg\,max}

\DMO{\prox}{prox}
\DMO{\UCB}{UCB}
\DMO{\LCB}{LCB}
\nc{\phidiff}{\phi\sups{\Delta}}
\nc{\pexp}{q_{\mathrm{exp}}}
\nc{\nn}{\nonumber}
\nc{\rk}{\mathrm{rk}}
\nc{\brk}[3]{{\rm br}_{#1}^{#2}({#3})}
\nc{\co}{{\rm co}}
\nc{\br}[2]{{\rm br}^{#1}({#2})}
\nc{\depth}[1]{{\rm d}({#1})}
\nc{\tA}{\textsc{A}}
\nc{\child}[2]{{\rm ch}_{#1}({#2})}
\nc{\parent}[1]{{\rm pa}({#1})}
\nc{\dg}{\dagger}
\nc{\bB}{\mathbf{B}}
\nc{\be}{\mathbf{e}}
\nc{\Span}{{\rm Span}}
\nc{\unif}{\mathsf{unif}}
\nc{\indsig}[2]{\mathcal{I}_{#1}({#2})}
\nc{\total}{{\rm fin}}
\nc{\early}{{\rm pre}}
\nc{\zsink}{z_{\rm sink}}
\nc{\lowv}{{\rm low}}
\nc{\ol}{\overline}
\nc{\ul}{\underline}
\nc{\madec}[3]{\texttt{ma-dec}_{#1}({#2}, {#3})}
\nc{\madeco}[1]{\texttt{ma-dec}_{#1}}
\nc{\madecd}[3]{\texttt{ma-dec}^{\texttt{d}}_{#1}({#2}, {#3})}
\nc{\SF}{\mathscr{F}}
\nc{\SP}{\mathscr{P}}
\nc{\SPc}{\wb{\mathscr{P}}}
\nc{\SB}{\mathscr{B}}
\nc{\SC}{\mathscr{C}}
\nc{\BS}{\mathbb{S}}
\nc{\PiMarkov}{\Pi^{\mathsf{M}}}
\nc{\PiM}{\PiMarkov}
\nc{\piopt}{\pi^{\mathsf{opt}}}
\nc{\Critic}{\texttt{Critic}\xspace}
\nc{\trunc}[2]{\mathsf{trunc}_{#2}({#1})}
\nc{\sbl}{of strong Bellman type\xspace}
\nc{\inormal}[1][\Phi, u,v]{\til{N}_{{#1}}}

\nc{\gamvec}{\gamma}
\nc{\til}{\widetilde}
\nc{\td}{\tilde}
\nc{\wh}{\widehat}
\nc{\todo}[1]{\ifnum\Comments=1 {\color{red}  [TODO: #1]}\fi}
\nc{\old}[1]{\ifnum\Comments=1 {\color{brown}  [OLD: #1]}\fi}
\nc{\noah}[1]{\ifnum\Comments=1 {\color{purple} [ng: #1]}\fi}
\nc{\dhruv}[1]{\ifnum\Comments=1 {\color{magenta} [dr: #1]}\fi}
\nc{\BP}{\mathbb{P}}
\nc{\BI}{\mathbb{I}}
\nc{\midpoint}[1][\Phi,\phi_1,\phi_2]{\mu^{\star}_{{#1}}}

\nc{\fools}[3]{\MF_{#3}({#1}, {#2})}
\nc{\fool}[2]{\MF({#1},{#2})}
\nc{\clip}[2]{{\rm clip}\left[ \left. {#1} \right| {#2} \right]}
\nc{\imax}{\omega}
\DMO{\conv}{conv}
\nc{\MH}{\mathcal{H}}
\nc{\MV}{\mathcal{V}}
\nc{\MC}{\mathcal{C}}
\nc{\MI}{\mathcal{I}}
\nc{\st}{\star}
\nc{\lng}{\langle}
\nc{\rng}{\rangle}
\DMO{\OOPT}{opt}
\nc{\dopt}[2]{\ell_{\OOPT}({#1},{#2})}
\nc{\grad}{\nabla}
\nc{\MG}{\mathcal{G}}
\nc{\MP}{\mathcal{P}}
\nc{\PP}{\mathbb{P}}
\nc{\TT}{\mathbb{T}}
\nc{\TTmax}{\TT_{\max}}
\DMO{\REG}{Reg}
\DMO{\WREG}{wReg}
\nc{\reg}[2]{{\Delta}_{{#1}}({#2})}
\nc{\wreg}[2]{{\Delta}^{\rm w}_{{#1}}({#2})}
\nc{\Reg}[2]{{\REG}_{{#1}}({#2})}
\nc{\wReg}[2]{{\WREG}_{{#1}}({#2})}
\DMO{\Ham}{Ham}
\DMO{\Gap}{Gap}
\DMO{\GD}{GD}
\DMO{\GDA}{GDA}
\DMO{\EG}{EG}
\nc{\TE}{\til{\E}}
\nc{\Var}{\mathbb{V}}
\DMO{\Cov}{Cov}
\DMO{\OGDA}{OGDA}
\DMO{\Unif}{Unif}
\DMO{\Tr}{Tr}
\nc{\Qu}{\ul{Q}}
\nc{\Qo}{\ol{Q}}
\nc{\Ro}{\ol{R}}
\nc{\Vu}{\ul{V}}
\nc{\Vo}{\ol{V}}
\nc{\RanQ}{\Delta Q}
\nc{\RanV}{\Delta V}
\nc{\clipQ}{\Delta \breve{Q}}
\nc{\frzQ}{\Delta \mathring{Q}}
\nc{\clipV}{\Delta \breve{V}}
\nc{\clipdelta}{\breve{\delta}}
\nc{\cliptheta}{\breve{\theta}}
\nc{\delmin}{\Delta_{{\rm min}}}
\nc{\delmins}[1]{\Delta_{{\rm min},{#1}}}
\nc{\gapfinal}[1]{\max \left\{ \frac{\frzQ_{{#1}}^{k^\st}(x,a)}{2H}, \frac{\delmin}{4H} \right\}}
\nc{\post}[2]{R({#1}; {#2})}
\nc{\posts}[3]{R_{#3}({#1}; {#2})}

\nc{\algnst}[1]{\begin{align*}#1\end{align*}}
\nc{\algn}[1]{\begin{align}#1\end{align}}
\nc{\matx}[1]{\left(\begin{matrix}#1\end{matrix}\right)}
\renewcommand{\^}[1]{^{(#1)}}

\nc{\nuu}{\nu}

\nc{\bel}[1]{\mathbf{b}({#1})}
\nc{\nbel}[1]{\bar{\mathbf{b}}({#1})}
\nc{\sbel}[2]{\mathbf{b}'_{#1}({#2})}
\nc{\nsbel}[2]{\bar{\mathbf{b}}'_{#1}({#2})}

\nc{\bv}{\mathbf{v}}
\nc{\bone}{\mathbf{1}}
\nc{\bX}{\mathbf{X}}
\nc{\bY}{\mathbf{Y}}
\nc{\bG}{\mathbf{G}}
\nc{\bz}{\mathbf{z}}
\nc{\bw}{\mathbf{w}}
\nc{\bA}{\mathbf{A}}
\nc{\bJ}{\mathbf{J}}
\nc{\bK}{\mathbf{K}}
\nc{\bb}{\mathbf{b}}
\nc{\ba}{\mathbf{a}}
\nc{\bc}{\mathbf{c}}
\nc{\bC}{\mathbf{C}}
\nc{\BR}{\mathbb R}
\nc{\BA}{\mathbb{A}}
\nc{\BC}{\mathbb C}
\nc{\bx}{\mathbf{x}}
\nc{\bS}{\mathbf{S}}
\nc{\bM}{\mathbf{M}}
\nc{\bR}{\mathbf{R}}
\nc{\bN}{\mathbf{N}}
\nc{\NN}{\mathbb{N}}
\nc{\by}{\mathbf{y}}
\nc{\sy}{y}
\nc{\sx}{x}

\nc{\MO}{\mathcal O}
\nc{\MU}{\mathcal{U}}
\nc{\ME}{\mathcal{E}}
\nc{\MN}{\mathcal{N}}
\nc{\MK}{\mathcal{K}}
\nc{\MM}{\mathcal{M}}
\nc{\MS}{\mathcal{S}}
\nc{\MT}{\mathcal{T}}
\nc{\BF}{\mathbb F}
\nc{\BQ}{\mathbb Q}
\nc{\MX}{\mathcal{X}}
\nc{\MA}{\mathcal{A}}
\nc{\MD}{\mathcal{D}}
\nc{\MB}{\mathcal{B}}
\nc{\MZ}{\mathcal{Z}}
\nc{\MJ}{\mathcal{J}}
\nc{\MW}{\mathcal{W}}
\nc{\MR}{\mathcal{R}}
\nc{\MY}{\mathcal{Y}}
\nc{\BZ}{\mathbb Z}
\nc{\BN}{\mathbb N}
\nc{\ep}{\epsilon}
\nc{\epbe}{\varepsilon_{\mathsf{BE}}}
\nc{\epout}{\varepsilon_{\mathsf{outlier}}}
\nc{\bellc}[1][h]{\MT_{#1}^\circ}
\nc{\vep}{\varepsilon}
\nc{\gapfn}[1]{\varepsilon_{#1}}
\nc{\ggapfn}[2]{\varphi_{#1}({#2})}
\nc{\epsahk}{\gapfn{0}}
\nc{\BH}{\mathbb H}
\nc{\BG}{\mathbb{G}}
\nc{\D}{\Delta}
\nc{\MF}{\mathcal{F}}
\nc{\One}[1]{\mathbbm{1}\{{#1}\}}
\nc{\bOne}{\mathbf{1}}
\nc{\Aopt}{\mathcal{A}^{\rm opt}}
\nc{\Amul}{\mathcal{A}^{\rm mul}}

\nc{\SQ}{\mathsf Q}

\nc{\DO}{\accentset{\circ}{\D}}
\nc{\mf}{\mathfrak}
\nc{\mfp}{\mathfrak{p}}
\nc{\mfq}{\mf{q}}
\nc{\Sp}{\mbox{Spec}}
\nc{\Spm}{\mbox{Specm}}
\nc{\hookuparrow}{\mathrel{\rotatebox[origin=c]{90}{$\hookrightarrow$}}}
\nc{\hookdownarrow}{\mathrel{\rotatebox[origin=c]{-90}{$\hookrightarrow$}}}
\nc{\hra}{\hookrightarrow}
\nc{\tra}{\twoheadrightarrow}
\nc{\sgn}{{\rm sgn}}
\nc{\aut}{{\rm Aut}}
\nc{\Hom}{{\rm Hom}}
\nc{\img}{{\rm Im}}
\DMO{\id}{Id}
\DMO{\supp}{supp}
\DMO{\KL}{KL}
\nc{\kld}[2]{\KL({#1}||{#2})}
\nc{\ren}[2]{D_2({#1}||{#2})}
\nc{\chisq}[2]{\chi^2({#1}||{#2})}
\nc{\tvd}[2]{D_{\mathsf{TV}}({#1}, {#2})}
\nc{\hell}[2]{D_{\mathsf{H}}^2({#1}, {#2})}
\nc{\dbi}[3][\pi]{D_{\mathsf{bi}}^{#1}({#2} \| {#3})}
\DMO{\BSS}{BSS}
\DMO{\BES}{BES}
\DMO{\BGS}{BGS}
\DMO{\poly}{poly}
\nc{\indep}{\perp}
\DMO{\sink}{sink}
\nc{\fp}[1]{\MP_1({#1})}
\nc{\BO}{\mathbb{O}}
\nc{\BT}{\mathbb{T}}

\nc{\RR}{\mathbb{R}}
\nc{\Gradient}{\nabla}
\DMO{\diag}{diag}
\nc{\norm}[1]{\left \lVert #1 \right \rVert}
\nc{\EE}{\mathop{\mathbb{E}}}
\nc{\MQ}{\mathcal{Q}}
\nc{\ML}{\mathcal{L}}
\nc{\cPhi}{\bar \Phi}
\nc{\SA}{\mathscr{A}}

\DMO{\PR}{Pr}
\renewcommand{\Pr}{\PR}
\nc{\E}{\mathbb{E}}
\nc{\ra}{\rightarrow}
\renewcommand{\t}{\top}
\nc{\hc}{\{0,1\}^n}
\nc{\pmhc}[1]{\{-1,1\}^{#1}}
\nc{\Dbnd}{D}
\nc{\Bbnd}{B}
\nc{\brew}{b_{\mathsf{rew}}}
\nc{\binit}{b_{\mathsf{init}}}
\nc{\Cvg}{\mathscr{C}}
\nc{\del}[1]{}
\nc{\Pisoft}{\Pi^{\mathsf{soft}}}
\nc{\pisoft}{\pi^{\mathsf{soft}}}

\title{The Role of Inherent Bellman Error in Offline Reinforcement Learning with Linear Function Approximation}
\author{Noah Golowich\thanks{Email: \texttt{nzg@mit.edu}. Supported by a Fannie \& John Hertz Foundation Fellowship and an NSF Graduate Fellowship.} \\ MIT \and Ankur Moitra\thanks{Email: \texttt{moitra@mit.edu}. Supported in part by a Microsoft Trustworthy AI Grant, an ONR grant and a David and Lucile Packard Fellowship.} \\ MIT}
\date{\today}

\begin{document}
\maketitle

\begin{abstract}
  In this paper, we study the offline RL problem with linear function approximation. Our main structural assumption is that the MDP has \emph{low inherent Bellman error}, which stipulates that linear value functions have linear Bellman backups with respect to the greedy policy. This assumption is natural in that it is essentially the minimal assumption required for value iteration to succeed. We give a computationally efficient algorithm which succeeds under a \emph{single-policy coverage} condition on the dataset, namely which outputs a policy whose value is at least that of any policy which is well-covered by the dataset. Even in the setting when the inherent Bellman error is 0 (termed \emph{linear Bellman completeness}), our algorithm yields the first known guarantee under single-policy coverage. \arxiv{

    }In the setting of positive inherent Bellman error $\epbe > 0$, we show that the suboptimality error of our algorithm scales with $\sqrt{\epbe}$. Furthermore, we prove that the scaling of the suboptimality with $\sqrt{\epbe}$ cannot be improved for \emph{any} algorithm. Our lower bound stands in contrast to many other settings in reinforcement learning with misspecification, where one can typically obtain performance that degrades \emph{linearly} with the misspecification error.
\end{abstract}

\section{Introduction}

The study of \emph{reinforcement learning (RL)} focuses on the problem of sequential decision making in a stateful and stochastic environment, typically modeled as a \emph{Markov Decision Process (MDP)}. An agent aims to maximize its expected reward over a finite time horizon $H$, also known as its \emph{value}, by choosing a \emph{policy}, or a mapping from states to actions. %
In the \emph{offline} (or \emph{batch}) \emph{RL} problem, an agent's only knowledge of its environment comes from a dataset $\MD$ consisting of samples drawn from the state transition distributions and reward functions of the MDP. Given $\MD$, the learning algorithm aims to compute a policy $\hat \pi$ whose value is at least as good as that of some \emph{reference policy} $\pi^\st$. 

A key challenge in offline RL is understanding how to choose the policy $\hat \pi$ when the dataset $\MD$ exhibits incomplete coverage of the environment, meaning that the transitions from many states are not represented in $\MD$. The naive approach to this problem would proceed via some variant of \emph{value iteration}. At each step $h = H, H-1, \ldots, 1$ of the time horizon, given a policy acting at steps $h' > h$, value iteration uses $\MD$ to compute estimates of the \emph{value function} of the policy, which maps each state-action pair to the policy's expected reward starting at that state-ation pair. Value iteration then greedily chooses a policy at step $h$ which maximizes this estimated value, and proceeds to step $h-1$. Unfortunately, this approach suffers from the fact that state-action pairs which are \emph{undercovered} by the dataset $\MD$ may have, due to random fluctuations in $\MD$, overly optimistic estimates of their value. The chosen policy $\hat \pi$ will then aim to visit such state-action pairs, which could in fact be very suboptimal. 

A common approach to correcting the above issue is the principle of \emph{pessimism} (e.g., \cite{fujimoto2019offpolicy,xie2021bellman,liu2020provably,yu2020model,jin2021pessimism}), which chooses $\hat \pi$ so as to maximize an \emph{underestimate} of its value, where the underestimate is chosen according to constraints that force it to be consistent with the dataset. By ensuring that $\hat \pi$ takes actions whose value is robust to random fluctuations in the dataset, pessimistic algorithms typically ensure a guarantee of the following form: for any policy $\pi^\st$ whose state-action pairs are well-covered by $\MD$, the value of $\hat \pi$ is guaranteed to be at least that of $\pi^\st$. The assumption made on $\pi^\st$ here is typically known as a \emph{single-policy coverage} condition (formalized in \cref{def:coverage}); along with several variants, it has come to represent a gold standard for obtaining offline RL guarantees. 

\paragraph{Function approximation \& misspecification in offline RL.} As the state and action spaces encountered in practice tend to be large or infinite, much of the theoretical work on offline RL makes \emph{function approximation} assumptions, which introduce function classes $\MQ_h$ for steps $h \in [H]$, whose elements can be used to approximate the value function for a policy. %
Due to the complexity of the offline RL problem, our understanding of the optimal guarantees attainable remains limited even for the fundamental special case in which the classes $\MQ_h$ are \emph{linear}, which is the focus of this paper. Concretely, letting the state and action spaces be denoted by $\MX$ and $\MA$, respectively, we assume the following: for some known \emph{feature mappings} $\phi_h : \MX \times \MA \to \BR^d$, $\MQ_h$ is the class of functions $(x,a) \mapsto \lng \phi_h(x,a), w \rng$, where $w \in \BR^d$ belongs to some bounded set. As an added benefit of focusing on the linear setting, we will be able to obtain end-to-end computationally efficient algorithms, without reliance on a regression oracle for $\MQ_h$. 

It is unreasonable to expect that elements of the classes $\MQ_h$ coincide exactly with the actual value functions for the underlying MDP. Therefore, it is essential to understand the price paid by having \emph{misspecification error} in $\MQ_h$. To quantify this misspecification error we use the \emph{inherent Bellman error}  \citep{munos2008finite,munos2005error}, denoted $\epbe$, which describes the maximum distance between the Bellman backup of any function in $\MQ_{h+1}$ with respect to the greedy policy and a best-approximating function in $\MQ_h$. The inherent Bellman error is particularly natural since it is exactly what quantifies the degree to which value iteration succeeds: in particular, the regression problems solved by value iteration are $\epbe$-approximatly well-specified. As a result, it can be shown that, if $\MD$ covers the entire state space in an appropriate sense and the inherent Bellman error is bounded by $\epbe$, then value iteration produces a policy whose suboptimality may be bounded by $\poly(d,H) \cdot \epbe$ \citep{munos2008finite}. %
 Moreover, the linear growth of suboptimality error with respect to $\epbe$ cannot be improved in general \citep{tsitsiklis1996feature}. 
 
 The results of \citet{munos2008finite,munos2005error} discussed above serve as a useful sanity check on the reasonableness of low inherent Bellman error, but do little to address the problems encountered in typical offline RL situations as a result of their stringent assumption that $\MD$ covers the full state space. In most such settings, ranging from autonomous driving to healthcare, one should expect the offline data to be gathered in regions of the state space that result from executing reasonably good policies. (For instance, one should not expect much offline data involving states corresponding to driving a car off a cliff.) Thus, positive results under only single-policy coverage conditions are much more desirable. 
 Our main goal is to address the following question: \emph{is boundedness of the inherent Bellman error sufficient for computationally efficient offline RL under only a single-policy coverage condition?} Prior to our work, this question was open even for the special case of 0 inherent Bellman error, which is typically known as \emph{linear Bellman completeness}.

\subsection{Main results} 
Our first result is a positive answer to our main question; to state it, we need the following notation (introduced formally in \cref{sec:prelim}). An offline RL algorithm takes as input a dataset $\MD$ of size $n$, consisting of $n$ tuples $(h, x, a, r, x')$, denoting a sample of the transition at step $h \in [H]$: namely, at state $x$, action $a$ was taken, reward $r$ was received, and the next state observed was $x'$.   For each $h \in [H]$, let $\Sigma_h$ denote the unnormalized covariance matrix of feature vectors in $\MD$ at step $h$. Moreover, for a policy $\pi$, let $\Cvg_{\MD, \pi} := \sum_{h=1}^H \| \E^\pi[\phi_h(x_h, a_h)] \|_{n\Sigma_h^{-1}}$ denote the \emph{coverage parameter} for $\pi$ with respect to $\MD$, which measures to what degree vectors in $\MD$ extend in the direction of a typical feature vector drawn from $\pi$. We denote the inherent Bellman error of the MDP by $\epbe \geq 0$. 
\begin{theorem}[Informal version of \cref{thm:pacle-ftpl}]
  \label{thm:main-informal}
  There is an algorithm (namely, \cref{alg:actor}) which given the dataset $\MD$ as input, outputs a policy $\hat \pi$ at random so that for any policy $\pi^\st$, we have
  \begin{align}
\E \left[V_1^{\pi^\st}(x_1) - V_1^{\hat \pi}(x_1)\right] \leq & \Cvg_{\MD, \pi^\st} \cdot \poly(d,H) \cdot \left( \sqrt{\epbe} + \frac{1}{\sqrt{n}} \right)\nonumber.
  \end{align}
  Moreover, \cref{alg:actor} runs in time $\poly(d,H,n)$. 
\end{theorem}
A notable feature of \cref{thm:main-informal} is the fact that the suboptimality of $\hat \pi$ scales with $\sqrt{\epbe}$, which contrasts with the \emph{linear} scaling in $\epbe$ seen in classic works on offline RL \citep{munos2008finite,munos2005error} and recent works studying \emph{online RL} under low inherent Bellman error \citep{zanette2020learning,zanette2020provably}, as well as linear dependence on the misspecification error for other types of misspecification in both online and offline RL settings \citep{xie2021bellman,zanette2021provable,nguyen2023on,jin2020provably,wei2021nonstationary}.\footnote{Note that, in \cite{xie2021bellman}, the realizability error $\vep_{\MF}$ and policy completeness error $\vep_{\MF, \MF}$ appear in a square root (see Theorem 3.1 of \citet{xie2021bellman}), since the quantities $\vep_{\MF}, \vep_{\MF, \MF}$ actually represent the mean \emph{squared} errors. Moreover, the cube-root dependence on $\vep_{\MF}, \vep_{\MF, \MF}$ in \cite{xie2021bellman} is suboptimal and is improved in \citet[Theorem 2]{nguyen2023on}.} We show below that, perhaps surprisingly, the square-root dependence on $\epbe$ cannot be improved, even in a statistical sense: %
\begin{theorem}[Informal version of \cref{thm:lb-formal}]
  \label{thm:lb-informal}
  Fix $\epbe \in (0,1)$, $n \in \BN$, and set $d = H = 2$. There are feature mappings $\phi_h : \MX \times \MA \to \BR^d$, $h \in [H]$, where $|\MA| = 4$, so that for any (randomized) offline RL algorithm $\mathfrak{A}$, %
  the following holds. There is some MDP with inherent Bellman error bounded by $\epbe$, together with some policy $\pi^\st$ so that $\Cvg_{\MD, \pi^\st} = O(1)$ yet the output policy $\hat \pi$ of $\mathfrak{A}$ satisfies
  \begin{align}
\E\left[V_1^{\pi^\st}(x_1) - V_1^{\hat \pi}(x_1)\right] \geq \Omega \left( \sqrt{\epbe} + \frac{1}{\sqrt{n}} \right)\nonumber.
  \end{align}
\end{theorem}
We emphasize that \cref{thm:lb-informal} establishes a surprising separation between offline and online RL: whereas in the online setting, as mentioned above, one can learn a policy whose suboptimality scales linearly with $\epbe$, the optimal scaling in the offline setting is linear in $\sqrt{\epbe}$. Thus, \emph{misspecification is more expensive in the offline setting}, i.e., when one is not allowed to adaptively gather data. 

\paragraph{Relation to prior work.} %
\cite{wang2021what,zanette2021exponential} showed exponential lower bounds for offline RL under the assumption of \emph{all-policy realizability}, which stipulates that the $Q$-value function of all policies is linear (i.e., belongs to $\MQ_h$). 
This lower bound is incomparable to that of \cref{thm:lb-informal}: whereas the inherent Bellman error of the instances in \cite{wang2021what,zanette2021exponential} satisfies $\epbe = \Omega(1)$ (so that lower bounds of $\sqrt{\epbe}$ and $\epbe$ are indistinguishable), the instance used to prove \cref{thm:lb-informal} does not satisfy all-policy realizability. Moreover, the lower bounds are unrelated on a technical level. 

A recent line of work has investigated a strengthening of all-policy realizability under which offline RL can be achieved, known as \emph{Bellman restricted closedness} (or often simply as \emph{(policy) completeness}). Under this condition, there are statistically efficient algorithms for offline RL with only single-policy coverage, for general function classes $\MQ_h$  \citep{zanette2021provable,xie2021bellman,cheng2022adversarially,nguyen2023on}. Moreover, given a regression oracle for the class $\MQ_h$ which implements a variant of regularized least-squares, many of these works \citep{xie2021bellman,cheng2022adversarially,nguyen2023on} show that the same offline RL guarantee can be obtained in an oracle-efficient manner. 
Since regularized least-squares is computationally efficient when $\MQ_h$ is linear, it follows that, under Bellman restricted closedness computationally efficient offline RL algorithms are known \citep{xie2021bellman,cheng2022adversarially,nguyen2023on,zanette2021provable}.

For a class $\Pi$ of policies, an MDP satisfies $\Pi$-Bellman restricted closedness if the Bellman backup of any function in $\MQ_{h+1}$, \emph{with respect to any policy in $\Pi$}, belongs to $\MQ_h$ (see \cref{def:brc}). Generally speaking, the above papers achieving the strongest bounds assume $\Pi$-Bellman restricted closedness for the class $\Pisoft$ of softmax policies (\cref{def:softmax})  \citep{zanette2021provable,nguyen2023on}. While technically speaking such an assumption is incomparable with linear Bellman completeness (i.e., 0 inherent Bellman error), the latter has the advantage of being universal in the sense that it is defined so as to allow the fundamental approach of \emph{value iteration} to succeed. In contrast, $\Pi$-Bellman restricted closedness only enjoys a similar ``universality'' property when one takes $\Pi$ to be the class of \emph{all Markov policies}, in which case $\Pi$-Bellman restricted closedness allows \emph{policy iteration} to succeed \cite[Theorem D.1]{du2020is}. 
However, as we discuss further in \cref{sec:related-work}, in this case $\Pi$-Bellman restricted closendess becomes significantly stronger than linear Bellman completeness: in fact, if each state has two distinct feature vectors, then it becomes equivalent to the linear MDP assumption \cite[Proposition 5.1]{jin2019provably}. %

\begin{remark}[Confluence of terminology] \label{rmk:confluence}
  Due to an unfortunate confluence of terminology, some prior works in the literature (e.g., \cite{uehara2022pessimistic}) refer to the setting of $\Pi$-Bellman restricted closedness when the classes $\MQ_h$ are linear as ``linear Bellman completeness''. We do not use this convention, and use ``linear Bellman completeness'' to refer to the setting when the inherent Bellman error is $0$.
\end{remark}

Finally, we remark that in the construction used to prove \cref{thm:lb-informal}, the comparator policy $\pi^\st$ is not the optimal policy in the MDP. This observation suggests the following intriguing open problem: if one further assumes that $\pi^\st$ is the optimal policy, then can one improve the $\sqrt{\epbe}$ upper bound in \cref{thm:main-informal} to be linear in $\epbe$, or does an analogue of \cref{thm:lb-informal} continue to hold?

\arxiv{\subsection{Overview of techniques}}
\arxiv{\paragraph{Proof idea of the upper bound.}}
\arxiv{A key ingredient in the proof of \cref{thm:pacle-ftpl} \arxiv{(the formal version of \cref{thm:main-informal})} is a new structural condition (\cref{thm:pert-linear-informal} below) proving that MDPs with low inherent Bellman error satisfy $\Pi$-Bellman restricted closedness for \arxiv{a certain class $\Pi$ consisting}\rlc{the class} of \emph{perturbed linear policies}. To understand this result, we first consider the special case of linear Bellman completeness, i.e., $\epbe = 0$. In this case, we show (as a special case of \cref{thm:pert-linear-informal}) that $\Pilin$-Bellman restricted closedness holds, where $\Pilin$ denotes the class of linear policies, namely those of the form $x \mapsto \argmax_{a \in \MA} \lng \phi_h(x,a), \theta \rng$, for some $\theta \in \BR^d$. Unfortunately, when $\epbe$ is positive but small,  $\Pilin$-Bellman restricted closedness may no longer hold even in an approximate sense. We correct for this deficiency by modifying the policy class $\Pilin$ to instead consist of perturbed linear policies\arxiv{: 
a perturbed linear policy at step $h$, $\pi_h : \MX \to \Delta(\MA)$, is indexed by a vector $v \in \BR^d$ and a scalar $\sigma > 0$. It chooses an action by first drawing a random vector $\theta$ from the normal distribution with mean $v$ and covariance $\sigma^2 \cdot I_d$, and then chooses $\argmax_{a' \in \MA} \lng \phi_h(x,a'), \theta \rng$. We denote the set of such policies by $\Pilinp[\sigma]$.}\rlc{ (\cref{def:plinear}).} %
\begin{theorem}[$\Pilinp$-Bellman restricted closedness; informal version of \cref{cor:ibe-linear}]
  \label{thm:pert-linear-informal}
  Suppose that $\pi \in \Pilinp[\sigma]$ and $h \in [H-1]$. Then there is a matrix $\MT_h^\pi : \BR^{d \times  d}$ so that, for all $w \in \BR^d$ and $(x,a) \in \MX \times \MA$,
  \begin{align}
\left| \lng \phi_h(x,a), \MT_h^\pi w \rng - \EE_{\substack{x' \sim P_h(x,a) \\ a' \sim \pi_{h+1}(x')}}\left[ \lng \phi_{h+1}(x', a'), w\rng \right]\right| \leq \tilde O \left( \| w \|_2 d^{3/2} \cdot \left( \sqrt{d} + \frac{1}{\sigma} \right) \right) \cdot \epbe\label{eq:epbe-ub-informal}. 
  \end{align}
\end{theorem}
\rlc{The proof of \cref{thm:pert-linear-informal} proceeds by considering, for a pair $(x,a) \in \MX \times \MA$, the function $w \mapsto Q_h(w; x,a) := \E_{x' \sim P_h(x,a)}[\max_{a' \in \MA} \lng \phi_{h+1}(x', a'), w \rng]$. The key observation is that if $\pi_{h+1} = \pi_{h+1,w,\sigma}$ is the perturbed linear policy specified by $w \in \BR^d, \sigma > 0$ (see \cref{def:plinear}), then the second term on the left-hand side of \cref{eq:epbe-ub-informal} may be expressed as follows in terms of the gradient with respect to $w$ of the Gaussian smoothing of $Q$:
\begin{align}
\grad_w \Sm Q(x,a,w) = \E_{x' \sim P_h(x,a)}[\phi_{h+1}(x', \pi_{h+1,w,\sigma}(x'))]. 
\end{align}
The existence of the desired matrix $\MT_h^\pi$ as claimed by \cref{thm:pert-linear-informal} then follows by using the fact that $|Q_h(w; x,a) - \lng \phi_h(x,a), \bellc w \rng|\leq \epbe$ (see \cref{eq:ibe-def-p}) as well as the fact that differentiating is a linear operation. This argument must overcome a few challenges in the setting that $\epbe, \sigma > 0$ due to the fact that $\Sm Q(x,a,w)$ is different from $Q(x,a,w)$; full details are given in \cref{sec:lbc-structural}. }

\rlc{We proceed to discuss the remainder of the proof of \cref{thm:pacle-ftpl}.} Previous work \citep{zanette2021provable} shows that, under Bellman restricted closedness with respect to a \emph{softmax} policy class, an actor-critic method suffices to obtain offline RL guarantees under single-policy coverage. While Bellman restricted closedness does not hold in general for the softmax policy class under the assumption of linear Bellman completeness (see \cref{lem:lbc-not-brc}), we prove \arxiv{\cref{thm:main-informal}}{\rlc{\cref{thm:pacle-ftpl}}} by adapting the actor-critic method in \cite{zanette2021provable} to work instead with the set of perturbed linear policies. To explain the requisite modifications, we briefly overview the actor-critic method: roughly speaking, the overall goal is to solve the problem $\max_{\pi} \min_{M \in \MM_\MD(\pi)} V\sups{M, \pi}(x_1)$, where $\MM_\MD(\pi)$ indicates a set of MDPs which, under trajectories drawn from $\pi$, are statistically consistent with the dataset $\MD$. Moreover, $V\sups{M, \pi}(x_1)$ denotes the value of policy $\pi$ in MDP $M$. Minimization over $M \in \MM_\MD(\pi)$ corresponds to the pessimism principle, and standard arguments \citep{xie2021bellman,zanette2021provable} show that an exact solution to this max-min problem would solve the offline RL task.

To solve this max-min problem in a computationally efficient manner, the actor-critic method uses the ``no-regret learning vs.~best response'' approach: a sequence of policies $\pi\^t$ and MDPs $M\^t \in \MM_\MD(\pi\^t)$ is generated in the following manner. At each step $t$, a no-regret learning algorithm, called the \texttt{Actor}, generates a policy $\pi\^t$; in response, an optimization algorithm, called the \texttt{Critic}, chooses $M\^t \in \MM_\MD(\pi\^t)$ so as to minimize $V\sups{M\^t, \pi\^t}(x_1)$. If $T$ is chosen sufficiently large, then, as shown in \cite{zanette2021provable}, a policy drawn uniformly from $\{ \pi\^t \ : \ t \in [T]\}$ will have sufficiently large value in the true MDP.

\arxiv{The \texttt{Actor} algorithm in \cite{zanette2021provable} uses the exponential weights algorithm at each state to update the policies $\pi\^t$, which leads $\pi\^t$ to be softmax policies. We replace exponential weights with a variant of the \emph{expected follow-the-perturbed leader (FTPL)} algorithm \cite[Algorithm 16]{hazan2016introduction}, which has the key advantage that the policies it produces are perturbed linear policies, as opposed to softmax policies. The optimal choice of the standard deviation of the perturbation turns out to be $\sigma \approx \sqrt{\epbe}$, which together with \cref{thm:pert-linear-informal} and appropriate modifications of the analysis in \cite{zanette2021provable} suffices to prove \arxiv{\cref{thm:main-informal}}\rlc{\cref{thm:pacle-ftpl}}.}

}

\arxiv{\paragraph{Proof idea of the lower bound.}}
\arxiv{\rlc{\paragraph{Proof overview for \cref{thm:lb-formal}.}}

  To explain the idea behind \arxiv{\cref{thm:lb-informal}}\rlc{\cref{thm:lb-formal}}, we first consider the following ``toy setup'': suppose that $H=d=2$, $\MA = \{0,1\}$, that rewards at step $h=1$ are known to be 0, and rewards at step $h=2$ are known to be induced by a coefficient vector $\thetar_2 \in \{ (1,1), (-1,1) \}$. Concretely, this type of uncertainty in $\thetar_2$ may be implemented by having a dataset $\MD$ for which all feature vectors at step 2 are parallel to $(0,1)$; thus, the first coordinate of $\thetar_2$ may remain unknown. 

Consider an MDP with states $ x_2, x_2' \in \MX$ so that, for each $a \in \MA$, $(x_1, a)$ transitions to either $x_2$ or $x_2'$, but which one is unknown. %
Moreover suppose that feature vectors at $x_2, x_2'$ are given as follows: for $a \in \{0,1\}$,
\begin{align}
\phi_2(x_2, a) = (1-2a, 1-(1-2a) \cdot \epbe/2), \quad \phi_2(x_2', a) = (1-2a, 1 + (1-2a) \cdot \epbe/2)\label{eq:phi2-ex}. %
\end{align}
We may ensure that the transitions described above are consistent with the requirement that the inherent Bellman error be bounded by $\epbe$, since the feature vector $\phi_2(x_2,a)$ is within distance $\epbe$ from $ \phi_2(x_2', a)$ for each $a$.

At a high level, our lower bound results from the following consideration: suppose the policy we wish to compete with at step 2, namely $ \pi_2^\st$, takes a uniformly random action at step 2 (so that its expected feature vector at step 2 is $(0,1)$).  %
Not knowing any information about the first coordinate of $\thetar_2$, a natural choice for $\hat \pi_2$ is the ``naive'' policy that maximizes the reward in the one ``known'' direction $(0,1)$, i.e., let $\hat \pi_2^{\mathsf{naive}}$ be the linear policy which, at state $x$, chooses $\argmax_{a \in \MA} \lng \phi_h(x,a), (0,1) \rng$. However, this choice suffers from the issue that, due to the $\epbe$ perturbation between $\phi_2(x_2, a)$ and $ \phi_2(x_2', a)$, $\hat \pi_2^{\mathsf{naive}}$ will choose an action $a$ at step 2, whose feature vector is either $(1, 1+\epbe/2)$ (if $x_1$ transitions to $x_2'$) or $(-1, 1+\epbe/2)$ (if $x_1$ transitions to $x_2$). By choosing $\thetar_2 = (-1,1)$ in the former case and $\thetar_2 = (1,1)$ in the latter case, we can thus force
this naive choice $\hat \pi_2^{\mathsf{naive}}$ to have suboptimality $-\Omega(1)$ compared to $\pi_2^\st$. This issue stems from the fact that the policy $\hat \pi_2^{\mathsf{naive}}$ is extremely brittle to small perturbations of $\phi_2(x_2, a)$: a change of size $\epbe$ in each of the feature vectors leads to a $\Omega(1)$-size change in the actual feature vector chosen by $\hat \pi_2^{\mathsf{naive}}$. 

Of course, in this specific example, one can simply instead set $\hat \pi_2$ to be the policy which chooses each action at step 2 with probability $1/2$, which will lead to a policy $\hat \pi$ whose value is within $O(\epbe)$ of $\pi^\st$. %
Roughly speaking, doing so corresponds to considering a perturbed linear policy, in the vein of \cref{thm:pert-linear-informal}. We can show, however, that such a perturbation must hurt the value of $\hat \pi$, for some alternative choice of MDP. Formally, we add transitions to states $x_2\^\ell$ at step 2 with state-action feature vectors  $\phi_2(x_2\^\ell, a) = (1-2a, 1 \pm (1-2a) \cdot \ell \cdot \epbe)$, for each value of $\ell \in \{ 1, 2, \ldots, \lfloor 1/\sqrt{\epbe} \rfloor \}$. %
A suboptimality of $\Omega(\sqrt{\epbe})$, with respect to some reference policy $\pi^\st$, arises because avoiding it would require $\hat \pi_2$ to act in a way consistent with the naive policy $\hat \pi_2^{\mathsf{naive}}$ (i.e., without perturbation) at states with features $\phi_\ell$, for $\ell = \Omega(1/\sqrt{\epbe})$. (At such states, the second component of the feature vectors deviates from $1$ by enough that it cannot be ``ignored''.) But because the MDP has inherent Bellman error of $\epbe$, similar reasoning to the previous paragraph shows that the algorithm's output policy $\hat \pi$ cannot act significantly differently at states with feature vectors $\phi_2(x_2\^\ell, a), \phi_2(x_2\^{\ell-1}, a)$, for each value of $\ell$ and $a$. Since, per the previous paragraph, the algorithm must add large perturbations to $\hat \pi_2$ for $\ell = O(1)$, we will ultimately reach a contradiction. There are many details we have omitted from this high-level description, such as ensuring that $\MD$ satisfies the requisite coverage property with respect to $\pi^\st$, and the fact that we wish to allow the algorithm to output arbitrary choices of $\hat \pi$, and not just perturbed linear policies -- the full details are in \arxiv{\cref{sec:lb}}\rlc{\cref{sec:lb-proof-rlc}}.

}

\arxiv{\paragraph{Organization of the paper.} In \cref{sec:prelim}, we introduce preliminaries. In \cref{sec:lbc-structural}, we prove \cref{thm:pert-linear-informal}, and in \cref{sec:offline-alg} we use this result to prove our main upper bound, \cref{thm:main-informal}. Finally, in \cref{sec:lb} we prove our lower bound, \cref{thm:lb-informal}.
}
\rlc{
\paragraph{Organization of the paper.} In \cref{sec:prelim}, we introduce preliminaries. In \cref{sec:offline-alg} we state and discuss our upper bound, \cref{thm:pacle-ftpl} (the formal version of \cref{thm:main-informal}), and in \cref{sec:lb} we state and discuss our lower bound, \cref{thm:lb-formal} (the formal version of \cref{thm:lb-informal}). \cref{sec:related-work} contains a detailed discussion of related work. The full proof of our upper bound is provided in \cref{sec:lbc-structural,sec:ub-proof-rlc}, and the full proof of our lower bound is provided in \cref{sec:lb-proof-rlc}. Finally, \cref{sec:lemmas,sec:brc} contain additional useful lemmas. 
  }

\section{Preliminaries}
\label{sec:prelim}
We consider the standard setting of a \emph{finite-horizon Markov decision process}, which consists of a tuple $M = (H, \MX, \MA, (P_h\sups{M})_{h=1}^H, (r_h\sups{M})_{h=1}^H, x_1)$, where $H \in \BN$ denotes the \emph{horizon}, $\MX$ denotes the \emph{state set}, $\MA$ denotes the \emph{action set}, $P_h\sups{M}(\cdot \mid x,a) \in \Delta(\MX)$ denote the \emph{transition kernels} (for $h \in [H]$), $r_h\sups{M} : \MX \times \MA \to [0,1]$ denote the \emph{reward functions} (for $h \in [H]$), and $x_1 \in \MX$ denotes the initial state. We omit the superscript $M$ from these notations when its value is clear.

A \emph{Markov policy} (or simply \emph{policy}) $\pi$ is a tuple $\pi = (\pi_1, \ldots, \pi_H)$, where $\pi_h : \MX \to \Delta(\MA)$ for each $h \in [H]$. We let $\PiMarkov$ denote the set of Markov policies. A policy $\pi \in \PiMarkov$ defines a distribution over \emph{trajectories} $(x_1, a_1, r_1, \ldots, x_H, a_H, r_H) \in (\MX \times \MA \times [0,1])^H$, in the following manner: for each $h \in [H]$, given the state $x_h$, an action $a_h$ is drawn according to $a_h \sim \pi_h(x_h)$, a reward $r_h(x_h,a_h)$ is received, and the subsequent state $x_{h+1}$ is generated according to $x_h \sim P_h\sups{M}(\cdot \mid x_h, a_h)$. We use the notation $\E\sups{M, \pi}[\cdot]$ to denote expectaton under the draw of a trajectory from policy $\pi$ in the MDP $M$, and we write $\E\sups{\pi}[\cdot]$ if the value of $M$ is clear.

Fix an MDP $M$. The \emph{$Q$-value function} and \emph{$V$-value function} associated to a policy $\pi \in \PiMarkov$ in MDP $M$ are defined as follows: for $h \in [H], x \in \MX, a \in \MA$,
\begin{align}
Q_h^{\pi}(x,a) := r_h(x,a) + \E^{\pi} \left[ \sum_{g=h+1}^H r_g(x_g, a_g) \mid (x_h, a_h) = (x,a) \right], \qquad V_h^\pi(x) := Q_h^\pi(x, \pi_h(x))\nonumber,
\end{align}
where for a function $Q : \MX \times \MA \to \BR$, we write $Q(x, \pi_h(x)) := \E_{a\sim \pi_h(x)}[Q(x,a)]$. We use the convention that all rewards and value functions evaluate to 0 at step $H+1$.

Given $h \in [H]$, a function $Q_{h+1} : \MX \times \MA \to \BR$, and a policy $\pi \in \PiM$, the \emph{Bellman backup of $Q_{h+1}$ with respect to $\pi$} is the function $Q_h : \MX \times \MA \to \BR$ defined by $Q_h(x,a) := r_h(x,a) + \E_{x' \sim P_h(x,a)}[Q_{h+1}(x', \pi_{h+1}(x'))]$. It is straightforward to see that, for any $\pi \in \PiM$, $Q_h^\pi$ is the Bellman backup of $Q_{h+1}^\pi$ with respect to $\pi$, for each $h \in [H]$. 

The \emph{optimal policy} $\piopt \in \PiM$ is defined as $\piopt := \argmax_{\pi \in \PiM} V_1^\pi(x_1)$. 

\subsection{Inherent Bellman Error}
\label{sec:ibe}
MDPs encountered in practical scenarios tend to have enormous state and action spaces. To address this challenge, it is common to use \emph{function approximation} assumptions, which consider function classes $\MQ_h \subset \BR^{\MX \times \MA}$ and posit that the value functions for the optimal policy belong to $\MQ_h$, i.e., $Q_h^{\piopt} \in \MQ_h$ for $h \in [H]$. As our goal is to obtain provable \emph{end-to-end computationally efficient algorithms} for offline RL, without reliance on intractable regression oracles, we focus on the setting when the classes $\MQ_h$ are linear. In particular, for some dimension $d \in \BN$ together with \emph{known} feature mappings $\phi_h : \MX \times \MA \to \BR^d$, we take $\MQ_h := \{ (x,a) \mapsto \lng \phi_h(x,a), w \rng \ : \ w \in \BR^d \}$. For simplicity of notation, we use the convention that $\phi_{H+1}(x,a)$ is the all-zeros vector for each $x,a$. 

Generally speaking, prior work on offline RL in the linear setting \citep{zanette2021provable,xie2021bellman,cheng2022adversarially,gabbianelli2023offline,nguyen2023on} interprets elements of the classes $\MQ_h$ as approximations of the $Q$-value functions $Q_h^\pi$, for all $\pi$ belonging to some subset $\Pi \subseteq \PiM$ consisting of policies whose values the learning algorithm wishes to compete with. Accordingly, these works make the assumption of \emph{$\Pi$-realizability}, which posits that for all $\pi \in \Pi$ and $h \in [H]$, $Q_h^\pi \in \MQ_h$. This assumption is natural in that it is sufficient for correctness of the \emph{Least-Squares Policy Iteration (LSPI)} algorithm \citep{lagoudakis2003least} (which assumes knowledge of the transitions and rewards of $M$) for finding an optimal policy \citep[Theorem D.1]{du2020is}. However, as shown in \cite{wang2021what,amortila2020variant,zanette2021exponential}, $\Pi$-realizability alone is insufficient for offline RL to succeed with polynomial sample complexity under our desired single-policy coverage condition.\footnote{In fact, these results rule out offline RL even under a stronger \emph{all-policy} coverage condition.} Thus, it is typical to make the stronger assumption of \emph{$\Pi$-Bellman restricted closedness} \citep{zanette2021provable} (also known as \emph{policy completeness}), namely that for all $\pi \in \Pi$ and $Q_{h+1} \in \MQ_{h+1}$, there is some $Q_h \in \MQ_h$ so that $Q_h(x,a) = \E_{x' \sim P_h(x,a)}[r_h(x,a) + Q_{h+1}(x', \pi_{h+1}(x'))]$.

\paragraph{Inherent Bellman error.} Bellman-restricted closedness is an unwieldy assumption in that it requires quantifying over both a policy and a value function at step $h+1$. In this work, we study offline RL under the alternative assumption of \emph{low inherent Bellman error} \citep{zanette2020learning}, as defined in \cref{asm:ibe} below. For $h \in [H]$, write $\MB_h := \{ w \in \BR^d :\ |\lng \phi_h(x,a), w \rng | \leq 1 \ \forall (x,a) \in \MX \times \MA \}$. 

\begin{assumption}[Low Inherent Bellman Error; \cite{zanette2020learning}]
  \label{asm:ibe}
  We say that a MDP $M$ has \emph{inherent Bellman error $\epbe$} if for each $h \in [H]$, there is a mapping $\MT_h : \MB_{h+1} \to \MB_h$ so that
  \begin{align}
\sup_{\theta \in \MB_{h+1}} \sup_{(x,a) \in \MX \times \MA} \left| \lng \phi_h(x,a), \MT_h\theta \rng - \E_{x' \sim P_h(x,a)}\left[ r_h(x,a) + \max_{a' \in \MA} \lng \phi_{h+1}(x', a'), \theta \rng \right] \right| \leq \frac{\epbe}{2}.\label{eq:ibe-def}
  \end{align}
  If $M$ has inherent Bellman error $\epbe = 0$, then we say that $M$ is \emph{linear Bellman complete}. 
\end{assumption}
The assumption of linear Bellman completeness is sufficient for correctness of the \emph{Least-Squares Value Iteration (LSVI)} algorithm \citep{munos2008finite,munos2005error}, which assumes knowledge of the transitions and rewards of $M$. This fact has made it a popular assumption under which to study \emph{online RL}, for which algorithms typically proceed via approximate variants of LSVI \citep{zanette2020learning,zanette2020provably}. In contrast, as recent offline RL algorithms typically bear more resemblance to \emph{LSPI} \citep{zanette2021provable,xie2021bellman,cheng2022adversarially,gabbianelli2023offline,nguyen2023on}, offline RL has not previously been studied under linear Bellman completeness as opposed to Bellman restricted closedness. As discused in \arxiv{\cref{sec:lbc-structural}}\rlc{\cref{sec:offline-alg}}, one of the contributions of this work is to draw connections between these two types of assumptions. 

It is convenient to separate the components of $\MT_h$ capturing the rewards at step $h$ and the transitions at step $h$, as follows: an immediate consequence of \cref{asm:ibe} (see \citet[Proposition 2]{zanette2021provable}) is that there are mappings
  $\bellc : \MB_{h+1} \ra \MB_h$ for each $h \in [H-1]$ and a vector $\theta_h^{\mathsf{r}} \in \MB_h$ for each $h \in [H]$ so that the below inequalities hold:
  \begin{align}
    \sup_{\theta \in \MB_{h+1}} \sup_{(x,a) \in \MX \times \MA} \left| \lng \phi_h(x,a), \bellc \theta \rng - \E_{x' \sim \BP_h(x,a)}\left[ \max_{a' \in \MA} \lng \phi_{h+1}(x', a'), \theta \rng \right] \right|  \leq &  \epbe\label{eq:ibe-def-p}\\
    \sup_{(x,a) \in \MX \times \MA} \left| r_h(x,a) - \lng \phi_h(x,a), \theta_h^{\mathsf{r}} \rng \right| \leq &  \epbe\label{eq:ibe-def-r}.
  \end{align}
It was observed in \citet[Proposition 5]{zanette2020learning} that, in general, the assumptions of linear Bellman completeness and $\PiM$-realizability are incomparable, in that neither one implies the other. Moreover, it is straightforward to see that linear Bellman completeness is a strictly weaker condition than $\PiM$-Bellman restricted closedness. For an arbitrary subset of policies $\Pi \subset \PiM$, linear Bellman completeness may be incomparable to the assumption of $\Pi$-Bellman restricted closedness.

  Finally, we make the following standard boundedness assumptions.
\begin{assumption}[Boundedness]
  \label{asm:boundedness}
  We assume the following:
  \begin{enumerate}
    \item For all $h \in [H], x \in \MX, a \in \MA$, we have $\| \phi_h(x,a) \|_2 \leq 1$. 
    \item For some parameter $\Bbnd \in \BR_+$: for all $w_h \in \MB_h$, it holds that $\| w_h \|_2 \leq \Bbnd$.
    \item For all $h \in [H]$, $\| \theta_h^\mathsf{r} \|_2 \leq 1$ (and hence $\sup_{x,a,h} | r_h(x,a)| \leq 1$).  %
    \end{enumerate}
  \end{assumption}
The assumption that $\| \phi_h(x,a) \|_2 \leq 1$ together with the definition of $\MB_h$ ensures that $\MB_h$ contains a ball of radius 1 centered at the origin.

\subsection{Perturbed linear policies}
Given $w \in \BR^d, h \in [H], x \in \MX$, define $\MA_{h,w}(x) := \argmax_{a \in \MA} \lng w, \phi_h(x,a) \rng \subset \MA$, where $\argmax$ is interpreted as the set of all actions maximizing $\lng w, \phi_h(x,a) \rng$. 
\begin{defn}[Perturbed linear policies]
  \label{def:plinear}
  For $\sigma > 0$, $h \in [H]$ and $w \in \BR^d$, define $\pi_{h,w,\sigma} : \MX \ra \Delta(\MA)$ by
  \begin{align}
\pi_{h,w,\sigma}(x)(a) = \Pr_{\theta \sim \MN(w, \sigma^2 \cdot I_d)} \left( a \in \MA_{h,\theta}(x)\right)\nonumber.
  \end{align}
  In words, to draw an action $a \sim \pi_{h,w,\sigma}(x)$, we draw $\theta \sim \MN(w, \sigma^2 \cdot I_d)$ and then play $\argmax_{a' \in \MA} \lng \phi_h(x,a'), \theta \rng$. We extend to the case that $\sigma = 0$ by taking a limit, i.e., define $\pi_{h,w,0}(x)(x)(a) := \lim_{\sigma \downarrow 0} \pi_{h,w,\sigma}(x)(a)$ (it is straightforward to see that this limit is well-defined). 
  Given $\sigma \geq 0$, we denote the set of all $\pi_{h,w,\sigma'}$, where $w \in \BR^d, \sigma' \geq 0$ satisfy $\sigma' / \| w \|_2 \geq\sigma$, by $\Pilinp_h$, and $\Pilinp := \prod_{h=1}^H \Pilinp_h$. %
  Moreover, we write $\Pilinpp_h :=  \Pilinp[0]_h =\bigcup_{\sigma \geq 0} \Pilinp_h$ and $\Pilinpp := \Pilinp[0] =\bigcup_{\sigma \geq 0} \Pilinp$. 
\end{defn}Note that, for any $c > 0$, $\pi_{h,cw, \sigma} = \pi_{h, w, \sigma/{c}}$. We refer to the policies in $\Pilinpp$ as \emph{perturbed linear policies}.  %
Given a (possibly randomzied) policy $\pi_h : \MX \to \Delta(\MA)$, we use the convention that $\phi_h(x, \pi_h(x))$ refers to $\E_{a \sim \pi_h(x)}[\phi_h(x,a)]$. 

\paragraph{Gaussian smoothing.} For $\theta \in \BR^d$ and $\sigma > 0$, we write
\begin{align}
\MN_\sigma(\theta)  := \MN(0, \sigma^2 \cdot I_d)(\theta) = \frac{1}{(2\pi)^{d/2} \sigma^d} \cdot \exp \left( -\frac{1}{2\sigma^2} \| \theta \|_2^2 \right)\nonumber.
\end{align}
Furthermore, for $f : \BR^d \ra \BR$, we write $\Sm f (\theta)$ to denote the convolution of $f$ with $\MN_\sigma$, namely
\begin{align}
\Sm f(\theta) := \int_{\BR^d} f(z) \MN_\sigma(\theta - z) dz = \int_{\BR^d} f(\theta-z) \MN_\sigma(z) dz = \E_{z \sim \MN(0, \sigma^2 \cdot I_d)}[f(\theta -z)]\label{eq:convolve-def}.
\end{align}

\subsection{The offline learning problem}
\label{sec:offline-problem}
In the \emph{offline learning model}, the algorithm is not allowed to interact with the environment. Rather, it is given a dataset $\MD$ consisting of tuples $(h,x,a,r,x')$, where $r = r_h(x,a)$ and $x' \sim P_h(\cdot | x,a)$. We allow the values of $h,x,a$ in the dataset $\MD$ to be chosen in an arbitrary adaptive manner, as formalized by the following assumption: 
\begin{assumption}
\label{asm:offline-data-assumption}
We assume the dataset $\MD = \{(h_i, x_i, a_i, r_i, x_i') \}_{i=1}^n$ is drawn from a distribution satisfying the following conditions. For $i \in [n]$, let $\MF_i$ denote the sigma-algebra generated by $\{ (h_j, x_j, a_j, r_j, x_j') \}_{j=1}^{i-1} \cup \{(h_i, x_i, a_i) \}$. %
We assume that, for each $i \in [n]$, conditioned on $\MF_i$, the reward $r_i$ is equal to $r_{h_i}(x_i, a_i) = \lng \phi_{h_i}(x_i, a_i), \theta_{h_i}^{\mathsf r} \rng$, and $x_i' \sim P_{h_i}(x_i, a_i)$. 
\end{assumption}
We remark that \cref{asm:offline-data-assumption} is essentially the same as Assumption 1 of \cite{zanette2021provable}. 
Based on the dataset $\MD$, the algorithm must output a policy $\hat\pi$ whose value, $V_1^{\hat \pi}(x_1)$ is as large as possible. Of course, it may be impossible to make $V_1^{\hat \pi}(x_1)$ very large if the dataset $\MD$ does not include states $(h_i, x_i, a_i)$ which explore certain directions of the feature space $\BR^d$. A large number of conditions have been proposed in the offline RL literature which capture the degree to which $\MD$ exhibits good ``coverage'' properties of the state space. Our bounds depend on one of the weakest such conditions, namely the following variant of \emph{single-policy coverage}, which is identical to that in \cite{zanette2021provable}. 
\begin{defn}[Coverage parameter]
  \label{def:coverage}
  Given a dataset $\MD$ as in \cref{asm:offline-data-assumption}, we define $\Sigma_h := \sum_{i : h_i = h} \phi_{h_i}(x_i, a_i) \phi_{h_i}(x_i, a_i)^\t$. For a policy $\pi \in \PiM$, its \emph{coverage parameter} for the dataset $\MD$ is
  \begin{align}
\Cvg_{\MD, \pi} := \sum_{h=1}^H \left\| \E^\pi[\phi_h(x_h, a_h)] \right\|_{n\Sigma_h^{-1}}\nonumber.
  \end{align}
\end{defn}
 In words, $\Cvg_{\MD, \pi}$ measures the degree to which an average feature vector drawn from $\pi$ at each step $h$ lines up with directions spanned by feature vectors in $\MD$. %
We refer to \citet[Section 6]{gabbianelli2023offline}, \citet[Section 4]{nguyen2023on}, and \citet{jiang2023} for a detailed comparison between $\Cvg_{\MD, \pi}$ and other coverage parameters considered in prior work. As discussed there, assuming boundedness of $\Cvg_{\MD, \pi}$ is essentially the \emph{weakest} coverage assumption in the literature, as many previous works (e.g., \cite{jin2021pessimism}) require instead boundedness of $\E^\pi[ \| \phi_h(x_h, a_h) \|_{n\Sigma_h^{-1}} ]$, where the norm is \emph{inside} the expectation.

\arxiv{\section{Low Inherent Bellman Error: Structural Properties}
\label{sec:lbc-structural}

In this section, we prove \cref{thm:pert-linear-informal} (restated below formally as \cref{cor:ibe-linear}), which shows that the Bellman backup of any linear function at step $h+1$, with respect to any perturbed linear policy, is an approximately linear function at step $h$.  %
The main ingredient in the proof of \cref{cor:ibe-linear} is \cref{lem:ibe-linear} below, which shows that the expected feature vector induced by a perturbed linear policy at step $h+1$ is a linear transformation of the state-action feature vector at step $h$. \cref{lem:ibe-linear} is a generalization of Lemma 4.3 of \cite{golowich2024linear}, which treats the special case of $\epbe = 0$ and used the result to develop an efficient algorithm for the related setting of \emph{online RL} under linear Bellman completeness. 
\begin{lemma}
  \label{lem:ibe-linear}
  Suppose that the MDP $M$ has inherent bellman error bounded by $\epbe$, and fix $\sigma > 0$. Then for each $h \in [H]$ and $w \in \BR^d$, there is a linear map $L_{h,w,\sigma} : \BR^d \ra \BR^d$ so that for all $(x,a) \in \MX \times \MA$, 
  \begin{align}
    & \left\| L_{h,w,\sigma}^\t \cdot \phi_h(x,a) - \E_{x' \sim \BP_h(x,a)}[\phi_{h+1}(x', \pi_{h+1,w,\sigma}(x'))] \right\|_2 \nonumber\\
    \leq & C_{\ref{lem:ibe-linear}} \epbe d^{3/2} \cdot \left(\sqrt{d \log(d/(\epbe \sigma))} + \frac{1}{\sigma} \right) \nonumber, %
  \end{align}
  for some constant $C_{\ref{lem:ibe-linear}}$. Moreover, for any $w,w', \sigma, \sigma'$ for which $\pi_{h+1,w,\sigma}(x') = \pi_{h+1,w',\sigma'}$ for all $x' \in \MX$, we have $L_{h,w,\sigma} = L_{h,w',\sigma'}$. 
\end{lemma}
Given a perturbed linear policy $\pi_{h+1} \in \Pilinpp_{h+1}$, so that $\pi_{h+1} = \pi_{h+1, \theta, \sigma}$ for some $\theta \in \BR^d, \sigma > 0$, we define $\MT_h^{\pi_{h+1}} w := \theta_h^{\mathsf{r}} + L_{h,\theta,\sigma} \cdot w$, where $L_{h,\theta,\sigma}$ is the map of \cref{lem:ibe-linear} and where $\theta_h^{\mathsf{r}}$ was defined in \cref{asm:ibe}.  Note that $\MT_h^{\pi_{h+1}}$ is well-defined in the sense that it only depends on $\pi_{h+1}$ (and not on the particular choice of $\theta,\sigma$), since for any $\theta, \sigma, \theta', \sigma'$ satisfying $\pi_{h+1,\theta,\sigma} = \pi_{h+1,\theta',\sigma}$, we have by \cref{lem:ibe-linear} that $L_{h,\theta,\sigma} = L_{h,\theta', \sigma'}$. We will at times abuse notation by writing $\MT_h^\pi := \MT_h^{\pi_{h+1}}$ for a policy $\pi \in \Pilinpp$. Given \cref{lem:ibe-linear}, the proof of \cref{cor:ibe-linear}, stated below, is straightforward. %
\begin{corollary}
  \label{cor:ibe-linear}
  Suppose that $\pi \in \Pilinp[\sigma]$. Then for all $h \in [H-1]$, $w \in \BR^d$, and $(x,a) \in \MX \times \MA$,
  \begin{align}
\left| \lng \phi_h(x,a), \MT_h^\pi w \rng - \left( r_h(x,a) + \E_{x' \sim P_h(x,a)}[\lng \phi_{h+1}(x', \pi_{h+1}(x')), w \rng ]\right) \right| \leq \nrm{w} \cdot \zeta_\sigma\label{eq:phi-t-ineq},
  \end{align}
  where $\zeta_\sigma = C_{\ref{cor:ibe-linear}} \epbe d^{3/2} \cdot \left( \sqrt{d \log(d/(\epbe\sigma))} + \frac{1}{\sigma} \right)$, for some constant $C_{\ref{cor:ibe-linear}}$. Moreover, if $\zeta_\sigma \leq 1$, then $\MT_h^\pi w \in(1+ 2 \| w \|_2) \cdot \MB_h$. 
\end{corollary}
\begin{proof}
  The inequality \cref{eq:phi-t-ineq} follows directly from \cref{lem:ibe-linear} and the definition of $\MT_h^\pi$, as well as \cref{eq:ibe-def-r}. To see that $\MT_h^\pi w \in (1+2 \| w \|_2) \cdot \MB_h$, we note that, by \cref{eq:phi-t-ineq}, for all $(x,a) \in \MX \times \MA$,
  \begin{align}
| \lng \phi_h(x,a), \MT_h^\pi w \rng | \leq 1 + \| w \|_2 + \| w \|_2 \cdot \zeta_\sigma \leq 1 + 2 \| w \|_2\nonumber,
  \end{align}
  since $\zeta_\sigma \leq 1$. 
\end{proof}

As a further corollary of \cref{cor:ibe-linear}, the $Q$-function for a perturbed linear policy is linear. %
\begin{corollary}
  \label{cor:qlin}
  Suppose that $M$ is linear Bellman complete, $\sigma > 0$, and that $\pi \in \Pilinp[\sigma]$. Then for each $h \in [H]$, there is a vector $w_h^\pi \in 2H \cdot \MB_h \subset \BR^d$ so that for all $(x,a) \in \MX \times \MA$,
  \begin{align}
\left|    Q_h^\pi(x,a) - \lng w_h^\pi, \phi_h(x,a) \rng\right| \leq 3(H+1-h)H B \cdot  \zeta_\sigma,\nonumber
  \end{align}
  where $\zeta_\sigma$ is as defined in \cref{cor:ibe-linear}. 
  Moreover, if $3HB\zeta_\sigma \leq 1$, then $w_h^\pi \in 2H \cdot \MB_h$ and  $\| w_h^\pi \|_2 \leq 2HB$. 
\end{corollary}
\begin{proof}
  We use reverse induction on $h$; the base case $h=H$ is immediate, so suppose the statement holds at step $h+1$. 
  By \cref{asm:ibe} and \cref{cor:ibe-linear}, we have that for all $(x,a) \in \MX \times \MA$, 
  \begin{align}
    & \left|r_h(x,a) + \E_{x' \sim P_h(x,a)}[\lng \phi_{h+1}(x', \pi_{h+1}(x')), w_{h+1}^\pi \rng -  \lng \phi_h(x,a),  \MT_h^\pi w_{h+1}^\pi \rng   \right| \nonumber\\
    \leq & \epbe + \zeta_\sigma \cdot \| w_{h+1}^\pi \|_2\nonumber.
  \end{align}
  Let us write $w_h^\pi :=  \MT_h^\pi w_{h+1}^\pi$, 
  Since $Q_h^\pi(x,a) = r_h(x,a) + \E_{x' \sim P_h(x,a)} [Q_{h+1}^\pi(x, \pi_{h+1}(x))]$, the inductive hypothesis then gives us that
  \begin{align}
    & |Q_h^\pi(x,a) - \lng w_h^\pi, \phi_h(x,a) \rng |\nonumber\\
    \leq & \E_{x' \sim P_h(x,a)} [| Q_{h+1}^\pi(x', \pi_{h+1}(x')) - \lng w_{h+1}^\pi, \phi_{h+1}(x' ,\pi_{h+1}(x')) \rng |] \nonumber\\
    &+ \left|r_h(x,a) + \E_{x' \sim P_h(x,a)}[\lng \phi_{h+1}(x', \pi_{h+1}(x')), w_{h+1}^\pi \rng -  \lng \phi_h(x,a), w_h^\pi \rng   \right|\nonumber\\
    \leq & 3(H-h)HB \cdot \zeta_\sigma + (\epbe + \zeta_\sigma \cdot \| w_{h+1}^\pi \|_2)\leq 3(H+1-h)HB \cdot \zeta_\sigma\label{eq:qhpi-whpi-diff},
  \end{align}
  where the final inequality uses that $\epbe \leq \zeta_\sigma$. 
  To see the upper bound on $\| w_h^\pi \|_2$, note that, by definition of $Q_h^\pi$ and \cref{eq:qhpi-whpi-diff}, we have $|\lng w_h^\pi, \phi_h(x,a) \rng |\leq  H + 3(H+1-h) HB \zeta_\sigma \leq 2H$ for all $x,a,h$, where we have used that $3HB \zeta_\sigma \leq 1$. Then it follows that $\| w_h^\pi \|_2 \leq 2HB$ by \cref{asm:boundedness}. 
\end{proof}

\begin{proof}[Proof of \cref{lem:ibe-linear}]
  We may assume without loss of generality that $\| w \|_2 = 1$, by increasing $\sigma$ by a factor of $1/\| w \|_2$. 
  Fix $h \in [H]$. For $x' \in \MX$ and $w \in \BR^d$, define $V(x', w) := \max_{a' \in \MA} \lng \phi_{h+1}(x', a'), w \rng$. For $x \in \MX, a \in \MA, w \in \BR^d$, define $Q(x,a,w) = \E_{x' \sim P_h(x,a)}[V(x', w)]$. Since $\MA, \MX$ are assumed to be countable, the mapping $w \mapsto Q(x,a,w)$ is piecewise linear with countably many pieces (and is also continuous). Next, \cref{asm:ibe} together with the fact that $\{ w \in \BR^d :\ \| w \|_2 \leq 1 \} \subset \MB_{h+1}$ gives us that
  \begin{align}
\sup_{\| w \|_2 \leq 1} \sup_{(x,a) \in \MX \times \MA} \left| \lng \phi_h(x,a), \bellc w \rng - Q(x,a,w) \right| \leq \epbe\label{eq:ibe-use}.
  \end{align}

Next,  we may choose $k \leq d$ and  $(x_1, a_1), \ldots, (x_k, a_k) \in \MX \times \MA$ so that $\{ (\phi_h(x_i, a_i) \}_{i=1}^k$  forms a barycentric spanner of $\{ \phi_h(x,a) \}_{(x,a) \in \MX \times \MA}$, and so that $\phi_h(x_i, a_i)$, $1 \leq i \leq k$, are linearly independent. By linear independence of $\phi_h(x_i, a_i)$, for each $w \in \BR^d$, there is a matrix $L_{h,w,\sigma} \in \BR^{d \times d}$ so that, for all $i \in [k]$,
  \begin{align}
L_{h,w,\sigma}^\t \cdot \phi_h(x_i, a_i) = \grad \Sm Q(x_i, a_i, w)\label{eq:smooth-spanner-elts}.
  \end{align}

  Next, for each $(x,a) \in \MX \times \MA$ and $w \in \BR^d$, we have
  \begin{align}
    \grad_w \Sm Q(x,a,w)    =& \grad_w \E_{z \sim \MN(0, \sigma^2 \cdot I_d)} \E_{x' \sim P_h(x,a)}[Q(x', w)]\nonumber\\
    =&  %
    \E_{z \sim \MN(0, \sigma^2 \cdot I_d)} \E_{x' \sim P_h(x,a)} [\grad_w V(x', w+z)] \nonumber\\
    =& \E_{x' \sim P_h(x,a)} \E_{z \sim \MN(0, \sigma^2 \cdot I_d)}[ \phi_{h+1}(x', \pi_{h+1,w+z}(x'))]\nonumber\\
    =& \E_{x' \sim P_h(x,a)} [\phi_{h+1}(x', \pi_{h+1,w,\sigma}(x'))]\label{eq:smooth-xa},
  \end{align}
  where the second equality uses the dominated convergence theorem  together with the fact that for each $x' \in \MZ, z \in \BR^d$, $w \mapsto V(x', w + z)$ is piecewise linear with finitely many pieces, with bounded derivative, i.e., $\| \grad_w V(x', w+z) \| \leq \max_{a' \in \MA} \| \phi_{h+1}(x', a') \|_2 \leq 1$. (Note that $\Sm Q(x,a,w)$ is infinitely differentiable since we can write $\Sm Q (x,a,w) = \int_{\BR^d} Q(x,a,w) \MN_\sigma(w -z)dz$ and since $\MN_\sigma(w - z)$ is infinitely differentiable in $w$.) Using \cref{eq:smooth-xa} with $(x,a) = (x_i, a_i)$ (for each $i \in [k]$), we see that $L_{h,w,\sigma} = L_{h,w',\sigma'}$ for all $w,w',\sigma, \sigma'$ satisfying $\pi_{h+1,w,\sigma}(x') \pi_{h+1,w',\sigma'}(x')$ for all $x' \in \MX$.

  Fix any $(x,a) \in \MX \times \MA$. By the barycentric spanner property, there are coefficients $\alpha_1, \ldots, \alpha_k \in [-1,1]$ (depending on $x,a$) so that $\phi_h(x,a) = \sum_{i=1}^k \alpha_i \cdot \phi_h(x_i, a_i)$. It therefore follows that, for all $w \in \MB_{h+1}$, 
  \begin{align}
    & \left| Q(x,a,w) - \sum_{i=1}^k \alpha_i Q(x_i, a_i, w) \right|\nonumber\\
    \leq & | Q(x,a,w) - \lng \phi_h(x,a), \bellc w \rng | + \left| \lng \phi_h(x,a), \bellc w \rng - \sum_{i=1}^k \alpha_i \lng \phi_h(x_i, a_i), \bellc w \rng \right| \nonumber\\
    &+ \left| \sum_{i=1}^k \alpha_i \cdot (Q(x_i, a_i, w) - \lng \phi_h(x_i, a_i), \bellc w \rng )\right|\nonumber\\
    \leq & | Q(x,a,w) - \lng \phi_h(x,a), \bellc w \rng | + \sum_{i=1}^k | Q(x_i, a_i, w) - \lng \phi_h(x_i, a_i), \bellc w \rng|\nonumber\\
    \leq & (d+1) \epbe,\label{eq:q-bs-ep-bound}
  \end{align}
  where the second inequality uses that $\lng \phi_h(x,a), \bellc w \rng = \sum_{i=1}^k  \alpha_i \lng \phi_h(x_i, a_i), \bellc w \rng$ and the final equality uses \cref{eq:ibe-use} applied to each of the tuples $(x, a), (x_1, a_1), \ldots, (x_k, a_k)$. %

  Next we apply \cref{lem:bound-grad-smooth} with \begin{align}f(w) =& Q(x,a,w) - \sum_{i=1}^k \alpha_i Q(x_i, a_i, w)\nonumber\\\ep = & (1 + 100 \sigma \sqrt{d \log(2d/(\epbe\sigma))})(d+1)\epbe\nonumber, \end{align}
  $D = 2d$, $B = 1$, and $\MB = \{ w \in \BR^d : \ \| w \|_2 \leq 1\}$. Since, for all $x \in \MX, a \in \MA, w \in \BR^d$, $| Q(x,a,w) | \leq \| w\|_2$,  we have $|f(w)| \leq (d+1) \cdot \| w \|_2 \leq 2d \| w \|_2$ for $w \in \BR^d$, verifying that the choice of $D=2d$ is admissible. Moreover, the set of $z$ with $\dist(z, \MB) \leq 100 \sigma \sqrt{d \log(2d/(\ep\sigma))}$ is contained in $(1 + 100 \sigma \sqrt{d \log(2d/\epbe\sigma)}) \cdot \MB$, since $\ep > \epbe$ and $\MB$ is a unit ball. %
  Therefore, scaling \cref{eq:q-bs-ep-bound}  verifies that for all $w$ with $\dist(w,\MB) \leq 100 \sigma \sqrt{d \log(2d/(\ep\sigma))}$,  %
  \begin{align}
\left| Q(x,a,w) - \sum_{i=1}^k \alpha_i Q(x_i, a_i, w) \right| \leq & (1 + 100 \sigma \sqrt{d \log(2d/(\epbe\sigma))}) \cdot (d+1) \cdot \epbe = \ep\nonumber.
  \end{align}
  Then the guarantee of \cref{lem:bound-grad-smooth} gives that for some constant $C > 0$, for all $w$ with $\| w \|_2 \leq 1$,
  \begin{align}
\left\| \grad \Sm Q(x,a,w) - \sum_{i=1}^k \alpha_i \cdot \grad \Sm Q(x_i, a_i, w) \right\|_2 \leq & \frac{C \ep \sqrt{d}}{\sigma}\nonumber. %
  \end{align}
  By our choice of $\alpha_i$ and \cref{eq:smooth-spanner-elts}, for all $w \in \BR^d$,
  \begin{align}
\sum_{i=1}^k \alpha_i \cdot \grad \Sm Q(x_i, a_i, w) = L_{h,w,\sigma}^\t \cdot \sum_{i=1}^k \alpha_i \cdot \phi_h(x_i, a_i) = L_{h,w,\sigma}^\t \cdot \phi_h(x,a)\nonumber.
  \end{align}
  Combining the above with \cref{eq:smooth-xa} and the definition of $\ep$ gives that, for some constants $C, C'$, for $\| w \|_2 \leq 1$, 
  \begin{align}
    & \left\|  \E_{x' \sim P_h(x,a)} [\phi_{h+1}(x', \phi_{h+1,w,\sigma}(x'))] - L_{h,w,\sigma}^\t \cdot \phi_h(x,a) \right\|_2 \nonumber\\
    \leq & \epbe \cdot  \frac{C \cdot  (1 + 100 \sigma \sqrt{d \log(2d/(\epbe\sigma))})(d+1) \cdot \sqrt{d}}{\sigma} \nonumber\\
    \leq & C' \epbe d^{3/2} \cdot \left(\sqrt{d \log(d/(\epbe \sigma))} + \frac{1}{\sigma} \right)\nonumber,
  \end{align}
  as desired.
\end{proof}

\paragraph{Bounding the gradient of a Gaussian convolution.}

For a subset $\MB \subset \BR^d$ and $z \in \BR^d$, we write $\dist(z, \MB) := \inf \{ \| w - z \|_2 \ : \ w \in \MB \}$. 
\begin{lemma}
  \label{lem:bound-grad-smooth}
There is a constant $C > 0$ so that the following holds.  Let $\sigma, \ep \in (0,1/2)$ and $B, D \geq 1$ be given, and suppose that $\MB \subset \BR^d$ is a set with nonempty interior $\MB^\circ$, and for which $\max_{\theta \in \MB} \| \theta \| \leq B$.  Suppose $f : \BR^d \ra \BR^d$ satisfies $|f(z)| \leq \ep$ for all $z$ with $\dist(z, \MB) \leq 100 \sigma \sqrt{d \log (BD/(\ep\sigma))}$, as well as $|f(z)| \leq D\| z \|_2$ for all $z \in \BR^d$, for some $D > 0$. Then for all $z \in \MB^\circ$, 
  \begin{align}
\| \grad \Sm f (z) \|_2 \leq \frac{C\ep\sqrt{d}}{\sigma}\nonumber. %
  \end{align}
\end{lemma}
\begin{proof}
  Let us write $\Delta := 100 \sigma \sqrt{d \log (BD/(\ep\sigma))}$. 
  Let $\MB_\Delta := \{ z \in \BR^d \ : \ \dist(z, \MB) \leq \Delta \}$, so that, by assumption, $|f(z)| \leq \ep$ for all $z \in \MB_\Delta$. Then, for any $\theta \in \MB$, 
  \begin{align}
    \| \grad \Sm f(\theta)\|_2 = &\left\| \grad \int_{\BR^d} f(z) \MN_\sigma(\theta -z) dz\right\|_2\nonumber\\
    =& \left\|\int_{\BR^d} f(z) \cdot \grad \MN_\sigma(\theta-z) dz \right\|_2\nonumber\\
    \leq &  \int_{\MB_\Delta} \left\|f(z) \cdot \grad \MN_\sigma(\theta-z)\right\|_2 dz  + \int_{\BR^d \backslash \MB_\Delta}  \left\| f(z) \cdot \grad \MN_\sigma(\theta-z) \right\|_2 dz \label{eq:smooth-f-decompose}.
  \end{align}
  Note that $\grad \MN_\sigma(\theta) = -\frac{\theta}{\sigma^2} \cdot \MN_\sigma(\theta)$. Then we have
  \begin{align}
    \int_{\BR^d \backslash \MB_\Delta}  \left\| f(z) \cdot \grad \MN_\sigma(\theta-z) \right\|_2 dz \leq & \int_{\BR^d \backslash \MB_\Delta} \frac{D\| z \|_2 \cdot \| \theta-z\|_2 }{\sigma^2} \cdot \MN_\sigma(\theta-z) dz\nonumber\\
    \leq & D \int_{\BR^d \backslash \MB_\Delta} \frac{ \| \theta - z \|_2^2 + B \| \theta -z \|_2}{\sigma^2} \cdot \MN_\sigma(\theta-z) dz\label{eq:f-gradgauss-convolve},
  \end{align}
  where the second inequality uses that $\| z \| \leq \| \theta - z \|_2 + B$ since $\theta \in \MB$. 
  
Using the tail bound $\Pr_{z \sim \MN(0, \sigma^2 I_d)} ( \| z \|_2^2 > 2td\sigma^2 ) \leq e^{-td/10}$ for $t \geq 1$ \cite[Lemma 1]{laurent2000adaptive}\footnote{See also \url{https://math.stackexchange.com/questions/2864188/chi-squared-distribution-tail-bound}.}
  it holds that %
  \begin{align}
\int_{\| z \|_2 \geq \Delta} \Delta \| z \|_2 \cdot \MN_\sigma(z) dz \leq     \int_{\| z \|_2 \geq \Delta} \| z \|_2^2 \cdot \MN_\sigma(z) dz \leq &\Delta^2 \cdot  e^{-\Delta^2/(20\sigma^2)} + \int_{\Delta^2}^\infty e^{-y/(20\sigma^2)}d\sigma \nonumber\\
    =& \Delta^2 \cdot e^{-\Delta^2/(20\sigma^2)} + \frac{1}{20\sigma^2} \cdot e^{-\Delta^2/(20\sigma^2)}\nonumber\\
    \leq & \left(\Delta^2 + \frac{1}{\sigma^2}\right) \cdot e^{-5 d \log(BD/(\ep\sigma))}\nonumber\\
    \leq & \left(10^4 d \log BD/(\ep\sigma) + \frac{1}{\sigma^2} \right) \cdot (\ep\sigma/(BD))^{5d}\nonumber\\
    \leq & \frac{10^4 \ep\sigma^3 \ep}{BD} + \frac{\ep\sigma^3}{BD} = \frac{10001\ep\sigma^3}{BD}\label{eq:layer-cake},
  \end{align}
  where the second inequality uses the layer cake formula and the fact that $\Delta^2/\sigma^2 \geq {2d}$,
  and the final inequality uses that $(\ep\sigma/(BD))^{5d} \cdot d \log (BD/\ep\sigma) \leq (\ep\sigma/BD)^{4d} \cdot d \leq \ep \sigma^3/(BD)$ (since $\ep, \sigma \leq 1/2$) and that $(\ep\sigma/(BD))^{5d}/\sigma^2 \leq \ep\sigma^3/(BD)$.

  Note that for all $z \in \BR^d \backslash \MB_\Delta$, we have that $\| \theta - z \| \geq \Delta$ since $\theta \in \MB$.  Thus, combining \cref{eq:f-gradgauss-convolve,eq:layer-cake}, we have, for some constant $C$, 
  \begin{align}
    \int_{\BR^d \backslash \MB_\Delta} \| f(z) \cdot \grad \MN_\sigma(\theta-z) \|_2 dz \leq & \frac{D}{\sigma^2} \int_{\BR^d\backslash \MB_\Delta} \| \theta-z \|_2^2 \cdot \MN_\sigma(\theta-z)dz + \frac{BD}{\sigma^2} \int_{\BR^d\backslash \MB_\Delta} \| \theta-z \|_2 \cdot \MN_\sigma(\theta-z) dz\nonumber\\
    \leq & \frac{D}{\sigma^2} \cdot \frac{C\ep\sigma^3}{BD} + \frac{BD}{\sigma^2} \cdot \frac{C\ep\sigma^3}{BD\Delta} \leq 2C\ep\label{eq:outside-b-bound},
  \end{align}
  where we have used in the second-to-last inequality that $B \geq 1$ and $\sigma/\Delta \leq 1$.

  Next, we compute
  \begin{align}
    \int_{\MB_\Delta} \| f(z) \cdot \grad \MN_\sigma(\theta-z) \|_2 dz \leq & \int_{\MB_\Delta} |f(z)| \cdot \frac{\| \theta-z \|_2}{\sigma^2} \MN_\sigma(\theta-z) dz \nonumber\\
    \leq & \frac{\ep}{\sigma^2} \int_{\MB_\Delta} \| \theta - z \|_2 \MN_\sigma(\theta-z) dz\nonumber\\
    \leq & \frac{\ep}{\sigma^2} \int_{\BR^d} \| \theta - z \|_2 \MN_\sigma(\theta-z) dz \leq \frac{\ep}{\sigma^2} \cdot\sigma\sqrt{d} = \frac{\ep \sqrt{d}}{\sigma}\label{eq:bdelta-bound},
  \end{align}
  where the final inequality uses the fact that $\E_{Z \sim \MN(0, \sigma^2 \cdot I_d)}[\| Z \|_2] = \sigma \sqrt{d}$. %

  Combining \cref{eq:smooth-f-decompose}, \cref{eq:outside-b-bound}, and \cref{eq:bdelta-bound}, we obtain that, for some constant $C' > 0$,
  \begin{align}
\| \grad \Sm f(\theta) \|_2 \leq 2 C \ep + \ep\sqrt{d}/\sigma \leq \frac{C' \ep \sqrt{d}}{\sigma},\nonumber%
  \end{align}
  which is the desired bound.
\end{proof}

}
\section{The offline actor-critic algorithm}
\label{sec:offline-alg}
In this section, we \arxiv{prove}\rlc{discuss} our main upper bound, \cref{thm:pacle-ftpl}, which shows a performance guarantee for the output policy $\hat \pi$ of the \Actor algorithm (\cref{alg:actor}).
\begin{theorem}
  \label{thm:pacle-ftpl}
  Consider any dataset $\MD$ with $n$ examples drawn according to \cref{asm:offline-data-assumption}, as well as parameters $\epfinal, \delta \in (0,1)$. Suppose $\epbe \leq c_0 (BH)^{-2} d^{-3}$ for some sufficiently small constant $c_0$. Then, for $\eta$ set as prescribed in \cref{def:params-offline}, $\Actor(\MD, \epfinal, \delta,\eta)$ (\cref{alg:actor}) returns a policy $\hat \pi$ so that, with probability $1-\delta$ the following holds: for any $\pi^\st \in \Pi$, 
  \begin{align}
    & V_1^{\pi^\st}(x_1) - V_1^{\hat \pi}(x_1) \nonumber\\
             \leq & O \left(d^{3/2}BH \cdot \epbe^{1/2}\log(1/\epbe) +  \frac{B H d \log^{1/2}(dnBH/(\epfinal\delta))}{\sqrt{n}} \right) \cdot  \left( {H} + \Cvg_{\MD, \pi^\st}\right) + \epfinal
             \nonumber.
  \end{align}
  The overall computational cost of \cref{alg:actor} is bounded above by $\poly(d, H,n, \log(B/\delta), 1/\epfinal)$.
\end{theorem}

\rlc{\subsection{High-level proof overview}}

\begin{algorithm}
  \caption{$\Actor(\MD, \epfinal, \delta, \eta)$}
  \label{alg:actor}
  \begin{algorithmic}[1]\onehalfspacing
    \Require Dataset $\MD = \{ (h_i, x_i, a_i, r_i, x_i') \}_{i=1}^n$; parameters $\epfinal, \delta \in (0,1)$ and $\eta > 0$.
    \State Define $T,  \epapx, \alpha, \beta$ as a function of $n, \epfinal, \delta$ per \cref{def:params-offline}. \label{line:actor-params}
    \For{$1 \leq t \leq T+1$}
    \For{$1 \leq h \leq H$}
    \State Set $\theta_h\^t = \sum_{s=1}^{t-1} w_h\^s$.
    \State Define $\pi_h\^t := \pi_{h, \theta_h\^t, \eta}$. \Comment{\emph{(See \cref{def:plinear}) }}
    \EndFor
  \State \label{line:call-critic} Set $(w_1\^t, \ldots, w_H\^t) = \texttt{Critic}(\MD, \pi\^t,\epapx, \alpha, \beta,\delta/(2T))$, where $\pi\^t = (\pi_1\^t, \ldots, \pi_H\^t)$. \Comment{\emph{\cref{alg:critic}}}
  \EndFor
\State\Return a policy $\hat \pi$ drawn as $\hat \pi \sim \Unif(\{ \pi\^1, \ldots, \pi\^T \})$. %
  \end{algorithmic}
\end{algorithm}

\begin{algorithm}[!h]
  \caption{$\texttt{Critic}(\MD, \pi, \epapx, \alpha,\beta,\delta)$}
  \label{alg:critic}
  \begin{algorithmic}[1]\onehalfspacing
    \Require Dataset $\MD = \{ (h_i, x_i, a_i, r_i, x_i') \}_{i=1}^n$; policy $\pi \in \Pilinpp$; parameters $\epapx, \alpha, \beta > 0$.
    \State For each $h \in [H]$, let $\MI_h \subset [n]$ denote the set of $i$ so that $h_i = h$.
    \State For $h \in [H]$, define $\Sigma_h = \del{\lambda} I_d + \sum_{i \in \MI_h} \phi_h(x_i, a_i) \phi_h(x_i, a_i)^\t$. \label{line:define-sigmah}
    \State \label{line:estimate-phi} For $i \in [n]$, set $\hat \phi_i^\pi \gets \EstFeature(x_i', \pi, h_i+1,\epapx, \delta/n)$. \Comment{\emph{\cref{alg:estfeature}}} %
\State Solve the following convex program with variables $w, \xi \in \BR^{dH}$:
  \begin{subequations}
    \label{eq:critic-program}
  \begin{align}
    & \min \lng w_1, \phi_1(x_1, \pi_1(x_1)) \rng \label{eq:critic-a}\\
    \mbox{ s.t. } & w_h = \xi_h + \Sigma_h^{-1} \sum_{i \in \MI_h} \phi_h(x_i, a_i) \cdot \left( r_i + \lng \hat \phi_i^\pi, w_{h+1} \rng \right) \quad \forall h \in [H] \label{eq:critic-b}\\
    &  \| \xi_h \|_{\Sigma_h}^2 \leq \alpha^2 \quad \forall h \in [H] \label{eq:critic-c}\\
    & \| w_h \|_{2}^2 \leq \beta^2 \quad \forall h \in [H]\label{eq:critic-d}.
  \end{align}
  \end{subequations}
\State \Return the solution $w = (w_1, \ldots, w_H) \in \BR^{dH}$ of the convex program. 
  \end{algorithmic}
\end{algorithm}

\begin{algorithm}[!h]
  \caption{$\EstFeature(x, \pi, h, \epapx, \delta)$}
  \label{alg:estfeature}
  \begin{algorithmic}[1]\onehalfspacing
    \Require State $x \in \MX$, policy $\pi \in\Pilinpp$, step $h \in [H]$, error $\epapx \in (0,1)$, failure probability $\delta \in (0,1)$.
    \State Choose $w \in \BR^d, \sigma > 0$ so that $\pi_h = \pi_{h,w,\sigma}$ \Comment{\emph{(This is possible by definition of $\Pilinpp$)}}.
    \State Choose $N \gets 2\epapx^{-2} \log(2d/\delta)$.
    \State Draw $N$ samples $\theta_1, \ldots, \theta_N \sim \MN(w, \sigma^2 \cdot I_d)$.
  \item \Return $\hat \phi := \frac 1N \sum_{i=1}^N \phi_h(x, \pi_{h, \theta_i}(x))$. 
  \end{algorithmic}
\end{algorithm}

\subsection{Algorithm description}
In our setting, the \Actor algorithm and its associated \Critic (\cref{alg:critic}) function similarly to the \Actor and \Critic algorithms in \cite{zanette2021provable}. For an appropriate choice of the number of iterations $T$, \Actor repeats the following steps $T$ times: at each iteration $t \in [T]$, \Actor lets $\theta_h\^t$ be the sum of vectors $w_h\^s$ for iterations $s < t$, and lets $\pi_h\^t$ be a perturbed linear policy with mean vector $\theta_h\^t$. The vector $\theta_h\^t$ represents an aggregation of the pessimistic estimates of the value function produced by the \Critic at previous iterations. In turn, at time step $t$, the \Critic (\cref{alg:critic}), given $\pi\^t = (\pi_1\^t, \ldots, \pi_H\^t)$ as input, then produces vectors $w_h\^t \in \BR^d$ for $h \in [H]$, which constitute a solution to the optimization problem \cref{eq:critic-program}. The program \cref{eq:critic-program} has as its objective to minimize the value at the initial state (namely, \cref{eq:critic-a}), subject to constraints that force $w_h\^t$ to have Bellman backups with respect to $\pi\^t$ which are consistent with $\MD$ (\cref{eq:critic-b,eq:critic-c,eq:critic-d}).

The main difference between \cref{alg:actor,alg:critic} and the approach in \cite{zanette2021provable} is the choice of the policies $\pi_h\^t$ in the \Actor (\cref{alg:actor}). While we take $\pi_h\^t$ to be a perturbed linear policy corresponding to $\theta_h\^t$ and an appropriate choice of the noise parameter $\eta$, in \cite{zanette2021provable}, $\pi_h\^t$ was taken to be a softmax policy corresponding to $\theta_h\^t$, with an appropriate choice of temperature. As we discuss further in \cref{sec:actor-analysis}, our choice of $\pi_h\^t$ implicitly leads $\pi_h\^t(x)$ to implement the no-regret \emph{follow-the-perturbed-leader} algorithm at each state $x$. In contrast, the softmax policy from \cite{zanette2021provable} corresponds to the exponential weights algorithm. 
This difference is crucial to allow us to use \arxiv{\cref{cor:ibe-linear}}\rlc{\cref{thm:pert-linear-informal}} to ensure that the program \cref{eq:critic-program} is feasible and outputs a pessimistic estimate of the value function $V_1^{\pi\^t}(x_1)$, for each $t \in [T]$. The settings of the parameters in our algorithms are given in \cref{def:params-offline} \arxiv{below}\rlc{in the appendix}. \rlc{We remark here that the optimal choice of the size $\eta$ of the perturbation turns out to scale proportionally to $ \sqrt{\epbe}$, which leads to the scaling of the error with $\sqrt{\epbe}$ in \cref{thm:pacle-ftpl}.} \rlc{The full details of the proof of \cref{thm:pacle-ftpl} may be found in \cref{sec:ub-proof-rlc}.}

\begin{remark}[Relation to prior work: Bellman restricted closedness]
We point out that many works on offline RL consider similar actor-critic methods to ours in the setting of general function approximation \citep{zanette2021provable,xie2021bellman,cheng2022adversarially,nguyen2023on}, generally under the assumption of $\Pi$-Bellman restricted closedness for various policy classes $\Pi$. While some of these results are sufficiently general to be instantiated in our setting (of low inherent Bellman error) using \cref{thm:pert-linear-informal}, which establishes (approximate) $\Pilinp$-Bellman restricted closedness when $\epbe$ is small, none of the resulting implications will be as strong as \cref{thm:pacle-ftpl}, either suffering from suboptimal statistical rates or computational intractability. 
We defer detailed discussion to \cref{sec:related-work}.
\end{remark}

\arxiv{\rlc{\section{Proof of \cref{thm:pacle-ftpl}}\label{sec:ub-proof-rlc}}
\rlc{In this section we prove \cref{thm:pacle-ftpl}. We begin by defining the values for the parameters used in \cref{alg:actor,alg:critic}.}
\begin{defn}[Parameter settings  for the algorithms]
  \label{def:params-offline}
  Given $d, H, \Bbnd > 0$ specifying the dimension, horizon, and boundedness parameters for the unknown MDP $M$, as well as dataset size $n$ and error parameters $\epfinal, \delta \in (0,1)$, we define the following parameters to be used in \cref{alg:actor,alg:critic} and its analysis:
  \begin{itemize}
  \item $\beta := 2BH$.
    \item $T := \frac{16\beta^2 d^{1/2}}{\epfinal^2}$.
    \item $\epapx := 1/\sqrt{n}$.
    \item $\eta:=\beta \cdot  \max\left\{ T^{1/2} d^{-1/4},T\epbe^{1/2} \right\}$.
    \item $\sigma := \frac{\eta}{T\beta}$.
    \item $\zeta = C_{\ref{cor:ibe-linear}} \epbe d^{3/2} \cdot \left(\sqrt{ d \log(d/(\epbe\sigma))} + \frac{1}{\sigma} \right)$, where $C_{\ref{cor:ibe-linear}}$ is the constant from \cref{cor:ibe-linear}. (Note that $\zeta = \zeta_\sigma$, where $\zeta_\sigma$ was defined in \cref{cor:ibe-linear}.) 
    \item $\alpha := 4\beta \zeta \sqrt{n} +  C_{\ref{lem:good-event-prob}}\beta d \log^{1/2}(dn\beta/(\sigma\delta))$, %
      where $C_{\ref{lem:good-event-prob}}$ is a constant chosen sufficiently large so as to ensure \cref{lem:good-event-prob} holds.       
  \end{itemize}
\end{defn}

\subsection{Critic analysis}
Consider a tuple $f = (f_1, \ldots, f_H)$, where $f_h : \MX \times \MA \ra \BR$, and a policy $\pi$. We define an MDP $M^{f,\pi}$ (called the \emph{induced MDP}, per \cite{zanette2021provable}), whose transitions are identical to those of $M$, but whose rewards are given as follows: for $h \in [H]$ and $(x,a) \in \MX \times \MA$, 
\begin{align}
r_h\sups{M^{f,\pi}}(x,a) =    f_h(x,a) - \E_{x' \sim P_h(x,a)} [f_{h+1}(x', \pi_{h+1}(x'))]\label{eq:induced-reward}.
\end{align}
\begin{lemma}[Lemma 1 of \cite{zanette2021provable}]
  \label{lem:imag-mdp}
  Fix any $f=(f_1, \ldots, f_H)$ with $f_h : \MX \times \MA \ra \BR$, any $\pi \in \Pi$, and write $M' := M^{f,\pi}$. Then the $Q$-value function of $M'$ for the policy $\pi$, denoted $Q\sups{M', \pi}$, satisfies the following:
  \begin{align}
\forall h \in [H],\ (x,a) \in \MX \times \MA, \qquad Q_h\sups{M', \pi}(x,a) = f_h(x,a)\nonumber,
  \end{align}
  which implies in particular that $V_h\sups{M', \pi}(x) = f_h(x, \pi_h(x))$. 
\end{lemma}
The proof of \cref{lem:imag-mdp} follows a simple telescoping argument and is provided in \cite{zanette2021provable}. Next, we establish the following straightforward guarantee for \EstFeature:
\begin{lemma}
  \label{lem:estfeature-guarantee}
For any $x \in \MX$, $\pi \in \Pilinpp$, $h \in [H]$, $\epapx \in (0,1)$ and $\delta \in (0,1)$, \EstFeature (\cref{alg:estfeature}) runs in time $O(d\epapx^{-2} \log(d/\delta))$ and returns a vector $\hat \phi$ so that, with probability $1-\delta$, $\| \hat \phi - \phi_h(x, \pi_{h,w}(x))\|_2 \leq \epapx$.
\end{lemma}
\begin{proof}
  By the definition of $\pi_{h,w,\sigma}$, it holds that for each $\theta_i \sim \MN(w, \sigma^2 \cdot I_d)$ in \cref{alg:estfeature}, we have $\E[\phi_h(x, \pi_{h, \theta_i}(x))] = \phi_h(x, \pi_h(x))$. Then by Hoeffding's inequality and a union bound, we have that, with probability $1-\delta$,
  \begin{align}
\left\| \frac 1N \sum_{i=1}^N \phi_h(x, \pi_{h,\theta_i}(x)) - \phi_h(x, \pi_h(x)) \right\|_2 \leq \sqrt{d} \cdot \left\| \frac 1N \sum_{i=1}^N \phi_h(x, \pi_{h,\theta_i}(x)) - \phi_h(x, \pi_h(x)) \right\|_\infty \leq \epapx\nonumber.
  \end{align}
\end{proof}

Recall that \texttt{Critic} (\cref{alg:critic}) computes a vector $w = (w_1, \ldots, w_H)\in \BR^{dH}$. Given such $w$, we define a function $f^w = (f^w_1, \ldots, f^w_H)$, where $f^w_h(x,a) := \lng \phi_h(x,a), w_h\rng$. 
We introduce the following notation, in the context of \Critic (\cref{alg:critic}). Given a dataset $\MD = \{ (h_i, x_i, a_i, r_i, x_i')\}_{i=1}^n$, at step $h \in [H]$, a policy $\pi \in \Pilinpp$, a collection of feature vectors $\hat \phi = (\hat \phi_1, \ldots, \hat \phi_n)$ (as produced in \cref{line:estimate-phi}), and a vector $w \in \BR^d$, we define
\begin{align}
\hat \MT_{h, \MD, \hat \phi}^\pi w := \Sigma_h^{-1} \sum_{i \in \MI_h} \phi_h(x_i, a_i) \cdot (r_i + \lng \hat \phi_i, w \rng )\label{eq:define-emp-bellman},
\end{align}
where $\Sigma_h = \del{\lambda} I_d + \sum_{i \in \MI_h} \phi_h(x_i, a_i) \phi_h(x_i, a_i)^\t$ and $\MI_h = \{ i : \ h_i = h \}$ are defined as in \cref{alg:critic}. %
Given parameters $\alpha, \sigma > 0$, we  define the following good event $\ME_{\alpha,\beta,\sigma,\epapx}$:
\begin{align}
\ME_{\alpha,\beta,\sigma,\epapx} = \left\{ \sup_{\| w_{h+1} \|_2 \leq \beta}  \sup_{\pi_{h+1} \in \Pilinp[\sigma]_{h+1}} \sup_{\substack{ (\hat \phi_i)_{i \in [n]} \in (\BR^d)^n :\ \forall i\in[n], \\ \| \hat \phi_i - \phi_{h+1}(x_i', \pi_{h+1}(x_i')) \|_2 \leq \epapx}}  \left\| \MT_h^{\pi_{h+1}} w_{h+1} - \hat \MT_{h, \MD,\hat\phi}^{\pi_{h+1}} w_{h+1} \right\|_{\Sigma_h} \leq \alpha \right\}\label{eq:offline-good-perturbed}.
\end{align}
\cref{lem:mprime-m-diff} below establishes that under the good event $\ME_{\alpha,\beta,\sigma,\epapx}$, upon given a perturbed linear policy $\pi$ as input, \Critic produces a vector $w^\st \in \BR^{dH}$ so that the value of $\pi$ in the induced MDP $M^{f^{w^\st}, \pi}$ lower bounds the value of $\pi$ in the true MDP (\cref{it:vmf}). Moreover, for \emph{any} policy $\pi' \in \Pi$, the value of $\pi'$ in $M^{f^{w^\st}, \pi}$ is close to the value of $\pi'$ in $M$, as controlled by a concentrability coefficient depending on $\pi'$ (\cref{it:all-piprime}).
Recall from \cref{cor:ibe-linear} that we have defined $\zeta_\sigma = C_{\ref{cor:ibe-linear}} \epbe d^{3/2} \cdot \left( \sqrt{d \log(d/(\epbe\sigma))} + \frac{1}{\sigma} \right)$. 
\begin{lemma}
  \label{lem:mprime-m-diff}
  Consider any $\sigma,\alpha,\epapx, \delta > 0$ for which $3BH\zeta_\sigma \leq 1$, suppose $\beta = 2BH$, and let $\pi \in \Pilinp$ be given. Consider the execution of $\Critic(\MD, \pi,\epapx, \alpha,\beta,\delta)$: there is some event $\ME_\pi$ depending only on the randomness in the call to \EstFeature in \cref{line:estimate-phi}, with $\Pr(\ME_\pi) > 1-\delta$, so that the following holds. Under the event $\ME_{\alpha,\beta,\sigma,\epapx} \cap \ME_\pi$, %
  the output of $\Critic(\MD, \pi,\epapx, \alpha,\beta,\delta)$ is a pair of vectors $w^\st, \xi^\st \in \BR^{dH}$ satisfying the following conditions, for $f := f^{w^\st}$:
  \begin{enumerate}
  \item\label{it:vmf} $V_1\sups{M^{f, \pi}}(x_1) \leq V_1\sups{M, \pi}(x_1)$.
  \item \label{it:all-piprime} For any $\pi' \in \Pi$,
    \begin{align}
\left| V_1\sups{M^{f,\pi}, \pi'}(x_1) - V_1\sups{M, \pi'}(x_1) \right| \leq 3\beta H \zeta_\sigma + 2 \alpha \sum_{h=1}^H \| \E\sups{M, \pi'}[\phi_h(x_h, a_h)] \|_{\Sigma_h^{-1}}\nonumber.
    \end{align}
    Moreover, given an oracle which can solve the convex program \cref{eq:critic-program}, the computational cost of \Critic is $\poly( d, n,H,  \log(1/\delta)/\epapx^2)$. 
  \end{enumerate}
\end{lemma}
For simplicity, we assume in the statement and proof of \cref{lem:mprime-m-diff} that the convex program \cref{eq:critic-program} may be efficiently solved exactly (i.e., that an oracle provides the answer). While strictly speaking this may not be the case, it is known that for any $\ep > 0$, an $\ep$-approximate solution to \cref{eq:critic-program} may be found in time $\poly(\alpha, \beta, n, d, H, \log(1/\ep))$. %
The guarantee of \cref{lem:mprime-m-diff} as well as those of subsequent lemmas hold with only minor modifications if we can only solve the program \cref{eq:critic-program} $\ep$-approximately in this manner. In particular, our main guarantee (\cref{thm:pacle-ftpl}) may be achieved computationally efficiently, without assumption of any oracle; the necessary modifications to the proof are described in \cref{rmk:e2e-comp}. 
\begin{proof}[Proof of \cref{lem:mprime-m-diff}] By \cref{lem:estfeature-guarantee} and a union bound over $i \in \MI_h$, there is some event $\ME_\pi$ with $\Pr(\ME_\pi) \geq 1-\delta$ so that, under $\ME_\pi$, the estimates $\hat \phi_i^\pi$ produced on \cref{line:estimate-phi} of \cref{alg:critic} satisfy $\max_{i \in [n]} \| \hat \phi_i^\pi - \phi_{h+1}(x_i', \pi_{h+1}(x_i')) \|_2 \leq \epapx$ for all $i \in \MI_h$. 
  We prove the two statements of the lemma in turn:

  \paragraph{Proof of \cref{it:vmf}.}
  Let $w_h^\pi \in 2H \cdot \MB_h$ be defined per \cref{cor:qlin}: in particular, $w_h^\pi= \MT_h^\pi w_{h+1}^\pi$. 
  
  For each $h \in [H]$, define $\xi_h^\pi := \MT_h^\pi w_{h+1}^\pi - \hat \MT_{h, \MD,\hat\phi}^\pi w_{h+1}^\pi$. We claim that the vectors $(w_h^\pi)_{h \in [H]}, (\xi_h^\pi)_{h \in [H]}$ constitute a feasible solution to the program \cref{eq:critic-program}. By the definition of $\hat \MT_{h,\MD,\hat \phi}^\pi$ in \cref{eq:define-emp-bellman}, note that \cref{eq:critic-b} requires that $w_h^\pi = \xi_h^\pi + \hat \MT_{h, \MD,\hat\phi}^\pi w_{h+1}^\pi$.  This equality holds by definition of $\xi_h^\pi$ and since $w_h^\pi = \MT_h^\pi w_{h+1}^\pi$. Next, the fact that $\ME_{\alpha,\beta,\sigma,\epapx} \cap \ME_\pi$ holds, $\pi_{h+1} \in \Pilinp_{h+1}$, and $w_h^\pi \in 2H \cdot \MB_h$ (and hence $\| w_h^\pi \|_2 \leq 2BH$; here we use \cref{cor:qlin} together with the fact that $3BH\zeta_\sigma \leq 1$) gives that $\| \xi_h^\pi \|_{\Sigma_h} = \| \MT_h^\pi w_{h+1}^\pi - \hat \MT_{h, \MD,\hat\phi}^\pi w_{h+1}^\pi \|_{\Sigma_h} \leq \alpha$; this verifies \cref{eq:critic-c}. Finally, since $\beta =2BH$ we have that \cref{eq:critic-d} holds. %
  It follows that \cref{eq:critic-program} is feasible under the event $\ME_{\alpha,\beta,\sigma,\epapx} \cap \ME_\pi$. 

  Let us denote the solution to \cref{eq:critic-program} by $w^\st, \xi^\st$. Since $w_h^\pi, \xi_h^\pi$ are feasible to \cref{eq:critic-program}, we must have that $\lng w_1^\st, \phi_1(x_1, \pi_1(x_1)) \rng \leq \lng w_1^\pi, \phi_1(x_1, \pi_1(x_1)) \rng = V_1^\pi(x_1)$. But \cref{lem:imag-mdp} together with our choice of $f = f^{w^\st}$ (so that $f_h(x,a) = \lng \phi_h(x,a), w_h^\st \rng$)  guarantees that $\lng w_1^\st, \phi_1(x_1, \pi_1(x_1)) \rng =V_1^{M^f, \pi}(x_1)$, which yields $V_1^{M^f, \pi}(x_1) \leq V_1^\pi(x_1)$.

  \paragraph{  Proof of \cref{it:all-piprime}.} For each $h \in [H]$ and $(x,a) \in \MX \times \MA$, we have, by the definition \cref{eq:induced-reward}, %
  \begin{align}
    r_h\sups{M^{f, \pi}}(x,a) - r_h(x,a) =& \lng w_h^\st, \phi_h(x,a) \rng -r_h(x,a) -  \E_{x' \sim P_h(x,a)} [ \lng w_{h+1}^\st, \phi_{h+1}(x', \pi_{h+1}(x')) \rng ]\nonumber\\
    =& \vep_h(x,a) +  \lng w_h^\st - \MT_h^\pi w_{h+1}^\st, \phi_h(x,a) \rng \nonumber\\
    =& \vep_h(x,a) + \lng \xi_h^\st, \phi_h(x,a) \rng + \lng \hat\MT_{h, \MD,\hat\phi}^\pi w_{h+1}^\st - \MT_h^\pi w_{h+1}^\st, \phi_h(x,a) \rng \label{eq:rdiff-expand},
  \end{align}
where $\vep_h(x,a)$ satisfies $|\vep_h(x,a)| \leq (1 + 2 \| w_{h+1}^\st \|_2) \cdot \zeta_\sigma \leq 3\beta \zeta_\sigma$ (for which such a choice is possible by \cref{cor:ibe-linear}).

  Since the MDP $M^{f,\pi}$ has transitions identical to those of $M$, for any $\pi' \in \Pi$,
  \begin{align}
    \left| V_1\sups{M^{f,\pi}, \pi'}(x_1) - V_1\sups{M, \pi'}(x_1) \right| =& \left| \sum_{h=1}^H \E\sups{M, \pi'} \left[ r_h\sups{M^{f,\pi}}(x_h, a_h) - r_h(x_h, a_h) \right] \right|\nonumber\\
    \leq & 3\beta H \zeta_\sigma+  \sum_{h=1}^H \left|\E\sups{M, \pi'} [ \lng \xi_h^\st, \phi_h(x_h,a_h) \rng ]\right| + \left| \E\sups{M, \pi'} \left[ \lng \hat\MT_{h, \MD,\hat\phi}^\pi w_{h+1}^\st - \MT_h^\pi w_{h+1}^\st, \phi_h(x_h,a_h) \rng \right] \right|\nonumber\\
    \leq &3 \beta H \zeta_\sigma +  \sum_{h=1}^H \left\| \E\sups{M, \pi'}[\phi_h(x_h,a_h)] \right\|_{\Sigma_h^{-1}} \cdot \left( \| \xi_h^\st \|_{\Sigma_h} + \| \hat\MT_{h, \MD,\hat\phi}^\pi w_{h+1}^\st - \MT_h^\pi w_{h+1}^\st \|_{\Sigma_h} \right)\nonumber\\
    \leq &3\beta H \zeta_\sigma+  2\alpha \sum_{h=1}^H \left\| \E\sups{M, \pi'}[\phi_h(x_h, a_h)] \right\|_{\Sigma_h^{-1}}\label{eq:vmf-vmpi-establish},
  \end{align}
  where the first inequality uses \cref{eq:rdiff-expand} and the triangle inequality, and the third inequality uses the fact that $(w_h^\st, \xi_h^\st)$ is feasible for \cref{eq:critic-program} (in particular, \cref{eq:critic-c}) as well as the fact that $\ME_{\alpha,\beta,\sigma,\epapx} \cap \ME_\pi$ holds and $\| w_{h+1}^\st \|_2 \leq \beta = 2BH$  (using \cref{eq:critic-d}).

Finally, we analyze the computational cost of \cref{alg:critic}. The call to $\EstFeature$ in \cref{line:estimate-phi} takes time $\poly(N, d) \leq \poly(d, \log(1/\delta)/\epapx^2)$. The remaining steps take time $\poly(d, H)$. 
\end{proof}

\begin{lemma}
  \label{lem:good-event-prob}
  There is a constant $C_{\ref{lem:good-event-prob}}$ so that the following holds. 
  Suppose the dataset $\MD$ is drawn according to \cref{asm:offline-data-assumption}, and  $\alpha,\beta, \sigma, \epapx > 0$ are given so that $\epapx \leq 1/\sqrt{n}$ and $\alpha \geq 4\beta \zeta_\sigma \sqrt{n} +  C_{\ref{lem:good-event-prob}}\beta d \log^{1/2}(dn\beta/(\sigma\delta))$. Then, %
  $\Pr(\ME_{\alpha,\beta,\sigma,\epapx}) \geq 1-\delta/2$, where the probability is over the draw of $\MD$. 
\end{lemma}
We remark that the only source of randomness in \Critic is over the randomness of \EstFeature in \cref{line:estimate-phi}. 
\begin{proof}[Proof of \cref{lem:good-event-prob}]
  Given $h \in [H]$, $w_{h+1} \in \BR^d$, $\pi_{h+1} \in \Pilinp$, and $\hat \phi = (\hat \phi_i)_{i \in [n]} \in (\BR^d)^n$, let us write, for $i \in \MI_h$,
  \begin{align}
    \vep_i(w_{h+1}, \pi_{h+1}) := & r_i + \lng \phi_{h+1}(x_i', \pi_{h+1}(x_i')), w_{h+1} \rng - (r_h(x_i, a_i) + \E_{x' \sim P_h(x_i, a_i)}[ \lng \phi_{h+1}(x' ,\pi_{h+1}(x')), w_{h+1} \rng ]) \nonumber\\ %
    \xi_i(w_{h+1}, \pi_{h+1}) := & r_h(x_i, a_i) + \E_{x' \sim P_h(x_i, a_i)}[ \lng \phi_{h+1}(x' ,\pi_{h+1}(x')), w_{h+1} \rng ] - \lng \phi_h(x_i, a_i), \MT_h^{\pi_{h+1}} w_{h+1} \rng\nonumber\\
    \eta_i(w_{h+1}, \pi_{h+1}, \hat \phi_i) :=&  \lng \hat \phi_i - \phi_{h+1}(x_i', \pi_{h+1}(x_i')), w_{h+1} \rng\nonumber. 
  \end{align}
  
\paragraph{Error decomposition.}  For any $w_{h+1} \in \BR^d$, $\pi_{h+1} \in \Pilinp_{h+1}$, and $\hat \phi = (\hat \phi_i)_{i \in [n]}$, we can decompose
  \begin{align}
    & \hat \MT_{h, \MD, \hat \phi}^{\pi_{h+1}} w_{h+1}\nonumber\\
    =& \Sigma_h^{-1} \sum_{i \in \MI_h} \phi_h(x_i, a_i) \cdot ( r_i + \lng \hat \phi_i, w_{h+1} \rng )\nonumber\\
    = & \Sigma_h^{-1} \sum_{i \in \MI_h} \phi_h(x_i, a_i) \cdot (r_i + \lng \phi_{h+1}(x_i', \pi_{h+1}(x_i')), w_{h+1} \rng ) + \Sigma_h^{-1} \sum_{i \in \MI_h} \phi_h(x_i, a_i) \cdot \eta_i( w_{h+1}, \pi_{h+1}, \hat \phi_i) \nonumber\\
    =& \Sigma_h^{-1} \sum_{i \in \MI_h} \phi_h(x_i, a_i) \cdot \lng \phi_h(x_i, a_i), \MT_h^{\pi_{h+1}} w_{h+1} \rng \nonumber\\
     & + \Sigma_h^{-1} \sum_{i \in \MI_h} \phi_h(x_i, a_i) \cdot \left(\vep_i(w_{h+1}, \pi_{h+1}) + \xi_i(w_{h+1}, \pi_{h+1})+  \eta_i( w_{h+1}, \pi_{h+1}, \hat \phi_i)\right)\nonumber\\
    =& -\Sigma_h^{-1} \cdot \del{\lambda I_d} \cdot \MT_h^\pi w_{h+1} + \Sigma_h^{-1} \left( \del{\lambda I_d} + \sum_{i \in \MI_h} \phi_h(x_i, a_i) \phi_h(x_i, a_i)^\t \right) \cdot \MT_h^{\pi_{h+1}} w_{h+1}\nonumber\\
    & \quad +  \Sigma_h^{-1} \sum_{i \in \MI_h} \phi_h(x_i, a_i) \cdot (\vep_i(w_{h+1}, \pi_{h+1}) + \xi_i(w_{h+1}, \pi_{h+1}) +  \eta_i( w_{h+1}, \pi_{h+1}, \hat \phi_i))\nonumber\\
    =& \MT_h^{\pi_{h+1}} w_{h+1} - \del{\lambda} \Sigma_h^{-1} \cdot \MT_h^{\pi_{h+1}} w_{h+1} + \nonumber\\
    & + \Sigma_h^{-1} \sum_{i \in \MI_h} \phi_h(x_i, a_i) \cdot (\eta_i(w_{h+1}, \pi_{h+1}, \hat \phi_i) + \xi_i(w_{h+1}, \pi_{h+1})+  \vep_i(w_{h+1}, \pi_{h+1})) \label{eq:offline-decompose},
  \end{align}
  where the final equality uses the definition of $\Sigma_h$ (in \cref{line:define-sigmah} of \cref{alg:critic}).

 \paragraph{Bounding the error terms.} First, we note that for any $w_{h+1}$ satisfying $\| w_{h+1} \|_2 \leq \beta$ and any $\pi_{h+1} \in \Pilinp_{h+1}$,
  \begin{align}
\| \del{\lambda} \Sigma_h^{-1} \cdot \MT_h^{\pi_{h+1}} w_{h+1} \|_{\Sigma_h} = \del{\lambda} \| \MT_h^{\pi_{h+1}} w_{h+1} \|_{\Sigma_h^{-1}} \leq \del{\sqrt{\lambda} \cdot} \| \MT_h^{\pi_{h+1}} w_{h+1} \|_2 \leq 3\del{\sqrt{\lambda}} \cdot \beta B\label{eq:lambda-term-bound},
  \end{align}
  where the first inequality uses that $\Sigma_h \succeq \del{\lambda} I_d$ and the second inequality uses that, since $\| w_{h+1} \|_2 \leq \beta$, we have $\MT_h^{\pi_{h+1}} w_{h+1} \in 3\beta \cdot \MB_h$ (\cref{cor:ibe-linear}), and hence $\| \MT_h^{\pi_{h+1}} w_{h+1} \|_2 \leq 3\beta B$ (\cref{asm:boundedness}).

  Next, as long as $\| w_{h+1} \|_2 \leq \beta$ and $\| \hat \phi_i - \phi_{h+1}(x_i', \pi_{h+1}(x_i')) \|_2 \leq \epapx$ for all $i \in [n]$, we have
  \begin{align}
| \eta_i(w_{h+1}, \pi_{h+1}, \hat \phi_i) | \leq \| \hat \phi_i - \phi_{h+1}(x_i', \pi_{h+1}(x_i')) \|_2 \cdot \| w_{h+1} \|_2 \leq \beta \epapx\nonumber
  \end{align}
  for all $i \in [n]$ and thus 
  \begin{align}
\left\| \Sigma_h^{-1} \sum_{i \in \MI_h} \phi_h(x_i, a_i) \cdot \eta_i(w_{h+1}, \pi_{h+1},\hat\phi_i)\right\|_{\Sigma_h} = \left\| \sum_{i \in \MI_h} \phi_h(x_i,a_i) \cdot \eta_i(w_{h+1}, \pi_{h+1},\hat\phi_i) \right\|_{\Sigma_h^{-1}} \leq \beta\epapx\sqrt{n}\label{eq:eta-term-bound},
  \end{align}
  where the inequality uses \cref{lem:projection-bound}.

  Next, by \cref{cor:ibe-linear}, for any $w_{h+1}$ satisfying $\| w_{h+1} \|_2 \leq \beta$ and any $\pi_{h+1} \in \Pilinpp_{h+1}$, for each $i \in \MI_h$ we have
  \begin{align}
\left| r_h(x_i, a_i) + \E_{x' \sim P_h(x_i, a_i)} [ \lng \phi_{h+1}(x' ,\pi_{h+1}(x')), w_{h+1} \rng] - \lng \phi_h(x_i, a_i), \MT_h^{\pi_{h+1}} w_{h+1} \rng \right| \leq & \beta \cdot \zeta_\sigma\nonumber.
  \end{align}
  Thus, by \cref{lem:projection-bound},
    \begin{align}
\left\| \Sigma_h^{-1} \sum_{i \in \MI_h} \phi_h(x_i, a_i) \cdot \xi_i(w_{h+1}, \pi_{h+1})\right\|_{\Sigma_h} = \left\| \sum_{i \in \MI_h} \phi_h(x_i,a_i) \cdot \xi_i(w_{h+1}, \pi_{h+1}) \right\|_{\Sigma_h^{-1}} \leq \beta\zeta_\sigma\sqrt{n}\label{eq:xi-term-bound}.
  \end{align}
  
  It remains to bound the final term in \cref{eq:offline-decompose}, namely the one involving $\vep_i(w_{h+1}, \pi_{h+1})$. To do so, we use a covering based argument. We define the following metric on $\Pilinp_h$: for $\pi_h, \pi_h' \in \Pilinp_h$, define
  \begin{align}
\| \pi_h - \pi_h' \|_{\infty, 1} := \sup_{x \in \MX} \sum_{a \in \MA} | \pi_h(a|x) - \pi_h'(a|x)|\label{eq:inf-1-norm}.
  \end{align}
  For $\ep > 0$,  we let $\MN_{\infty, 1}(\Pilinp_h, \ep)$ denote the minimum size of an $\ep$-cover of $\Pilinp_h$ with respect to $\| \cdot \|_{\infty, 1}$. Similarly, let $\MG := \{ w \in \BR^d :\ \| w \|_2 \leq \beta \}$, and let $\MN_2(\MG, \ep)$ denote the miniimum size of an $\ep$-cover of $\MG$ with respect to $\| \cdot \|_2$. We use the following lemma:
  \begin{lemma}[Covering number bounds]
    \label{lem:covering}
    There is a constant $C_{\ref{lem:covering}}$ so that for all $h \in [H]$, 
    \begin{align}
      \log \MN_{\infty, 1}(\Pilinp_h, \ep) \leq & d \log (C_{\ref{lem:covering}}/(\ep\sigma))\nonumber\\
      \log \MN_2(\MG, \ep) \leq & d \log(C_{\ref{lem:covering}} \beta/\ep)\nonumber.
    \end{align}
  \end{lemma}
  Let $\til \MG \subset \MG$ be an $\ep$-cover of $\MG$ with respect to $\| \cdot \|_2$ with size bounded per \cref{lem:covering}, and $\tilPilinp_{h+1} \subset \Pilinp_{h+1}$ be an $\ep$-cover of $\Pilinp_{h+1}$ with respect to $\| \cdot \|_{\infty, 1}$ with size bounded per \cref{lem:covering}.

  For fixed $\til w_{h+1} \in \til \MG$ and $\til \pi_{h+1} \in \tilPilinp$, for each $i \in \MI_h$, we have from the definition of $\vep_i(\til w_{h+1}, \til \pi_{h+1})$ that 
  \begin{align}
| \vep_i(\til w_{h+1}, \til \pi_{h+1}) | \leq 2 + 2 \| \til w_{h+1} \|_2 \leq 2 + 2 \beta \leq 4\beta \label{eq:vep-bounded},
  \end{align}
  where we have used \cref{asm:boundedness} and the fact that $\beta \geq 1$. Moreover, by \cref{asm:offline-data-assumption}, if we let $\MF_i$ denote the sigma-algebra generated by all tuples $(h_j, x_j, a_j, r_j, x_j')$ for $j \leq i$ and by $(h_{i+1}, x_{i+1}, a_{i+1})$, we have $\E[\vep_i(\til w_{h+1}, \til \pi_{h+1}) | \MF_{i-1}] = 0$ and $\E[e^{\lambda \vep_i(\til w_{h+1}, \til \pi_{h+1})} | \MF_{i-1}] \leq e^{\lambda^2 \cdot (4\beta)^2/2}$, where we have used \cref{eq:vep-bounded}. Moreover, $\phi_h(x_i, a_i)$ is measurable with respect to $\MF_{i-1}$. Let us write $N := |\til \MG| \cdot |\tilPilinp|$, so that $\log N \leq C d \log(\beta/\ep\sigma)$, for some constant $C$. By \cref{lem:conc-sn-martingale}, there is some event $\ME_{\til w, \til \pi_{h+1}}$ which occurs with probability at least $1-\delta /N$, so that under $\ME_{\til w, \til \pi_{h+1}}$, we have
  \begin{align}
\left\| \sum_{i \in \MI_h} \phi_h(x_i, a_i) \cdot \vep_i(\til w_{h+1}, \til \pi_{h+1}) \right\|_{\Sigma_h^{-1}} \leq 32\beta^2 \log^{1/2}  \left( \frac{\det(\Sigma_h)^{1/2} \del{\cdot \lambda^{-d/2}}}{\delta/N} \right) \leq  8\beta d^{1/2} \log^{1/2} \left(N\cdot \frac{\del{\lambda} d + n}{\del{\lambda} d \delta} \right)\nonumber,
  \end{align}
  where the final inequality uses
  \begin{align}
\det (\Sigma_h) \leq \left( \frac 1d \Tr \Sigma_h \right)^d \leq \left( \frac 1d \cdot \left( \del{\lambda} d + \sum_{i \in \MI_h} \| \phi_h(x_i, a_i) \|_2^2 \right) \right)^d \leq \left( \frac{\del{\lambda} d + n}{d} \right)^d.\nonumber
  \end{align}
  By a union bound, it follows that under some event $\ME$ that occurs with probability at least $1-\delta$, we have
  \begin{align}
 \sup_{\substack{\til w \in \til \MG \\ \til \pi_{h+1} \in \tilPilinp_{h+1}}}\left\| \sum_{i \in \MI_h} \phi_h(x_i, a_i) \cdot \vep_i(\til w_{h+1}, \til \pi_{h+1}) \right\|_{\Sigma_h^{-1}} \leq 8\beta d^{1/2} \log^{1/2} \left(N \cdot \frac{\del{\lambda} d + n}{\del{\lambda} d \delta} \right).\label{eq:cover-highprob}
  \end{align}
  Now consider any $w_{h+1} \in \MG$ and $\pi_{h+1} \in \Pilinp_{h+1}$. Choose $\til w_{h+1} \in \til \MG$ and $\til \pi_{h+1} \in \tilPilinp_{h+1}$ so that $\| w_{h+1} - \til w_{h+1} \|_2 \leq \ep$ and $\| \pi_{h+1} - \til \pi_{h+1} \|_{\infty, 1} \leq \ep$. We have
  \begin{align}
    \sup_{\pi_{h+1} \in \Pilinp_{h+1}} | \vep_i(w_{h+1}, \pi_{h+1}) - \vep_i(\til w_{h+1}, \pi_{h+1}) | \leq & 2 \cdot \| w_{h+1} - \til w_{h+1} \|_2 \leq 2\ep\label{eq:wcover-allpi}\\
    \sup_{w_{h+1} \in \MG} \left| \vep_i(w_{h+1}, \pi_{h+1}) - \vep_i(w_{h+1}, \til \pi_{h+1}) \right| \leq & 2\beta \cdot \sup_{x' \in \MX} \| \phi_{h+1}(x', \pi_{h+1}(x')) - \phi_{h+1}(x', \til \pi_{h+1}(x')) \|_2 \leq 2\beta\ep\label{eq:picover-allw},
  \end{align}
  where \cref{eq:picover-allw} uses the definition of $\| \cdot \|_{\infty, 1}$ in \cref{eq:inf-1-norm}. 
  Then, under the event $\ME$, 
  \begin{align}
    & \left\| \sum_{i \in \MI_h} \phi_h(x_i, a_i) \cdot \vep_i(w_{h+1}, \pi_{h+1}) \right\|_{\Sigma_h^{-1}} \nonumber\\
    \leq &  \left\| \sum_{i \in \MI_h} \phi_h(x_i, a_i) \cdot \vep_i(\til w_{h+1}, \til \pi_{h+1}) \right\|_{\Sigma_h^{-1}} +  \left\| \sum_{i \in \MI_h} \phi_h(x_i, a_i) \cdot \left( \vep_i(\til w_{h+1}, \pi_{h+1}) - \vep_i(w_{h+1}, \pi_{h+1}) \right) \right\|_{\Sigma_h^{-1}}\nonumber\\
    &+  \left\| \sum_{i \in \MI_h} \phi_h(x_i, a_i) \cdot\left( \vep_i(\til w_{h+1}, \til \pi_{h+1}) -  \vep_i(\til w_{h+1}, \pi_{h+1}) \right) \right\|_{\Sigma_h^{-1}}\nonumber\\
    \leq &  \left\| \sum_{i \in \MI_h} \phi_h(x_i, a_i) \cdot \vep_i(\til w_{h+1}, \til \pi_{h+1}) \right\|_{\Sigma_h^{-1}} + 4\beta\ep \sqrt{n}\nonumber\\
    \leq &  8\beta d^{1/2} \log^{1/2} \left(N \cdot \frac{\del{\lambda} d + n}{\del{\lambda} d \delta} \right) + 4\beta\ep \sqrt{n}\label{eq:vep-term-bound}, 
  \end{align}
  where the first inequality uses the triangle inequality, the second inequality uses \cref{lem:projection-bound} together with \cref{eq:wcover-allpi,eq:picover-allw}, and the final inequality uses \cref{eq:cover-highprob} together with the fact that $\ME$ holds.

  Combining \cref{eq:offline-decompose,eq:lambda-term-bound,eq:eta-term-bound,eq:vep-term-bound,eq:xi-term-bound}, we conclude that under the event $\ME$, for any $w_{h+1} \in \MG$, $\pi_{h+1} \in \Pilinp_{h+1}$, and $\hat \phi \in (\BR^d)^n$ so that $\max_{i \in [n]} \| \hat \phi_i - \phi_{h+1}(x_i', \pi_{h+1}(x_i')) \|_2 \leq \epapx$, 
  \begin{align}
    \| \hat \MT_{h, \MD, \hat \phi}^{\pi_{h+1}} w_{h+1} - \MT_h^{\pi_{h+1}} w_{h+1} \|_{\Sigma_h} \leq &3 \del{\sqrt{\lambda}} \beta B + \beta\epapx \sqrt{n}+ \beta \zeta_\sigma \sqrt{n} + 8\beta  d^{1/2} \log^{1/2} \left( N \cdot \frac{\del{\lambda} d + n}{\del{\lambda} d \delta} \right) + 4\beta\ep \sqrt{n}\nonumber\\
    \leq &  11C \beta d \log^{1/2} \left( \frac{dn\beta}{\ep \delta\sigma} \right) + 4\beta(\ep + \epapx+\zeta_\sigma) \sqrt{n}\nonumber,
  \end{align}
  where the second inequality uses the bound $\log N \leq Cd \log(\beta/\ep\sigma)$. %
  Choosing $\ep = 1/\sqrt{n}$, and as long as $\epapx \leq 1/\sqrt{n}$, we see that $\ME \subset \ME_{\alpha, \beta,\sigma, \epapx}$, since we have chosen $\alpha \geq 4\beta \zeta_\sigma \sqrt{n} +  C\beta d \log^{1/2}(dn\beta/(\sigma\delta))$ for a sufficiently large constant $C$.  Thus $\Pr(\ME_{\alpha,\beta,\sigma,\epapx}) \geq \Pr(\ME) \geq 1-\delta$. Rescaling $\delta$ to $\delta/2$ yields the result. 
\end{proof}

\begin{proof}[Proof of \cref{lem:covering}]
  First, we note that $\log \MN_2(\MG, \ep) \leq d \cdot \log(3\beta/(\ep))$ by \citet[Example 5.8]{wainwright2019}. To bound the covering number of $\Pilinp_h$, let $\MC \subset \{ w \in \BR^d :\ \| w \|_2 \leq  1 \}$ denote a $\ep\sigma$-cover of the unit ball, so that \citet[Example 5.8]{wainwright2019} ensures that we can choose $\MC$ with $|\MC| \leq d \log(3/(\ep\sigma))$. For any $w \in \BR^d$ with $\| w \|_2 \leq 1$, there is some $w' \in \MC$ with $\| w-w'\|_2 \leq \ep\sigma$. Then we have
  \begin{align}
\tvd{\MN(w, \sigma^2 \cdot I_d)}{\MN(w', \sigma^2 \cdot I_d)} \leq \sqrt{ \frac 12 \cdot \kld{\MN(w, \sigma^2 \cdot I_d)}{\MN(w', \sigma^2 \cdot I_d)}} = \frac{1}{2\sigma} \cdot \| w-w'\|_2 \leq \frac{\ep}{2}\nonumber,
  \end{align}
  where we have used Pinsker's inequality and the formula for the KL divergence between two Gaussians. Moreover, for any $x \in \MX$, we have
  \begin{align}
\sum_{a \in \MA} |\pi_{h,w,\sigma}(a|x) - \pi_{h,w',\sigma}(a|x)| \leq & \tvd{\MN(w, \sigma^2 \cdot I_d)}{\MN(w', \sigma^2 \cdot I_d)}\nonumber
  \end{align}
  by the data processing inequality for total variation distance (in particular, the deterministic function $W \mapsto \argmax_{a \in \MA} \lng \phi_h(x,a), W \rng$ maps a random variable $W \sim \MN(w, \sigma^2 \cdot I_d)$ to an action $A \sim \pi_{h, w,\sigma}(\cdot | x)$). Combining the above displays gives that $\| \pi_{h,w,\sigma} - \pi_{h,w',\sigma} \|_{\infty, 1} \leq \ep$, as desired. 
\end{proof}

\subsection{Actor analysis}
\label{sec:actor-analysis}
Notice that the \Actor algorithm (\cref{alg:actor}) is a special case of the Follow-the-Perturbed-Leader (FTPL) algorithm. This observation is central to our proof. Below we first review some basic facts pertaining to the FTPL algorithm.

\paragraph{Review of FTPL.}
For $J, L > 0$, a distribution $\mu$ on $\BR^d$ is defined to be \emph{$(J, L)$-stable} with respect to the Euclidean norm $\| \cdot \|_2$ if the following two conditions hold:
\begin{align}
  \E_{\rho \sim \mu} [ \| \rho \|_2] \leq  & J \nonumber\\
  \forall v \in \BR^d, \quad \int_{\rho \in \BR^d} | \mu(\rho) - \mu(\rho - v) | d\rho \leq & L \cdot \| v\|_2\nonumber.
\end{align}
We consider the setting of \emph{online linear optimization}: the algorithm is given a set of actions $\Phi \subset \BR^d$, and operates over some number $T \in \BN$ of rounds. At each round $t \in [T]$, an adversary chooses a \emph{reward vector} $w\^t$, which may be random and depend arbitrarily on past choices of the algorithm (i.e., we consider the case of an \emph{adaptive adversary}). Simultaneously, the algorithm chooses a vector $\phi\^t \in \Phi$, and then observes $w\^t$ and receives a reward $\lng w\^t, \phi\^t \rng$. The algorithm's goal is to minimize its regret with respect to the best-in-hindsight fixed choice of action in $\Phi$. The \emph{expected FTPL} algorithm, presented in \cref{alg:e-ftpl}, solves the online linear optimization problem. Its choice of action at each round is given by the best action for the previous rounds, perturbed by a distribution which satisfies $(J, L)$-stability.
\begin{algorithm}
  \caption{Expected Follow-the-Perturbed-Leader: $\ExpFTPL(\Phi, \mu, T, \eta)$}
  \label{alg:e-ftpl}
  \begin{algorithmic}[1]\onehalfspacing
\Require Action set $\Phi \subset \BR^d$, distribution $\mu \in \Delta(\BR^d)$, parameters $\omega > 0, T \in \BN$.
\For{ $1 \leq t \leq T$}
  \State Choose
    \begin{align}
\phi\^t := \E_{\rho\^t \sim \mu} \left[\argmax_{\phi \in \MA} \left\{ \omega \cdot \sum_{s=1}^{t-1} \lng w\^s, \phi \rng + \lng  \rho\^t,\phi \rng \right\}\right]\nonumber.
    \end{align}
  \State Receive reward vector $w\^t\in\BR^d$, earn reward $\lng \phi\^t, w\^t \rng$. 
\EndFor
\end{algorithmic}
\end{algorithm}

\begin{theorem}[\cite{hazan2016introduction}, Theorem 5.8]
  \label{thm:ftpl}
  Suppose that $\mu$ is $(J, L)$-stable, and write $D := \max_{\phi \in \Phi} \| \phi \|_2$. Then for any adaptive adversary choosing $w\^1, \ldots, w\^T$ satisfying $\| w\^t \|_2 \leq G$ for all $t$, the iterates $\phi\^t$ produced by $\ExpFTPL(\Phi, \mu, T, \omega)$ (\cref{alg:e-ftpl}) satisfy
  \begin{align}
\max_{\phi^\st \in \Phi} \left\{ \sum_{t=1}^T \lng \phi^\st, w\^t \rng - \sum_{t=1}^T \lng \phi\^t, w\^t \rng \right\} \leq \omega L D G^2 T + \frac{J D}{\omega}\nonumber. 
  \end{align}
\end{theorem}
We remark that since the iterates $\phi\^t$ of \cref{alg:e-ftpl} are deterministic, in the setting of no-regret learning (i.e., that of \cref{thm:ftpl}), adaptive and oblivious adversaries are equivalent. 

\paragraph{Analysis of \Actor.} Recall that for $w = (w_1, \ldots, w_H) \in \BR^{dH}$, we have defined $f^w := (f_1^w, \ldots, f_H^w)$, where $f_h^w(x,a) := \lng \phi_h(x,a), w_h \rng$. 
To simplify notation, we write $f\^t := f^{w\^t}$, where $w\^t = (w_1\^t, \ldots, w_H\^t)$ is the vector returned by \texttt{Critic} on \cref{line:call-critic} of \Actor (\cref{alg:actor}).
\begin{lemma}
  \label{lem:ftpl-actor}
For any $\MD, \epfinal,\delta$, the algorithm $\Actor(\MD, \epfinal,\delta, \eta)$ (\cref{alg:actor}) satisfies the following: for any $\pi^\st \in \Pi$, $h \in [H]$, and $x \in \MX$, we have
  \begin{align}
\sum_{t=1}^T f_h\^t(x, \pi_h^\st(x)) - \sum_{t=1}^T f_h\^t(x, \pi_h\^t(x)) \leq  \eta^{-1} (2BH)^2 T + {\eta \sqrt{d}}\nonumber. %
  \end{align}
\end{lemma}
We emphasize that the policy $\pi^\st$ in \cref{lem:ftpl-actor} is not required to be a (perturbed) linear policy.
\begin{proof}[Proof of \cref{lem:ftpl-actor}]
  Fix $\pi^\st \in \Pi$ and $x \in \MX$. Note that, by definition of $f_h\^t$,
  \begin{align}
\sum_{t=1}^T f_h\^t(x, \pi_h^\st(x)) = \sum_{t=1}^T \lng w_h\^t, \phi_h(x, \pi_h^\st(x)) \rng\nonumber,
  \end{align}
  and
  \begin{align}
\sum_{t=1}^T f_h\^t(x, \pi\^t(x)) = \sum_{t=1}^T \lng w_h\^t, \phi_h(x, \pi_h\^t(x)) \rng\nonumber.
  \end{align}
Moreover, the definition of $\pi_h\^t$ in \Actor (\cref{alg:actor}) ensures that 
  \begin{align}
    \phi_h(x, \pi_h\^t(x)) =& \E_{\rho_h\^t \sim \MN(0, \eta^2 \cdot I_d)} \left[ \argmax_{a \in \MA} \left\{ \left\lng \phi_h(x,a), \theta_h\^t + \rho_h\^t \right\rng \right\} \right]\nonumber\\
    =& \E_{\rho_h\^t \sim \MN(0, \eta^2  \cdot I_d)} \left[\argmax_{a \in \MA} \left\{  \sum_{s=1}^{t-1} \lng \phi_h(x,a), w_h\^s \rng + \lng \phi_h(x,a), \rho_h\^t \rng \right\}\right]\nonumber,
  \end{align}
  where we have written $\theta_h\^t = \sum_{s=1}^{t-1} w_h\^s$, as in \cref{alg:actor}. 
  In particular, $\phi_h(x, \pi_h\^t(x))$ is exactly the choice of $\ExpFTPL(\Phi, \mu, T, \omega)$ (\cref{alg:e-ftpl}) at round $t$ with action set $\Phi = \cPhi_h(x) = \mathrm{co}(\{ \phi_h(x,a) \ : \ a \in \MA \})$, distribution $\mu= \MN(0, \eta^2 \cdot I_d)$, and $\omega=1$. It follows from \cref{thm:ftpl} that
  \begin{align}
    \sum_{t=1}^T f_h\^t(x, \pi_h^\st(x)) - \sum_{t=1}^T f_h\^t(x, \pi_h\^t(x))  =& \sum_{t=1}^T \lng w_h\^t, \phi_h(x, \pi_h^\st(x)) \rng - \sum_{t=1}^T \lng w_h\^t, \phi_h(x, \pi_h\^t(x)) \rng  \nonumber\\
    \leq  &  \eta^{-1} (2BH)^2 T + {\eta \sqrt{d}}\nonumber,
  \end{align}
 where we have used that $\| \phi_h(x,a) \| \leq 1$ for all $x,a, h$, that $\MN(0, \eta^2 \cdot I_d)$ is $(\eta \sqrt{d}, \eta^{-1})$-stable (\cref{lem:normal-stability}), and that $\| w_h\^t \|_2 \leq \beta = 2BH$, using the constraint \cref{eq:critic-d} in \Critic (and where the value of $\beta$ is set on \cref{line:actor-params} of \cref{alg:actor} per \cref{def:params-offline}).%
\end{proof}

\begin{lemma}
  \label{lem:normal-stability}
  For any $\eta > 0$, $\MN(0, \eta^2 \cdot I_d)$ is $(\eta \sqrt{d}, \eta^{-1})$-stable.
\end{lemma}
\begin{proof}
  We first note that $\E_{Z \sim \MN(0, \eta^2 I_d)}[\| Z \|_2] \leq \eta \sqrt{\E_{Z \sim \MN(0, I_d)}[Z^2]} = \eta \sqrt{d}$.

  Let $\mu_\eta : \BR^d \ra \BR$ denote the probability density function of $\MN(0, \eta^2 I_d)$. To verify the second condition of stability, we compute, for any $v \in \BR^d$, 
  \begin{align}
    \int_{\rho \in \BR^d} |\mu_\eta(\rho) - \mu_\eta(\rho - v)| d\rho =& 2 \tvd{\MN(0, \eta^2 I_d)}{\MN(v, \eta^2 I_d)}\nonumber\\
    \leq &  \sqrt{2 \kld{\MN(0, \eta^2 I_d)}{\MN(v, \eta^2 I_d)}} = \| v \|_2 / \eta\nonumber,
  \end{align}
  where we have used Pinsker's inequality and the formula for KL divergence between Gaussians. 
\end{proof}

Finally, we are ready to prove \cref{thm:pacle-ftpl}.
\begin{proof}[Proof of \cref{thm:pacle-ftpl}]
  Let the parameters $T, \eta, \epapx, \alpha,\beta, \sigma$ be chosen as in \cref{def:params-offline}. Note that these parameter settings ensure that $1/\sigma \leq 1/\sqrt{\epbe}$, which in turn ensures that
  \begin{align}
3BH\zeta \leq 3BH \cdot C_{\ref{cor:ibe-linear}} \epbe d^{3/2} \cdot \left( \sqrt{d \log(d/\epbe^2)} + 1/\sqrt{\epbe}\right) \leq 1\label{eq:bhzeta-small},
  \end{align}
  by our assumption that $\epbe \leq c_0 (BH)^{-2} d^{-2}$ and as long as $c_0$ is sufficiently small. 
  
  Fix an arbitrary policy $\pi^\st \in \Pi$. Recall that $\pi\^t$ denotes the policy chosen by \Actor (\cref{alg:actor}) in step $t \in [T]$, and $f\^t = f^{w\^t}$. Moreover, let us write $M\^t := M^{f\^t, \pi\^t}$. The choices of $\alpha, \beta, \epapx$ in \cref{def:params-offline} ensure that, by \cref{lem:good-event-prob}, we have that, over the draw of $\MD$,  $\Pr(\ME_{\alpha,\beta,\sigma,\epapx}) \geq 1-\delta/2$. We next wish to use \cref{lem:mprime-m-diff} with $\pi = \pi\^t$, for each $t \in [T]$. To do so, we need to check that $\pi\^t \in \Pilinp[\sigma]$: indeed, $\pi_h\^t = \pi_{h, \theta_h\^t, \eta}$, where $\theta_h\^t$ satisfies $\| \theta_h\^t\|_2 \leq t \cdot \max_{s \leq t} \| w_h\^s \|_2\leq T\beta$ (by \cref{cor:qlin}, which is applicable because of \cref{eq:bhzeta-small}). Then $\frac{\eta}{\| \theta_h\^t \|_2} \geq \frac{\eta}{T\beta} =\sigma$, where we have used the definition of $\eta,\sigma$ in \cref{def:params-offline}.
  
  If we let $\MF\^t$ denote the sigma-algebra generated by $\MD, w\^1, \ldots, w\^t$, then \cref{lem:mprime-m-diff} ensures that, for each $t$, $\Pr(\ME_{\pi\^t} | \MF\^{t-1}) \geq 1-\delta/(2T)$. (In particular, we use here that the randomness in the call to \EstFeature in \Critic is chosen independently at each step $t$.) By a union bound, it follows that $\Pr\left( \ME_{\alpha,\sigma,\epapx} \cap \bigcap_{t \in [T]} \ME_{\pi\^t} \right) \geq 1-\delta$. We write $\ME^\st := \ME_{\alpha,\sigma,\epapx} \cap \bigcap_{t \in [T]} \ME_{\pi\^t} $. 

  By \cref{lem:mprime-m-diff}, under the event $\ME^\st$, for each $t \in [T]$, we have
  \begin{align}
    V_1\sups{M, \pi^\st}(x_1) - V_1\sups{M, \pi\^t}(x_1) \leq & V_1\sups{M\^t, \pi^\st}(x_1) - V_1\sups{M\^t, \pi\^t}(x_1) +3\beta H \zeta +  2\alpha \sum_{h=1}^H \| \E\sups{M\, \pi^\st}[\phi_h(x_h, a_h)]\|_{\Sigma_h^{-1}}\label{eq:m-mt-relate}. 
  \end{align}
  Next, using the performance difference lemma applied (\cref{lem:perf-diff}) to the MDP $M\^t$, we have
  \begin{align}
    & \sum_{t=1}^T V_1\sups{M\^t, \pi^\st}(x_1) - V_1\sups{M\^t, \pi\^t}(x_1)\nonumber\\
    =& \sum_{t=1}^T \sum_{h=1}^H \E\sups{M\^t, \pi^\st}\left[ Q_h\sups{M\^t, \pi\^t}(x_h, \pi_h^\st(x_h)) - Q_h\sups{M\^t, \pi\^t}(x_h, \pi_h\^t(x_h)) \right]\nonumber\\
    =& \sum_{t=1}^T \sum_{h=1}^H \E\sups{M\^t, \pi^\st} \left[ f_h\^t(x_h, \pi_h^\st(x_h)) - f_h\^t(x_h, \pi_h\^t(x_h))\right]\nonumber\\
    = &  \sum_{t=1}^T \sum_{h=1}^H \E\sups{M, \pi^\st} \left[ f_h\^t(x_h, \pi_h^\st(x_h)) - f_h\^t(x_h, \pi_h\^t(x_h))\right]\nonumber\\
    =& \sum_{h=1}^H \E\sups{M, \pi^\st} \left[ \sum_{t=1}^T f_h\^t (x_h, \pi_h^\st(x_h)) - f_h\^t(x_h, \pi\^t_h(x_h)) \right]\label{eq:regret-decomposition},
  \end{align}
  where the second equality uses \cref{lem:imag-mdp}, the third equality uses the fact that the dynamics of $M\^t$ are the same as those of $M$, and the final equality rearranges. Next, \cref{lem:ftpl-actor} guarantees that for each possible choice of $x_h$, we have
  \begin{align}
 \sum_{t=1}^T f_h\^t(x_h, \pi_h^\st(x_h)) - \sum_{t=1}^T f_h\^t(x_h, \pi_h\^t(x_h))\leq \eta^{-1} (2BH)^2 T + {\eta \sqrt{d}} \label{eq:indiv-state-reg}. 
  \end{align}
  Combining \cref{eq:regret-decomposition}, \cref{eq:indiv-state-reg}, and \cref{eq:m-mt-relate} yields that, under $\ME^\st$,
  \begin{align}
    \frac 1T \sum_{t=1}^T\left( V_1\sups{M, \pi^\st} -  V_1\sups{M, \pi\^t}(x_1)\right) \leq & 2\alpha \sum_{h=1}^H \| \E\sups{M, \pi^\st}[\phi_h(x_h, a_h)] \|_{\Sigma_h^{-1}} + \eta^{-1} (2BH)^2  + \frac{\eta \sqrt{d}}{T} + 3\beta H \zeta\label{eq:raw-final-bound}.
  \end{align}
  Note that the choices of $\eta, \sigma, \zeta$ in \cref{def:params-offline} gives that $1/\sigma \leq \epbe^{-1/2}$, meaning that
  \begin{align}
\zeta \leq C_{\ref{cor:ibe-linear}} \epbe d^{3/2} \cdot \left( \sqrt{d \log (d/(\epbe \sigma))} + \epbe^{-1/2} \right) \leq C \epbe^{1/2} d^{3/2} \log(1/\epbe)\label{eq:zeta-bound},
  \end{align}
  for some constant $C > 0$; note that we have used that $\epbe \leq d^{-1}$ and thus $\sqrt{d \log d} \leq O( \epbe^{-1/2}\log^{1/2}(1/\epbe))$ in the above display. Combining \cref{eq:raw-final-bound,eq:zeta-bound} gives
  \begin{align}
    & \frac 1T \sum_{t=1}^T\left( V_1\sups{M, \pi^\st} -  V_1\sups{M, \pi\^t}(x_1)\right) \nonumber\\
    \leq & 2\alpha \sum_{h=1}^H \| \E\sups{M, \pi^\st}[\phi_h(x_h, a_h)] \|_{\Sigma_h^{-1}} + \frac{4BH d^{1/4}}{\sqrt{T}} + 2BH\sqrt{d}  \epbe^{1/2} + 6CBH^2d^{3/2} \cdot \epbe^{1/2} \log(1/\epbe)\nonumber.
  \end{align}
  
  Thus, by our choice of $T = \frac{16B^2H^2d^{1/2}}{\epfinal^2}$ (\cref{def:params-offline}), it follows that  the policy $\hat \pi := \frac 1T \sum_{t=1}^T \pi\^t$ satisfies
  \begin{align}
    & V_1\sups{M, \pi^\st}(x_1) - V_1\sups{M, \hat \pi}(x_1) \nonumber\\
    \leq &    2\alpha \sum_{h=1}^H \| \E\sups{M, \pi^\st}[\phi_h(x_h, a_h)] \|_{\Sigma_h^{-1}} + \epfinal + O\left( BH^2d^{3/2} \cdot \epbe^{1/2} \log(1/\epbe)\right)\nonumber\\
    \leq&  O \left(d^{3/2}BH\epbe^{1/2}\log(1/\epbe) \sqrt{n} +  B H d \log^{1/2}(dnBH/(\epfinal\delta)) \right) \cdot  \left( \frac{H}{\sqrt{n}} + \sum_{h=1}^H \| \E\sups{M, \pi^\st}[\phi_h(x_h, a_h)] \|_{\Sigma_h^{-1}}\right) + \epfinal\nonumber,
  \end{align}
  which is equal to the desired bound.

  Finally, we analyze the computational cost of \cref{alg:actor}. The only nontrivial computation to be performed is the call to \Critic in \cref{line:call-critic}, which is made $T$ times. Thus, by \cref{lem:mprime-m-diff}, given an oracle which can solve the convex program \cref{eq:critic-program}, the overall computational cost of \cref{alg:actor} is $\poly(T, d, n, H, \log(1/\delta)/\epapx^2) \leq  \poly(d, H,n, \log(1/\delta), 1/\epfinal)$. The same guarantee holds without existence of such an oracle; the necessary modifications to the proof are described below in \cref{rmk:e2e-comp}. 
\end{proof}

  \begin{remark}\label{rmk:epbe0} Notice that in the case $\epbe = 0$, the final steps of the proof above yield the (slightly stronger) bound
    \begin{align}
      & V_1\sups{M, \pi^\st}(x_1) - V_1\sups{M, \hat \pi}(x_1) \nonumber\\
      \leq &  O \left(\frac{ B H d \log^{1/2}(dnBH/(\epfinal\delta)) }{\sqrt{n}}\right) \cdot   \sum_{h=1}^H \| \E\sups{M, \pi^\st}[\phi_h(x_h, a_h)] \|_{n\Sigma_h^{-1}} + \epfinal\nonumber.
    \end{align}

  \end{remark}

  \begin{remark}[Solving the convex program]
    \label{rmk:e2e-comp}
In this remark, we discuss the minor modifications necessary to the proof of \cref{thm:pacle-ftpl} to establish its guarantee without assumption of an oracle which can solve the program \cref{eq:critic-program}.

Note that the program \cref{eq:critic-program} is a convex program in $O(dH)$ variables consisting of linear equalities and $\ell_2$ norm constraints, for which all coefficients of the variables can be specified with $\log \poly(\alpha, \beta, d, n) \leq \log \poly(n, B, d, H, \log(1/\delta))$ bits. Thus, for any $\ep > 0$, the ellipsoid algorithm returns vectors $w, \xi \in \BR^{dH}$ which satisfy the constraints up to $\ep$ error and $\ep$-approximately minimize the objective \cref{eq:critic-a} in time $\poly(d, H, \log(nB/(\delta \ep)))$. Then it is immediate that \cref{it:vmf} of \cref{lem:mprime-m-diff} gives only the weaker guarantee $V_1^{M^{f,\pi}}(x_1) \leq V_1^{M, \pi}(x_1) + \ep$. Moreover, in the proof of \cref{it:all-piprime} of \cref{lem:mprime-m-diff}, the right-hand side of \cref{eq:vmf-vmpi-establish} has an additional $O(\ep \cdot \alpha H)$ term (as \cref{eq:critic-c} holds up to additive $\ep$), which may be absorbed in the first term (namely, $3\beta H \zeta_\sigma$) by increasing the constants, as long as $\ep$ is sufficiently small. In turn, \cref{lem:mprime-m-diff} is used to establish \cref{eq:m-mt-relate}: thus this equation gains an additional term of $\ep$ on the right-hand side as well as additional constant factor. This additional terms propagate to degrade the bound of \cref{thm:pacle-ftpl} by a constant factor. 
\end{remark}

}

\section{Lower bounds}
\label{sec:lb}
\arxiv{In this section we prove \cref{thm:lb-informal}, restated formally below as \cref{thm:lb-formal}. In \cref{sec:m-family}, we  construct a family of MDPs that forms the basis of the lower bound. Using this family of MDPs, we then prove \cref{thm:lb-informal} in \cref{sec:lb-proof}. }
\rlc{In this section we state \cref{thm:lb-informal} formally below as \cref{thm:lb-formal}, and sketch its proof (which is provided in full in the appendix).}
\begin{theorem}
  \label{thm:lb-formal}
Let $\epbe \in (0,1)$ and $n \in \BN$ be given, and set $d = H = 2$. Then there are state and action spaces $\MX, \MA$ with $|\MA| = 4$ as well as feature mappings $\phi_h : \MX \times \MA \to \BR^d$ ($h \in [H]$), so that the following holds, for any offline RL algorithm $\mathfrak{A}$.  There is some MDP $M$ which has inherent Bellman error with respect to the feature mappings $\phi_h$ bounded by $2\epbe$ and which satisfies \cref{asm:boundedness}, some distribution over datasets $\MD$ satisfying \cref{asm:offline-data-assumption}, and some policy $\pi^\st \in \Pilinpp$ so that: $\Cvg_{\MD, \pi^\st} = O(1)$ with probability 1 over the draw of $\MD$ yet the output policy $\hat \pi$ of $\mathfrak{A}$ satisfies
  \begin{align}
\E\left[V_1^{\pi^\st}(x_1) - V_1^{\hat \pi}(x_1)\right] \geq \Omega \left( \sqrt{\epbe} + \frac{1}{\sqrt{n}} \right)\label{eq:offline-lb-formal},
  \end{align}
  where the expectation is over the draw of $\MD$ and the randomness in $\mathfrak{A}$. 
\end{theorem}

\rlc{}

\arxiv{\rlc{\section{Proof of \cref{thm:lb-formal}}\label{sec:lb-proof-rlc}}
\rlc{In this section we prove \cref{thm:lb-formal}. First, in \cref{sec:m-family} we introduce the family of MDP instances used to prove the theorem, and in \cref{sec:lb-proof} we analyze the performance of any algorithm on this family.}
\subsection{Construction of the family $\MM_{\epbe}$.}
\label{sec:m-family}
Fix $\epbe > 0$. For simplicity we assume that $L := 1/\sqrt{\epbe}$ is an even integer. Given bits $\brew, \binit \in \{0,1\}$ and $( b_{\ell, e} )_{\ell \in [L], e \in \{0,1\}} \in \{0,1\}^{2L}$, we write $\bb = (\brew, \binit, (b_{\ell, e})_{\ell \in [L], e \in \{0,1\}}) \in \{0,1\}^{2L+2}$ to denote the collection of all these bits. 
We construct a class $\MM_{\epbe} = \{ M^{\bb} \}_{\mathbf{b} \in \{0,1\}^{2L+2} }$  of MDPs $M^{\bb}$, each of which has inherent Bellman error bounded above by $\epbe$ with respect to some fixed feature mappings. All MDPs $M \in \MM_{\epbe}$ have horizon $H=2$ and feature dimension $d=2$. The state and action spaces of each $M \in \MM_{\epbe}$ are given as follows:
  \begin{align}
\MX = \{ \mf s_1, \mf t_1, \mf s_2, \bar{\mf s}_2, \mf q_2\} \cup \{\mf s_2^\zeta, \mf t_{2,0}^\zeta, \mf t_{2,1}^\zeta \}_{\zeta \in [0,1]}\qquad \MA = \{0,1,2,3\}\label{eq:xa-lb-define},
  \end{align}
  where $\zeta$ ranges over $[0,1]$. (We remark that we will only need to use the states $\mf s_2^\zeta, \mf t_{2,0}^\zeta, \mf t_{2,1}^\zeta$ for values of $\zeta$ which are in $\{ 0, \epbe, 2\epbe, \ldots, \sqrt{\epbe} \}$, but to simplify notation we opt to define states corresponding to all $\zeta \geq 0$.)
  \paragraph{Feature vectors.} The feature vectors corresponding to $\MX, \MA$ are defined below. We use the convention that if we specify fewer than $4$ actions at a state, then all remaining actions at the state have equal behavior (i.e., feature vectors and transition) to action 0 at that state. For normalizing constant $\cphi := (\sqrt 2)^{-1}$, we define, for each $h \in [2]$:
  \begin{itemize}
  \item $ \phi_h(\mf s_1, 0) = \cphi\cdot (1,0)$, $\phi_h(\mf s_1, 1) = \cphi\cdot(0,1)$. %
  \item $\phi_h({\mf t}_1, 0) = \cphi\cdot(1,1)$ and $\phi_h({\mf t}_1, 1) = \cphi\cdot(1,-1)$.
  \item $\phi_h(\bar{\mf s}_2, 0) = \cphi\cdot(0,0)$. %
  \item $\phi_h(\mf s_2, 0) = \cphi\cdot(0, 1)$, $\phi_h(\mf s_2, 1) = \cphi\cdot(0, -1)$. 
    \item For each $\zeta \in [0,1]$, $\phi_h(\mf s_2^\zeta, 0) = \cphi\cdot(0, \zeta)$, $\phi_h(\mf s_2^\zeta, 1) = \cphi\cdot(1,0)$, $\phi_h(\mf s_2^\zeta, 2) = \cphi\cdot(0, -\zeta)$, $\phi_h(\mf s_2^\zeta, 3) = \cphi\cdot(-1, 0)$. 
    \item For each $\zeta \in [0,1]$ and $b \in \{0,1\}$, $\phi_h(\mf t_{2,b}^\zeta, 0) = \cphi\cdot(1, (1-2b) \zeta)$ and $\phi_h(\mf t_{2,b}^\zeta, 1) = \cphi\cdot(-1, -(1-2b)\zeta)$. Via slight abuse of notation, we identify $\mf t_{2,0}^0, \mf t_{2,1}^0$, i.e., $\mf t_{2,0}^0 = \mf t_{2,1}^0$.
  \end{itemize}

\paragraph{Transitions and rewards.}  Consider any $M := M^{\bb} \in \MM_{\epbe}$, and write $\bb = (\brew, \binit, \{ b_{\ell, e} \}_{\ell \in [L], e \in \{0,1\}}) \in \{0,1\}^{2L+2}$. We proceed to define the initial state distribution, transitions and rewards of $M$. The initial state distribution has all its mass concentrated on $\mf t_1$. The transitions are defined as follows:
  \begin{itemize}
  \item $(\mf s_1, 1)$ transitions to $\bar{\mf s}_2$ with probability 1.
\item $(\mf s_1, 0)$ and $(\mf t_1, \binit)$ each transition to the distribution that puts mass $1/L$ on each of the states $\mf s_2^{\ell\cdot \epbe}$, for $\ell \in [L]$. %
\item $(\mf t_1, 1-\binit)$ transitions to the distribution which:
  \begin{itemize}
  \item Puts mass $1/(2L)$ on $\mf t_{2, 0}^{(\ell - b_{\ell,0}) \cdot \epbe}$ for each $\ell \in [L]$;
  \item Puts mass $1/(2L)$ on $\mf t_{2, 1}^{(\ell-b_{\ell,1}) \cdot \epbe}$ for each $\ell \in [L]$.
  \end{itemize}
\end{itemize}
The rewards are defined as follows:
\begin{itemize}
\item $r_1(x, a) = 0$ for all $x \in \MX, a \in \MA$, i.e., rewards at step 1 are linear with $\thetar_1 = (0,0)$. 
\item %
  The rewards at step 2 are linear with respect to the coefficient vector $\thetar_2 = (1-2\brew,1/R)$, where $R := 16$. 
\end{itemize}

\paragraph{Verifying low inherent bellman error.} Next, we verify that each MDP $M^\bb \in \MM_{\epbe}$ has low inherent Bellman error and satisfies \cref{asm:boundedness} with respect to the feature mappings $\phi_h$ defined above.
\begin{lemma}[Low inherent Bellman error and boundedness]
  \label{lem:lb-ibe-bnd}
For any $M = M^\bb \in \MM_{\epbe}$, $M$ has inherent Bellman error $2\epbe$ and satisfies \cref{asm:boundedness} with $B = \sqrt{2}$. 
\end{lemma}
\begin{proof}
  It is straightforward to see that
  \begin{align}
 \MB_1 = \MB_2 = \left\{ w \in \BR^2 \ : \  \cphi|w_1| + \cphi|w_2| \leq 1 \right\}\nonumber, %
  \end{align}
  meaning that the second item of \cref{asm:boundedness} is satisfied with $B = 1/\alpha = \sqrt{2}$. Moreover, by choice of $\alpha$ the first item is immediate, and the third item holds since for all $(x,a)$, $|r_2(x,a)| \leq \alpha \cdot \| \thetar_2 \|_1 = \alpha \cdot (1+1/R) \leq 1$. 
  
  We must verify \cref{eq:ibe-def} for $h \in \{1,2\}$. We begin with the case $h=1$. For $\theta \in \BR^2$, define
  \begin{align}
\MT_1 \theta := \left(\frac{1}{L} \sum_{\ell=1}^L \left(\max_{a \in \MA} \lng \theta, \phi_2(\mf s_2^{\ell \cdot \epbe}, a)\rng\right), 0 \right)\nonumber,
  \end{align}
  which belongs to $\MB_1$ since its first coordinate is bounded above in absolute value by $\alpha \cdot \max\{ |\theta_1|, L\epbe |\theta_2| \} \leq \alpha \cdot (|\theta_1| + |\theta_2|) \leq 1$. 
  
  Since $r_1(x,a) = 0$ for all $x,a$, we certainly have that, for each $(x,a) \in (\{\mf s_1\} \times \MA) \cup (\mf t_1, \binit)$, $(\phi_1(x,a))_1 = 1$ and thus
  \begin{align}
\lng \phi_h(x,a), \MT_1 \theta \rng = \E_{x' \sim P_h\sups{M}(x,a)} \left[r_1(x,a) +  \max_{a' \in \MA} \lng \phi_2(x', a'), \theta \rng \right]\label{eq:lbc-lb-1},
  \end{align}
  meaning that \cref{eq:ibe-def} holds for these $(x,a)$. 
Now consider the state-action pair $(x,a) = (\mf t_1, 1-\binit)$. For each $\ell \in [L]$ and $\theta \in \BR^2$, we have
  \begin{align}
    & \left|\max_{a' \in \MA} \lng \phi_2(\mf s_2^{\ell\epbe}, a'), \theta \rng - \frac 12 \max_{a' \in \MA} \lng \phi_2(\mf t_{2,0}^{(\ell - b_{\ell, 0} ) \epbe}, a'), \theta \rng - \frac 12 \max_{a' \in \MA} \lng \phi_2(\mf t_{2,1}^{(\ell - b_{\ell, 1}) \epbe}, a'), \theta \rng \right|\nonumber\\
    \leq & \alpha |\theta_2| \epbe + \left|\max_{a' \in \MA} \lng \phi_2(\mf s_2^{\ell\epbe}, a'), \theta \rng - \frac 12 \max_{a' \in \MA} \lng \phi_2(\mf t_{2,0}^{\ell \epbe}, a'), \theta \rng - \frac 12 \max_{a' \in \MA} \lng \phi_2(\mf t_{2,1}^{\ell \epbe}, a'), \theta \rng \right|=\alpha |\theta_2| \epbe \leq \epbe\nonumber.
  \end{align}
  It follows that
  \begin{align}
\left|\lng \phi_h(\mf t_1, 1-\binit), \MT_1 \theta \rng - \E_{x' \sim P_h\sups{M}(\mf t_1, \binit)} \left[r_1(\mf t_1, \binit) +  \max_{a' \in \MA} \lng \phi_2(x', a'), \theta \rng \right]\right| \leq \epbe\nonumber,
  \end{align}
  verifying \cref{eq:ibe-def} holds for $(x,a) = (\mf t_1, \binit)$. Finally, the validity of \cref{eq:ibe-def} for $(x,a) = (\mf s_1, 1)$ is immediate since $(\MT_1 \theta)_2 = 0$.

  Next we verify \cref{eq:ibe-def} for $h=2$. Since all feature vectors are identically 0 at step $h=H+1 = 3$ (by convention), we take $\MT_2 \theta = \thetar_2$ (which is in $\MB_2$ since $\alpha \cdot (1 + 1/R) \leq 1$), and satisfies $r_2(x,a) = \lng \phi_2(x,a), \thetar_2 \rng$ for all $(x,a)$. 
\end{proof}

\subsection{Proof of \cref{thm:lb-formal}}
\label{sec:lb-proof}
We are now ready to prove \cref{thm:lb-formal}. %
\begin{proof}[Proof of \cref{thm:lb-formal}]
First, note that if $1/\sqrt{n} \geq \sqrt{\epbe}$, then the lower bound is straightforward and well-known. In particular, consider $\MX := \{\mf s_1, \mf s_2\}, \MA := \{ 0,1\}$ with $\phi_1(x, 0) = (1, 0), \phi_1(x, 1) = (0,1)$ for each $x \in \MX$, and $\phi_2(\mf s_1, a) = (1,0), \phi_2(\mf s_2, a) = (0,1)$ for each $a \in \MA$. We let the rewards be linear with respect to some vectors $\thetar_1, \thetar_2 \in \BR^2$ with $\thetar_1 = (0,0)$ some choice of $\thetar_2 \in \{ (0,1), (1,0) \}$. The initial state is $\mf s_1$; $(\mf s_1, 0)$ transitions to $\mf s_1$ with probability $1/2 + (10\sqrt{n})^{-1}$ (and to $\mf s_2$ with the remaining probability), and $(\mf s_1, 1)$ transitions to $\mf s_1$ with probability $1/2 - (10\sqrt{n})^{-1}$ (and to $\mf s_2$ with the remaining probability). Finally, the dataset $\MD$ consists of $n/4$ transitions from each of the tuples $(x,a)  \in \{ (\mf s_1, 0), (\mf s_1, 1), (\mf s_2, 0), (\mf s_2, 0) \}$. It is straightforward to see that the inherent Bellman error is 0 and that \cref{asm:boundedness,asm:offline-data-assumption} are satisfied. Moreover, it follows from well-known arguments \cite[Chapter 15]{lattimore2020bandit} that for any algorithm $\mathfrak{A}$, there is some choice of $\thetar_2$ as above so that the optimal policy $\pi^\st$ (i.e., defined by $\pi_1^\st(\mf s_1) = 0$ if $\thetar_2 = (1,0)$, and $\pi_1^\st(\mf s_1) = 1$ if $\thetar_2 = (0,1)$) and the output policy $\hat \pi$ of $\mathfrak{A}$ satisfy \cref{eq:offline-lb-formal}. 
  
For the remainder of the proof we may therefore assume that $\sqrt{\epbe} > 1/\sqrt{n}$. Moreover, by decreasing $\epbe$ by a constant factor, we may assume that $L = 1/\sqrt{\epbe}$ is an even integer.  We use the state and action spaces defined in \cref{eq:xa-lb-define} and the feature mappings $\phi_h$ defined in \cref{sec:m-family} above. Fix some randomized offline RL algorithm $\mathfrak{A}$. We will choose some MDP $M \in \MM_{\epbe}$ and define a distribution over datasets $\MD$ satisfying \cref{asm:offline-data-assumption} so that \cref{eq:offline-lb-formal} holds for some $\pi^\st$.

  Consider some $\bb = (\brew, \binit, (b_{\ell, e})_{\ell \in [L], e \in \{0,1\}}) \in \{0,1\}^{2L+2}$, to be specified below, and set $M = M^\bb$. By \cref{lem:lb-ibe-bnd}, $M$ has inherent Bellman error $2\epbe$ and satisfies \cref{asm:boundedness} with $B = \sqrt{2}$. 
  The dataset $\MD$ consists of tuples $(h_i, x_i, a_i, r_i, x_i')$, $i \in [n]$, drawn as follows:
\begin{itemize}
\item There are $n/3$ points of the form $(1, \mf s_1,1,0, \bar{\mf s}_2)$.
\item There are $n/3$ points of the form $(1, \mf s_1, 0,0, x_i)$, where $x_i \sim P_1^M(\cdot \mid \mf s_1, 0)$ for $i \in [n/3]$. 
\item There are $n/3$ points of the form $(2, \mf s_2^{L \epbe}, 0, L \epbe \alpha / R, \perp)$. 
\end{itemize}
It is immediate that the distribution of $\MD$ satisfies \cref{asm:offline-data-assumption}.

\paragraph{Controlling the coverage coefficient.} Note that we have
\begin{align}
  \Sigma_1 =& \frac n3 \cdot \phi_1(\mf s_1, 1) \phi_1(\mf s_1, 1)^\t + \frac n3 \cdot \phi_1(\mf s_1, 0) \phi_1(\mf s_1, 0)^\t = \frac{\cphi^2 n}{3} \cdot I_2\nonumber\\
  \Sigma_2 =& \frac n3 \cdot \phi_2(\mf s_2, 0) \phi_2(\mf s_2, 0)^\t =   \frac{n \cdot (\alpha L\epbe)^2}{3} \cdot  \matx{0 & 0 \\ 0 & 1 }\nonumber.
\end{align}

Define $w_1^\st := (1, 1-2\binit)$, $w_2^\st := (0,1)$, and $\pi_h^\st := \pi_{h, w_h^\st, 0}$ for $h \in [2]$, and write $\pi^\st = (\pi_1^\st, \pi_2^\st) \in \Pilinpp$. Note that %
\begin{align}
\E^{M, \pi^\st}[\phi_1(x_1, a_1)] = \alpha \cdot (1,1-2\binit), \qquad \E^{M, \pi^\st}[\phi_2(x_2, a_2)]= \alpha \cdot (0, L\epbe)\nonumber,
\end{align}
so that
\begin{align}
\| \E^{M, \pi^\st}[\phi_1(x_1, a_1)] \|_{n\Sigma_1^{-1}} = \sqrt{6}, \qquad \| \E^{M, \pi^\st}[\phi_2(x_2, a_2)] \|_{n\Sigma_2^{-1}} = \sqrt{3}\nonumber.
\end{align}

The value function of $\pi^\st$ for $M = M^\bb \in \MM_{\epbe}$ does not depend on the choice of $\bb$ and may be computed as follows: first, note that
$
V_2^{M, \pi^\st}(\mf s_{2}^\zeta) = \lng \phi_2(\mf s_{2}^\zeta, 0), \thetar_2 \rng = \alpha\zeta/R ,
$
for each $\zeta \in [0,1]$, 
which implies that $V_1^{M, \pi^\st}(\mf t_1) = Q_1^{M, \pi^\st}(\mf t_1, b_1) = V_2^{M, \pi^\st}(\mf s_{2}) = \frac{\alpha}{R} \cdot \frac{(1 + \cdots + L)\epbe}{L}  = \frac{\alpha}{R} \cdot (L+1)\epbe/2$. %

\paragraph{Controlling the performance of $\mathfrak{A}$.} Let $\hat \pi$ denote the (random) output of the algorithm $\mathfrak{A}$, given the random dataset $\MD$. Note that the distribution $P_1^M(\cdot \mid \mf s_1, 0)$ does not depend on the choice of $\bb$. Thus, the distribution of $\hat \pi$, which we denote by $\MD_0 \in \Delta(\PiM)$, is the same for all possible choices of $\bb$. \cref{lem:bad-bb} below, which is the main technical component of the proof, yields that there is some $\bb$ so that \cref{eq:offline-lb-formal} holds (as we have assumed $1/\sqrt{n} < \sqrt{\epbe}$), which completes the proof of \cref{thm:lb-formal}.
\begin{lemma}
  \label{lem:bad-bb}
Let $\MD_0 \in \Delta(\PiM)$ be an arbitrary distribution over Markov policies. %
  Then there is some choice of $\bb$ so that, for $M := M^{\bb} \in \MM_{\epbe}$, we have
  \begin{align}
\E_{\pi \sim \MD_0} \left[ V_1^{M, \pi}(\mf t_1) \right] \leq V_1^{M, \pi^\st}(\mf t_1) - \Omega \left(  \sqrt{\epbe}\right)\label{eq:d0-deficit}.
  \end{align}
\end{lemma}
The proof of \cref{lem:bad-bb} is provided below. 
\end{proof}
\begin{proof}[Proof of \cref{lem:bad-bb}]
  Let $\MD_0 \in \Delta(\PiM)$ be given. Choose $\binit \in \{0,1\}$ so that $Z_0 := \E_{\pi \sim \MD_0}[\pi(1-\binit \mid \mf t_1)] \geq 1/2$. 
  For $0 \leq \ell \leq L$, define
  \begin{align}
    \bar\eta(\ell) :=&  \E_{\pi \sim \MD_0} \left[ \pi(1-\binit \mid \mf t_1) \cdot \left( \pi(0 \mid \mf t_{2,0}^{\ell \epbe}) - \pi(1 \mid \mf t_{2,1}^{\ell\epbe}) \right)\right]\nonumber\\
    \bar\gamma(\ell) :=& \E_{\pi \sim \MD_0} \left[ \pi(\binit \mid \mf t_1) \cdot \left( \pi (1 \mid \mf s_2^{\ell\epbe}) - \pi(3 \mid \mf s_2^{\ell\epbe}) \right)\right]\nonumber.
  \end{align}
  Let $\MD \in \Delta(\Pi)$ be defined by
  \begin{align}
\MD(\pi) := \frac{\MD_0(\pi) \cdot \pi(1-\binit \mid \mf t_1)}{\E_{\pi' \sim \MD_0}[\pi'(1-\binit \mid \mf t_1)]} = \frac{\MD_0(\pi) \cdot \pi(1-\binit \mid \mf t_1)}{Z_0}\nonumber.
  \end{align}
  For $0 \leq \ell \leq L$, define
  \begin{align}
    \rho(\ell) := \E_{\pi \sim \MD}\left[\pi(0 \mid \mf t_{2,0}^{\ell \epbe}) + \pi(1 \mid \mf t_{2,1}^{\ell\epbe})\right]. \nonumber %
  \end{align}
Consider any choice of $\brew \in \{0,1\}$ and $b_{\ell,e} \in \{0,1\}$ for each $\ell \in [L], e \in \{0,1\}$, and write $\bb = (\brew, \binit, (b_{\ell, e})_{\ell \in [L], e \in \{0,1\}})$. Then for any policy $\pi$, $\zeta > 0$, and $b \in \{0,1\}$, we have %
  \begin{align}
    V_2^{M^{\bb}, \pi}(\mf t_{2,b}^\zeta) =&  \pi(b \mid \mf t_{2,b}^\zeta) \cdot \lng \phi_2(\mf t_{2,b}^\zeta, b), \thetar_2 \rng + \pi(1-b \mid \mf t_{2,b}^\zeta) \cdot \lng \phi_2(\mf t_{2,b}^\zeta, 1-b), \thetar_2 \rng\nonumber\\
    =&\cphi \cdot  \pi(b \mid \mf t_{2,b}^\zeta) \cdot \lng (1-2b, \zeta), (1-2\brew, 1/R)\rng + \cphi \cdot (1-\pi(b \mid \mf t_{2,b}^\zeta)) \cdot \lng (2b-1, -\zeta), (1-2\brew, 1/R) \rng\nonumber\\
    =& 2\cphi \cdot  \pi(b \mid \mf t_{2,b}^\zeta) \cdot ((1-2b)(1-2\brew) + \zeta/R) - \cphi \cdot ((1-2b)(1-2\brew) + \zeta/R)\label{eq:v2-expand}.
  \end{align}
  Hence
  \begin{align}
    V_2^{M^\bb, \pi}(\mf t_{2,0}^\zeta) + V_2^{M^\bb, \pi}(\mf t_{2,1}^\zeta) =& 2\cphi(1-2\brew) \cdot \left(  \pi(0 \mid \mf t_{2,0}^\zeta) -  \pi (1 \mid \mf t_{2,1}^\zeta) \right) + 2\zeta\cphi R^{-1} \left( \pi(0 \mid \mf t_{2,0}^\zeta) + \pi(1 \mid \mf t_{2,1}^\zeta) - 1\right)\label{eq:v2-add}.
  \end{align}
  Moreover,
  \begin{align}
    V_2^{M^\bb, \pi}(\mf s_2^\zeta) =& \cphi\cdot (1-2\brew) \cdot \left( \pi(1 \mid \mf s_2^\zeta) - \pi(3 \mid \mf s_2^\zeta) \right) + \cphi R^{-1} \zeta \cdot \left( \pi(0 \mid \mf s_2^\zeta) - \pi(2 \mid \mf s_2^\zeta) \right)\nonumber\\
    \leq & \cphi \cdot (1-2\brew) \cdot \left( \pi(1 \mid \mf s_2^\zeta) - \pi(3 \mid \mf s_2^\zeta) \right) + \cphi R^{-1} \zeta\label{eq:v2-s}.
  \end{align}

    \paragraph{Case 1: $\left| \frac 1L \sum_{\ell=1}^L (\bar\eta(\ell) + \bar\gamma(\ell)) \right| > \frac{\sqrt{\epbe}}{10}$.} Recall our choice $\binit$ from above. Set $\brew = 0$ if $\sum_{\ell=1}^\ell (\bar \eta(\ell) + \bar\gamma(\ell)) < 0$, and otherwise $\brew = 1$. Finally set $b_{\ell, e} = 0$ for all $\ell \in [L], e \in \{0,1\}$, and write $\bb = (\brew, \binit, (b_{\ell, e})_{\ell, e})$. 
  Therefore,
  \begin{align}
    & \E_{\pi \sim \MD_0}\left[ V_1^{M^\bb, \pi}(\mf t_1) \right]\nonumber\\
    =& \frac{1}{2L} \sum_{\ell=1}^L \E_{\pi \sim \MD_0}\left[\pi(1-\binit \mid \mf t_1) \cdot \left( V_2^{M^\bb, \pi}(\mf t_{2,0}^{\ell \epbe}) + V_2^{M^\bb, \pi}(\mf t_{2,1}^{\ell \epbe})\right) + 2 \pi(\binit \mid \mf t_1) \cdot V_2^{M^\bb, \pi}(\mf s_2^{\ell\epbe}) \right]\label{eq:v1-decompose-first}\\
    \leq& \frac{1}{2L} \sum_{\ell=1}^L2\cphi(1-2\brew) \cdot  \E_{\pi \sim \MD_0} \left[ \pi(1-\binit \mid \mf t_1) \cdot (\pi(0 \mid \mf t_{2,0}^{\ell\epbe}) - \pi(1 \mid \mf t_{2,1}^{\ell\epbe}) ) \right]\nonumber\\
                                                                 &+\frac{1}{2L} \sum_{\ell=1}^L  2 \cphi R^{-1} \ell\epbe \E_{\pi \sim \MD_0}\left[ \pi(1-\binit \mid \mf t_1) \cdot (\pi(0 \mid \mf t_{2,0}^{\ell\epbe}) + \pi(1 \mid \mf t_{2,1}^{\ell\epbe})-1)\right]\nonumber\\
                                                                 &+ \frac{1}{L} \sum_{\ell=1}^L \left(\cphi (1-2\brew) \cdot \E_{\pi \sim \MD_0} \left[ \pi(\binit \mid \mf t_1) \cdot(\pi(1\mid \mf s_2^{\ell\epbe}) - \pi(3 \mid \mf s_2^{\ell\epbe})) \right] +  \cphi R^{-1} \E_{\pi \sim \MD_0}\left[ \pi(\binit \mid \mf t_1) \cdot \ell\epbe \right] \right)  \nonumber\\
    =& \cphi(1-2\brew) \frac{1}{L} \sum_{\ell=1}^L \left(\bar\eta(\ell) + \bar\gamma(\ell)\right) + \frac{1}{L} \sum_{\ell=1}^L Z_0 \cdot \cphi R^{-1} \ell \epbe \cdot (\rho(\ell) - 1) + \frac 1L \sum_{\ell=1}^L \cphi R^{-1} (1-Z_0) \cdot \ell\epbe\label{eq:v1pi-decompose}\\
    \leq & -\frac{\cphi}{L} \left| \sum_{\ell=1}^L (\bar\eta(\ell) + \bar\gamma(\ell)) \right| + \cphi R^{-1} (L+1)\epbe/2\nonumber\\
    =&  -\frac{\cphi}{L} \left| \sum_{\ell=1}^L (\bar\eta(\ell)  + \bar\gamma(\ell))\right| + V_1^{M^\bb, \pi^\st}(\mf t_1)  \leq -\frac{\cphi \sqrt{\epbe}}{10} + V_1^{M^\bb, \pi^\st}(\mf t_1)\nonumber,
  \end{align}
  where the first inequality uses \cref{eq:v2-add,eq:v2-s}, 
  and the second inequality uses $\rho(\ell) \leq 2$ for each $\ell \in [L]$ as well as our choice of $\brew$. The above chain of inequalities thus verifies \cref{eq:d0-deficit} in this case.

From here on, we assume that Case 1 does not hold.   Since $\mf t_{2,0}^0 = \mf t_{2,1}^0$, we have  $\rho(0) -1= 0$. Therefore, 
 either $\sum_{\ell=L/2}^L (\rho(\ell)-1) \leq \frac{L}{2} \cdot 1/2$ or $\frac 1L \sum_{\ell=1}^L [\rho(\ell) - \rho(\ell-1)]_+ \geq 1/(2L) = \sqrt{\epbe}/2$. %
  Thus, it suffices to consider the (exhaustive) Cases 2 and 3 below:

  \paragraph{Case 2: $\sum_{\ell=L/2}^L (\rho(\ell)-1) \leq \frac{L}{2} \cdot 1/2$.} We make the same choice of $\bb$ as in Case 1. Then using \cref{eq:v1pi-decompose}, we have
  \begin{align}
    \E_{\pi \sim \MD_0} \left[ V_1^{M^\bb, \pi}(\mf t_1)\right] =&  \cphi(1-2\brew) \frac{1}{L} \sum_{\ell=1}^L (\bar\eta(\ell) + \bar\gamma(\ell)) + \frac{\cphi Z_0}{R L} \sum_{\ell=1}^L (\rho(\ell)-1) \cdot \ell\epbe + \frac{\cphi(1-Z_0)}{RL} \sum_{\ell=1}^L \ell\epbe\nonumber\\
    \leq & \frac{\cphi Z_0}{R L} \sum_{\ell=1}^L (\rho(\ell)-1) \cdot \ell\epbe + \cphi R^{-1} (1-Z_0) \cdot (L+1)\epbe/2\nonumber\\
    \leq & \frac{\cphi (L+1)\epbe}{2R}\cdot (1-Z_0) + \frac{\cphi Z_0}{R} \cdot \left( \frac{(L+1)\epbe}{2} - \frac{L}{4} \cdot \frac{L\epbe/2}{L} \right)\nonumber\\
    = &\frac{\cphi(L+1)\epbe}{2R}  - \frac{\cphi Z_0 \cdot L\epbe}{8R}  \nonumber\\
    \leq & V_1^{M^\bb, \pi^\st}(\mf t_1) - \cphi R^{-1} L \epbe/16 =  V_1^{M^\bb, \pi^\st}(\mf t_1) -\cphi R^{-1} \sqrt{\epbe}/16\nonumber,
  \end{align}
  where the second inequality uses our assumption that $\sum_{\ell=L/2}^L (\rho(\ell)-1) \leq L/4$, and the final inequality uses the fact that $Z_0 \geq 1/2$. 

  \paragraph{Case 3: $\sum_{\ell=1}^L [\rho(\ell) - \rho(\ell-1)]_+ \geq \sqrt{\epbe}/2$.} %
  For each $\ell \in [L]$ and $b \in \{0,1\}$, define $\rho_b(\ell) := \E_{\pi \sim \MD} \left[ \pi(b \mid \mf t_{2,b}^{\ell \epbe})\right]$, so that $\rho(\ell) = \rho_0(\ell) + \rho_1(\ell)$. We may choose some $e^\st \in \{0,1\}$ so that $\sum_{\ell=1}^L [\rho_{e^\st}(\ell) - \rho_{e^\st}(\ell-1)]_+ \geq \sqrt{\epbe}/4$. Choose $\brew = {e^\st}$ and define the values $b_{\ell, e}$ as follows:
  \begin{align}
    b_{\ell, e} := \begin{cases}
      1 &: \rho_{e^\st}(\ell) - \rho_{e^\st}(\ell-1) \geq 0,\ e = e^\st \\
      0 &: \mbox{otherwise}. 
    \end{cases}\label{eq:define-ble}
  \end{align}
  Write $\bb = (\binit, \brew, (b_{\ell, e})_{\ell, e})$. 
Using \cref{eq:v2-expand}, for each $\ell \in [L]$, we have
  \begin{align}
    &  V_2^{M^{\bb},\pi}(\mf t_{2,0}^{(\ell - b_{\ell, 0}) \epbe}) +  V_2^{M^{\bb},\pi}(\mf t_{2, 1}^{(\ell-b_{\ell,1}) \epbe}) \nonumber\\
     =&  2 \cphi \pi(1-e^\st \mid \mf t_{2,1-e^\st}^{\ell\epbe}) \cdot ((2e^\st-1)(1-2\brew) + \ell\epbe/R) - \cphi ((2e^\st-1)(1-2\brew) + \ell\epbe/R) \nonumber\\
     &+  2 \cphi \pi(e^\st \mid \mf t_{2,e^\st}^{(\ell-b_{\ell,e^\st})\epbe}) \cdot ((1-2e^\st)(1-2\brew) + (\ell-b_{\ell,e^\st})\epbe/R) - \cphi ((1-2e^\st)(1-2\brew) + (\ell-b_{\ell,e^\st})\epbe/R)\nonumber\\
    =&  2\cphi \pi(1-e^\st \mid \mf t_{2,1-e^\st}^{\ell\epbe}) \cdot ((2e^\st - 1)(1-2\brew) + \ell\epbe/R) - \cphi (2\ell-b_{\ell,e^\st})\epbe/R \nonumber\\
    &-  2\cphi \pi(e^\st \mid \mf t_{2,e^\st}^{(\ell-b_{\ell,e^\st})\epbe}) \cdot ((2e^\st - 1)(1-2\brew) - (\ell-b_{\ell,e^\st})\epbe/R)\nonumber\\
    \leq & 2\cphi(2e^\st - 1)(1-2\brew) \left( \pi(1-e^\st \mid \mf t_{2, 1-e^\st}^{\ell \epbe}) - \pi(e^\st \mid \mf t_{2,e^\st}^{(\ell-b_{\ell,e^\st})\epbe})\right)+ 2\cphi\ell\cdot\epbe/R\nonumber.
  \end{align}
Combining the above display with \cref{eq:v2-add}, we obtain that
  \begin{align}
    & V_2^{M^\bb, \pi}(\mf t_{2,0}^{(\ell - b_{\ell,0})\epbe}) + V_2^{M^\bb, \pi}(\mf t_{2,1}^{(\ell-b_{\ell,1})\epbe}) - \left( V_2^{M^\bb, \pi}(\mf t_{2,0}^{\ell\epbe}) + V_2^{M^\bb, \pi}(\mf t_{2,1}^{\ell\epbe})\right)\nonumber\\
    \leq & \left(2\cphi(2e^\st - 1)(1-2\brew) \left( \pi(1-e^\st \mid \mf t_{2, 1-e^\st}^{\ell \epbe}) - \pi(e^\st \mid \mf t_{2,e^\st}^{(\ell-b_{\ell,e^\st})\epbe})\right)+ 2\cphi R^{-1} \ell\cdot\epbe\right)\nonumber\\
    &\quad  -\left( 2\cphi(1-2\brew) \cdot \left(  \pi(0 \mid \mf t_{2,0}^{\ell\epbe}) -  \pi (1 \mid \mf t_{2,1}^{\ell\epbe}) \right) + 2\cphi R^{-1} \ell\epbe \left( \pi(0 \mid \mf t_{2,0}^{\ell\epbe}) + \pi(1 \mid \mf t_{2,1}^{\ell\epbe}) - 1\right)\right)\nonumber\\
    \leq &2\cphi(1-2e^\st)(1-2\brew) \cdot \pi(e^\st \mid \mf t_{2,e^\st}^{(\ell - b_{\ell,e^\st})\epbe}) - 2\cphi(1-2\brew)(1-2e^\st) \cdot \pi(e^\st \mid \mf t_{2,e^\st}^{\ell\epbe}) + 4\cphi R^{-1} \ell \epbe\nonumber\\
    =& -2\cphi \cdot \left( \pi(e^\st \mid \mf t_{2,e^\st}^{\ell \epbe}) - \pi(e^\st \mid \mf t_{2,e^\st}^{(\ell - b_{\ell, e^\st})\epbe} ) \right) + 4\cphi R^{-1} \ell\epbe\label{eq:bound-diff-v2},
  \end{align}
  where the final equality uses our choice of $\brew = e^\st$. 
  Define $b_{\ell, e}' = 0$ for all $\ell \in [L], e \in \{0,1\}$, and set $\bb' := (\brew, \binit, (b_{\ell,e})'_{\ell, e})$. Then, using \cref{eq:v1-decompose-first} as well as the fact that $V_2^{M^\bb, \pi}(\mf s_2^\zeta) = V_2^{M^{\bb'}, \pi}(\mf s_2^\zeta)$ for all $\pi$ and $\zeta$, 
  \begin{align}
    & \frac{1}{Z_0} \E_{\pi \sim \MD_0} \left[V_1^{M^\bb, \pi}(\mf t_1)   - V_1^{M^{\bb'}, \pi}(\mf t_1) \right]\nonumber\\
    = & \frac{1}{2LZ_0} \sum_{\ell=1}^L \E_{\pi \sim \MD_0} \left[\pi(1-\binit \mid \mf t_1) \cdot \left( V_2^{M^\bb, \pi}(\mf t_{2,0}^{(\ell - b_{\ell,0})\epbe}) + V_2^{M^\bb, \pi}(\mf t_{2,1}^{(\ell-b_{\ell,1})\epbe}) - \left( V_2^{M^\bb, \pi}(\mf t_{2,0}^{\ell\epbe}) + V_2^{M^\bb, \pi}(\mf t_{2,1}^{\ell\epbe})\right)\right)\right]\nonumber\\
    =& \frac{1}{2L} \sum_{\ell=1}^L \E_{\pi \sim \MD} \left[ V_2^{M^\bb, \pi}(\mf t_{2,0}^{(\ell - b_{\ell,0})\epbe}) + V_2^{M^\bb, \pi}(\mf t_{2,1}^{(\ell-b_{\ell,1})\epbe}) - \left( V_2^{M^\bb, \pi}(\mf t_{2,0}^{\ell\epbe}) + V_2^{M^\bb, \pi}(\mf t_{2,1}^{\ell\epbe})\right)\right]\nonumber\\
    \leq & 2\cphi R^{-1} (L+1)\epbe - \cphi \sum_{\ell=1}^L \E_{\pi \sim \MD} \left[  \pi(e^\st \mid \mf t_{2,e^\st}^{\ell \epbe}) - \pi(e^\st \mid \mf t_{2,e^\st}^{(\ell - b_{\ell, e^\st})\epbe} )\right]\nonumber\\
    = & 2\cphi R^{-1} (L+1)\epbe - \cphi \sum_{\ell=1}^L \left(\rho_{e^\st}(\ell) - \rho_{e^\st}(\ell - b_{\ell, e^\st})\right)\nonumber\\
    = & 2\cphi R^{-1} (L+1)\epbe - \cphi \sum_{\ell=1}^L [\rho_{e^\st}(\ell) - \rho_{e^\st}(\ell - b_{\ell,e^\st})]_+\nonumber\\
    \leq & 2\cphi R^{-1} (L+1)\epbe - \cphi \sqrt{\epbe}/2\leq -\cphi\sqrt{\epbe}/4\label{eq:mb-mbprime},
  \end{align}
  where the first inequality uses \cref{eq:bound-diff-v2}, the final equality uses the definition of $b_{\ell, e^\st}$ in \cref{eq:define-ble}, and the final inequality uses the fact that $R = 16$. 
  Also note that, from \cref{eq:v1pi-decompose}, 
  \begin{align}
    \E_{\pi \sim \MD_0} \left[ V_1^{M^{\bb'}, \pi}(\mf t_1) \right] =&  \cphi(1-2\brew) \frac{1}{L} \sum_{\ell=1}^L \left(\bar\eta(\ell) + \bar\gamma(\ell)\right) + \frac{\cphi}{RL} \sum_{\ell=1}^L Z_0 \cdot \ell \epbe \cdot (\rho(\ell) - 1) + \frac{\cphi}{RL} \sum_{\ell=1}^L (1-Z_0) \cdot \ell\epbe\nonumber\\
    \leq & \cphi \cdot \frac{\sqrt{\epbe}}{10} + \frac{\cphi(L+1)\epbe}{2R} = V_1^{M^{\bb}, \pi^\st}(\mf t_1) + \cphi \cdot \sqrt{\epbe}/10\label{eq:mb-mbstar}.
  \end{align}
  Combining \cref{eq:mb-mbprime,eq:mb-mbstar} and using the fact that $Z_0 \geq 1/2$ gives that
  \begin{align}
\E_{\pi \sim \MD_0} \left[ V_1^{M^\bb, \pi}(\mf t_1) \right] - V_1^{M^\bb, \pi^\st}(\mf t_1) \leq \cphi \sqrt{\epbe}/10 - \cphi \sqrt{\epbe}/8 = -\cphi \sqrt{\epbe}/40\nonumber,
  \end{align}
  as desired.
\end{proof}
    
}

\rlc{\bibliography{lbc.bib}}
\rlc{\bibliographystyle{rlc}}

\section*{Acknowledgements}
We thank Dylan Foster and Sham Kakade for bringing this problem to our attention and for helpful discussions.

\appendix

\section{Detailed comparison to related work}
\label{sec:related-work}
In this section, we discuss prior work on offline RL with function approximation and compare existing provable guarantees to our own. 

\paragraph{Actor-critic methods.} 
As mentioned above, the most closely related work is \cite{zanette2021provable}, which proves an analogous upper bound to ours for the special case of linear MDPs. More generally, the results of \cite{zanette2021provable} apply to the broader class of MDPs satisfying \emph{Bellman restricted closedness}, which requires that the Bellman backup of any linear function under \emph{any policy} is (approximately) linear in the features. In contrast, the weaker condition of low inherent Bellman error (\cref{asm:ibe}) requires that the Bellman backup of any linear function under the \emph{single policy} induced by that linear function is approximately linear. Intuitively, the fact that Bellman restricted closedness requires that Bellman backups under all poicies be linear places it much ``closer'' to the assumption of linear (i.e., low-rank) MDPs than to low inherent Bellman error. This intuition is formalized in \citet[Proposition 5.1]{jin2019provably}, which shows that under the mild additional assumption that each state has at least two distinct feature vectors, then Bellman restricted closedness implies that the MDP is a linear MDP. 

Several works have studied actor-critic methods in the setting of general function approximation, beginning with \cite{xie2021bellman}. In particular, \cite{xie2021bellman} considers the setting of a general function class $\MF \subset [0,1]^{\MX \times \MA}$ and policy class $\Pi \subset \Delta(\MA)^\MX$, and assumes that $(\MF, \Pi)$ satisfy approximate realizability and completeness.\footnote{Formally, to satisfy the realizability and completeness assumptions of \cite{xie2021bellman}, it suffices that for all $h \in [H]$, %
  $\sup_{\pi \in \Pi, f_{h+1} \in \MF}\inf_{f_h \in \MF} \| f_h - \MT_h^\pi f_{h+1} \|_\infty \leq \vep$. } Their first result, namely \citet[Theorem 3.1]{xie2021bellman}, establishes a \emph{computationally inefficient} algorithm for offline RL based on the principle of Bellman-consistent pessimism. One may use \cref{cor:ibe-linear} together with an appropriate choice of $\sigma$ to instantiate their result with $\Pi = \Pilinp$ and $\MF = \{ (x,a) \mapsto \lng \phi_h(x,a), w \rng \ : \ h \in [H], w \in \BR^d \}$ to obtain a computationally inefficient version of \cref{thm:pacle-ftpl}.\footnote{The resulting single-policy coverage parameter is somewhat larger than our own, though a subsequent observation by the authors of \cite{xie2021bellman} (see \cite{jiang2023}) leads to a tightening of their analysis which results in a matching coverage parameter.} %
Finally, \cite{xie2021bellman} establish an algorithm, \texttt{PSPI}, which is oracle-efficient given an oracle for $\MF$ which can solve a certain regularized least-squares problem. However, this result requires taking $\Pi$ to be a softmax policy class, for which approximate Bellman completeness does not in general hold under \cref{asm:ibe} (see \cref{lem:lbc-not-brc}). Thus, \texttt{PSPI} cannot be combined with \cref{cor:ibe-linear} to achieve an efficient algorithm under the assumption of low inherent Bellman error. (Moreover, their rate of $O(n^{-1/3})$ \citep[Corollary 5]{xie2021bellman} is suboptimal.) 

Subsequently, \cite{cheng2022adversarially} considered a similar setting, with general function and policy classes $(\MF, \Pi)$ satisfying approximate realizability and completeness. Their algorithm, \texttt{ATAC}, implements the actor using a generic no-regret learning algorithm, and implements the critic by solving a Lagrange relaxation of a least-squares regression problem. When combined with our main structural result, \cref{cor:ibe-linear}, it is possible to use \texttt{ATAC} to obtain an end-to-end computationally efficient learning algorithm for offline RL under \cref{asm:ibe}. In particular, one may take $\Pi$ to be the class of perturbed linear policies, the no-regret learning algorithm to be expected FTPL (i.e., \cref{alg:e-ftpl}), and one may implement the critic, which a priori appears to require minimizing a nonconvex quadratic function, using the approach in \citet[Appendix D]{xie2021bellman} (see also \cite{antos2008learning}). However, due to the Lagrangian term in the critic's optimization problem, the resulting rate (see \citet[Theorem 5]{cheng2022adversarially}) is worse than ours, scaling with $O(n^{-1/3})$. We remark that \cite{cheng2022adversarially} also considers the problem of \emph{robust policy improvement}, which shows that if the data is drawn from a behavior policy, then the algorithm's output performs nearly as well as the behavior policy, with no assumptions on the hyperparameters, no dependence on any concentrability coefficient, and no assumption of Bellman completeness. Recently, \cite{nguyen2023on} has given a refined analysis of \texttt{ATAC} which obtains the optimal $O(n^{-1/2})$ statistical rate, though only in the case when the policy class is the softmax policy class and Bellman completeness holds with respect to this class (which is not the case under \cref{asm:ibe}). 

Finally, \cite{zhu2023importance} introduces an algorithm, \texttt{A-Crab}, which is similar to \texttt{ATAC} but incorporates the idea of \emph{marginalized importance sampling}. %
As we discuss in the following paragraph, this approach cannot be instantiated in the setting of low inherent Bellman error to yield  computationally or statistically efficient guarantees for offline RL in full generality. 

\paragraph{Approaches via marginalized importance weighting.} In addition to \texttt{A-Crab}, many other works, starting with \cite{xie2019towards} have considered the approach of marginalized importance sampling (MIS), which introduces an additional bounded function class $\MW$ consisting of \emph{importance weights}. Elements of $\MW$ may be interpreted as possible values for the density ratio between the state-action visitation distribution of the policy $\pi^\st$ one wishes to compete with and the data distribution. A line of work, including \texttt{CORAL} \citep{rashidinejad2023optimal}, \texttt{PRO-RL} \citep{zhan2022offline},  \texttt{PABC} \citep{chen2022offline}, \texttt{MLB-PO} \citep{jiang2020minimax} and that of \cite{ozdaglar2023revisiting}, makes the assumption that $\MW$ contains the density ratio for $\pi^\st$, amongst other assumptions. Generally speaking, these algorithms implement pessimism for offline RL using primal-dual methods applied to the linear programming formulation of policy optimization. The introduction of $\MW$ allows many of them to establish upper bounds even in the absence of Bellman completeness. 

Several factors prevent such approaches from implying bounds similar to our own for the setting of MDPs satisfying linear Bellman completeness. First, all of the above works assume the existence of an oracle for optimizing over the class $\MW$, which would translate into an intractable nonlinear optimization problem in the setting of linear Bellman completeness. More fundamentally, it can be impossible to satisfy the assumption of realizability with respect to $\MW$: %
 even if we only wish to compete with a single policy $\pi^\st$, there is a class $\MM$ of linear Bellman complete MDPs so that there is no distribution $\mu_h$ for which $\sup_{M \in \MM} \left\| \frac{d_h^{\pi^\st, M}(x,a)}{\mu_h(x,a)} \right\|_\infty$ is bounded.\footnote{For instance, consider a class of MDPs for which an initial state-action pair $(x_1, a_1)$ transitions deterministically to any of infinitely many copies of some state, denoted $x_2^1, x_2^2, \ldots$, each of which is equivalent in the sense that $\phi_2(x_2^i, a) = \phi_2(x_2^j, a)$ for all $a \in \MA$, $i \neq j$.} Since the above results require that $\frac{d_h^{\pi^\st,M}(x,a)}{\mu_h(x,a)}$ belongs to $\MW$ for any $M \in \MM$, a bounded class $\MW$ of importance weights, satisfying realizability for the model class $\MM$, \emph{does not exist}. %

 \cite{uehara2023offline} analyses a primal-dual approach of a slightly different nature, but with the common goal of relaxing Bellman completeness at the cost of assuming realizability of a suitable class of Lagrange multipliers.  Finally, \cite{gabbianelli2023offline} uses a similar primal-dual method applied to the LP formulation of policy optimization as many of the above approaches, but only treats the special case of linear MDPs. Due to this additional structure, \cite{gabbianelli2023offline} does not need to explicitly make any assumptions regarding a class $\MW$. 

\paragraph{Linear MDPs: pessimistic value iteration.} In contrast to the above approaches, which phrase the problem of finding a pessimistic value function as a \emph{global} optimization problem, a line of work, including \texttt{PEVI} \citep{jin2021pessimism}, \texttt{VAPVI} \citep{yin2022nearoptimal}, \texttt{R-LSVI} \citep{zhang2022corruption}, and \texttt{LinPEVI-ADV} \citep{xiong2023nearly}   has considered a \emph{local} approach to implementing pessimism. In particular, these algorithms perform value iteration but subtract ``penalties'' at each state which are inversely proportional to how well the state is visited in the given dataset. Since the penalties do not necessarily have Bellman backups which are linear functions, these approaches do not directly generalize to the setting of linear Bellman complete MDPs.\footnote{Moreover, a necessary truncation step in pessimistic value iteration presents another obstacle to extending this approach to our setting.} The approach of pessimistic value iteration has also been extended to settings with nonlinear function approximation \citep{di2023pessimistic}. 

\paragraph{Additional approaches for offline RL with function approximation.} An older line of work \citep{munos2008finite,chen2019information} has studied offline RL under the stronger assumption of \emph{all policy concentrability}, meaning that the data distribution covers the state-action distribution of \emph{any} policy. These approaches proceed via variants of fitted $Q$-iteration, and therefore require approximate realizability and Bellman completeness in the setting of general function approximation. 
\cite{xie2021batch} show that the assumption of Bellman completeness can be avoided under an even stronger concentrability assumption. \cite{liu2020provably} analyzes pessimistic variants of value and policy iteration with only single-policy concentrability, under somewhat non-standard assumptions regarding completeness with respect to truncated Bellman backups. In the special case of tabular MDPs, \cite{rashidinejad2021bridging,shi2022pessimistic,xie2021policy} have focused on obtaining the optimal polynomial dependence on the various problem parameters, under single-policy concentrability. Finally, several works have considered model-based offline RL \citep{ross2012agnostic,chang2021mitigating,uehara2022pessimistic,bhardwaj2023adversarial}, which construct an estimate of all of the MDP's transitions and rewards (perhaps pessimistically) as opposed to estimating the value functions.

Finally, we mention that in a distinct setting to ours (namely, that of nonlinear dynamical systems), \cite{block2023provable} use the idea of injecting Gaussian noise into the learned policy to establish guarantees (see Definition 5.3 therein). This technique is analogous to our technique of using perturbed linear policies. 

\paragraph{Lower bounds.} \cite{wang2021what,amortila2020variant} show an exponential lower bound for offline RL even when the value function for \emph{any} policy is assumed to be linear in some known features, and when the distribution of the data has good coverage of \emph{all} feature directions. \cite{zanette2021exponential} shows a similar exponential lower bound, but which is stronger in that it holds no matter how the distribution of offline data is chosen. Finally, \cite{foster2021offline} shows a lower bound for offline RL in a nonlinear setting when the stronger assumption of \emph{concentrability} is made. Taken together, these results may be seen as motivating the assumption of Bellman completeness: when only realizability (as well as an appropriate coverage or concentrability notion) is assumed, little is possible.

\rlc{}

\rlc{}

\rlc{}

\section{Useful lemmas}
\label{sec:lemmas}
\subsection{Concentration}

\begin{lemma}[Concentration for self-normalized process; e.g., Theorem D.3 of \cite{jin2019provably}]
  \label{lem:conc-sn-martingale}
  Fix $n \in \BN$ and let $\vep_1, \ldots, \vep_n$ be random variables which are adapted to a filtration $(\MF_i)_{0 \leq i \leq n}$. Suppose that for each $i \in [n]$, $\E[\vep_i | \MF_{i-1}] = 0$ and $\E[e^{\lambda \vep_i} | \MF_{i-1}] \leq e^{\lambda^2 \sigma^2/2}$. Suppose that $\phi_1, \ldots, \phi_n$ is a sequence which is predictable with respect to $(\MF_i)_{0 \leq i \leq n}$, i.e., $\phi_i$ is measurable with respect to $\MF_{i-1}$ for all $i \in [n]$. Suppose that $\Gamma_0 \in \BR^{d \times d}$ is positive definite, and let $\Gamma_i = \Gamma_0 + \sum_{j=1}^i \phi_j \phi_j^\t$. Then for any $\delta > 0$, with probability at least $1-\delta$,
    \begin{align}
\left\| \sum_{i=1}^n \phi_i \vep_i \right\|_{\Gamma_i^{-1}}^2 \le 2 \sigma^2 \log \left( \frac{\det(\Gamma_t)^{1/2} \det (\Gamma_0)^{-1/2} }{\delta} \right)\nonumber. 
  \end{align}
\end{lemma}

\subsection{Projection bound}
\begin{lemma}
  \label{lem:projection-bound}
  Consider any sequence of vectors $\phi_1, \ldots, \phi_n \in \BR^d$ and a sequence of real numbers $b_1, \ldots, b_n \in \BR$, so that, for some $\ep > 0$, $|b_i| \leq \ep$  for all $i \in [n]$. Then for any $\lambda \geq 0$,
  \begin{align}
\left\| \sum_{i=1}^n b_i \phi_i \right\|_{\left( \lambda I + \sum_{i=1}^n \phi_i \phi_i^\t \right)^{-1}}^2 \leq n\ep^2\nonumber.
  \end{align}
\end{lemma}

\subsection{Performance difference lemma}
\begin{lemma}[Performance difference lemma; \cite{kakade2002approximately}]\label{lem:perf-diff}
For any MDP $M$, policies $\pi,\pi'\in\Pi$, it holds that
\[\E\sups{M,\pi}\left[\sum_{h=1}^H r_h(x_h,a_h)\right] - \E\sups{M,\pi'}\left[\sum_{h=1}^H r_h(x_h,a_h)\right] = \sum_{h=1}^H \E\sups{M,\pi'}\left[V\sups{M,\pi}_h(x_h) -  Q\sups{M,\pi}_h(x_h,a_h)\right].\]
\end{lemma}

\section{Bellman restricted closedness}
\label{sec:brc}

In this section, we show that linear Bellman completeness does not, in general, imply that Bellman restricted closedness holds, even when the policy class is restricted to be softmax policies. We first make the requisite definitions. 
\begin{defn}[Softmax policy class]
  \label{def:softmax}
  Given feature mappings $(\phi_h : \MX \times \MA \to \BR^d)_{h \in [H]}$, the associated \emph{softmax policy class} $\Pisoft$ consists of the set of all policies $\pi = (\pi_1, \ldots, \pi_H)$, for which there is some $\eta > 0$ and $w_1, \ldots, w_H \in \BR^d$ so that for all $h \in [H]$ and $x \in \MX$, 
  \begin{align}
\pi_h(a|x) = \frac{ \exp(\eta \cdot \lng \phi_h(x,a), w_h \rng)}{\sum_{a' \in \MA} \exp(\eta \cdot \lng \phi_h(x,a'), w_h \rng)}\nonumber.
  \end{align}
  If $\pi_h$ satisfies the above display, we write $\pi_h = \pisoft_h[{w_h, \eta}]$. 
\end{defn}

\begin{defn}[Bellman restricted closedness]
  \label{def:brc}
  An MDP $M$ is said to satisfy \emph{Bellman restricted closendess} with respect to $d$-dimensional feature mappings $(\phi_h)_{h \in [H]}$ for a policy class $\Pi'$ if for each $\pi \in \Pi'$, there are mappings $\MT_h^\pi : \BR^d \to \BR^d$ so that the following holds for each $h \in [H]$:
  \begin{align}
\sup_{w \in \BR^d} \sup_{(x,a) \in \MX \times \MA} \left|\lng \phi_h(x,a), \MT_h^\pi w \rng - \E_{x' \sim P_h(x,a)} \left[r_{h}(x, a) + \lng \phi_{h+1}(x', \pi_{h+1}(x')), w \rng \right] \right| = 0\nonumber.
  \end{align}
\end{defn}

\begin{lemma}
  \label{lem:lbc-not-brc}
There is an MDP with $H=2$, $A = 3$ together with feature mappings in $d=1$ dimension which satisfies linear Bellman completeness (i.e., \cref{asm:ibe} with $\epbe = 0$) but not Bellman restricted closedness (\cref{def:brc}) with respect to the softmax policy class $\Pisoft$. 
\end{lemma}
\begin{proof}
  Consider the MDP with $H=2, \MA = \{1,2,3\}, d = 1$, $\MX = \{ \mf s_1, \mf s_2 \}$, and
  \begin{align}
    \phi_1(x, a) = 1 \quad \forall x \in \MX, a \in \MA \nonumber\\
    \phi_2(\mf s_1, 1) = 1,\  \phi_2(\mf s_1, 2) = 0,\ \phi_2(\mf s_1, 3) = -1 \nonumber\\
    \phi_2(\mf s_2, 1) = 1,\ \phi_2(\mf s_2, 2) = 1, \phi_2(\mf s_2, 3) = -1\nonumber.
  \end{align}
  All rewards are 0. The transitions are as follows: for any $a \in \MA$, $x \in \MX$, $(x,a)$ transitions to $x$ at step 1. By defining $\MT_1 w = |w|$, we may ensure that linear Bellman completeness holds with respect to the above feature mappings: indeed, for each $x \in \MX$, $\max_{a \in \MA} w \cdot \phi_2(x,a) = |w| = \lng \phi_1(x, a'), \MT_1 w \rng$ for all $a' \in \MX$.

  Now let $\pi_2 := \pisoft_2[1, 1]$ (see \cref{def:softmax}) and $w = 1$. Then
  \begin{align}
    \lng \phi_2(\mf s_1, \pi_2(\mf s_1)), w \rng =& \frac{e - e^{-1}}{e + 1 + e^{-1}}\nonumber\\
    \lng \phi_2(\mf s_2, \pi_2(\mf s_2)), w \rng =& \frac{2e}{2e + e^{-1}}\nonumber.
  \end{align}
  Since $\frac{e-e^{-1}}{e+1+e^{-1}} \neq \frac{2e}{2e + e^{-1}}$, and $\phi_1(\mf s_1, a) = \phi_1(\mf s_2, a)$ for all $a$, Bellman restricted closedness cannot hold (even up to constant approximation error). 
\end{proof}

\arxiv{\bibliographystyle{alpha}}
\arxiv{\bibliography{lbc.bib}}

\end{document}